\documentclass[10pt, letterpaper]{article}
\usepackage{arxiv}

\title{
\vspace{-5.5mm}
   \bf \Large Distributionally Robust Reinforcement Learning with Interactive Data Collection:  Fundamental Hardness and 
     Near-Optimal Algorithms
}
\author{
    Miao Lu\thanks{Equal contributions. Email to \texttt{miaolu@stanford.edu}, \texttt{hanzhong@stu.pku.edu.cn}} \thanks{Department of Management Science and Engineering, Stanford University.} \qquad 
    Han Zhong\footnotemark[1] \thanks{Center for Data Science, Peking University.} \qquad
    Tong Zhang\thanks{Department of Computer Science, University of Illinois Urbana-Champaign.}\qquad 
    Jose Blanchet\footnotemark[2]
}

\date{\small{April 5, 2024; \quad Revised: July 13, 2026}}

\begin{document}


\maketitle

\begin{abstract}
    The sim-to-real gap, which represents the disparity between training and testing environments, poses a significant challenge in reinforcement learning (RL). 
    A promising approach to addressing this challenge is distributionally robust RL, often framed as a robust Markov decision process (RMDP). 
    In this framework, the objective is to find a robust policy that achieves good performance under the worst-case scenario among all environments within a pre-specified uncertainty set centered around the training environment. 
    Unlike previous work, which relies on a generative model or a pre-collected offline dataset enjoying good coverage of the deployment environment, we tackle robust RL via interactive data collection, where the learner interacts with the training environment only and refines the policy through trial and error. 
    In this robust RL paradigm, two main challenges emerge: managing distributional robustness while striking a balance between exploration and exploitation during data collection. 
    Initially, we establish that sample-efficient learning without additional assumptions is unattainable owing to the curse of support shift; i.e., the potential disjointedness of the distributional supports between the training and testing environments. 
    To circumvent such a hardness result, we introduce the vanishing minimal value assumption to RMDPs with a total-variation (TV) distance robust set, postulating that the minimal value of the optimal robust value function is zero. 
    We prove that such an assumption effectively eliminates support shift pathologies for RMDPs with a TV distance robust set, and present an algorithm with near-optimal sample complexity. 
    To demonstrate the breadth of our framework, we
    further extend our algorithm and theory to new robust set formulations and robust Markov game settings.
    Finally, to illustrate the operational relevance of our framework, we apply our algorithm to the data-driven robust inventory control, yielding explicit learning guarantees for robust decision-making under demand shifts.
    Our work makes the initial step to uncovering the inherent difficulty of robust RL via interactive data collection and sufficient conditions for designing a sample-efficient algorithm accompanied by sharp sample complexity analysis.
\end{abstract}


\newpage 

\tableofcontents


\newpage
\section{Introduction}

Reinforcement learning (RL) serves as a framework for addressing complex decision-making problems through iterative interactions with environments.
Recent advancements in deep reinforcement learning have enabled the successful application of the general RL framework across various domains, including mastering strategic games, such as Go \citep{silver2017mastering}, robotics \citep{kober2013reinforcement}, and aligning large language models (LLMs; \citealp{ouyang2022training}).
The critical factors contributing to these successes encompass not only the potency of deep neural networks and modern deep RL algorithms but also the availability of substantial training data.
However, there are scenarios, such as healthcare \citep{wang2018supervised}, inventory control \citep{boute2022deep}, and autonomous driving \citep{kiran2021deep}, among others, where collecting RL data in the target domain is challenging, costly, or even infeasible.
In such cases, the sim-to-real transfer \citep{kober2013reinforcement,sadeghi2016cad2rl,peng2018sim,zhao2020sim} becomes a remedy, where the RL agent is trained in simulated environments and subsequently deployed in the real world.
Nevertheless, the discrepancy between the training environments and the testing environments, also known as the sim-to-real gap, will typically result in suboptimal performance of RL agents in real-world applications.
One promising strategy to mitigate performance degradation due to the sim-to-real gap is robust RL \citep{iyengar2005robust,pinto2017robust,hu2022provable}, which aims to learn policies exhibiting strong (i.e. robust) performance under environmental deviations from the training environment.
It effectively hedges the epistemic uncertainty arising from the differences between the training environment and the unknown testing environments.

A robust RL problem is formulated as a robust Markov decision process (RMDP), with different types of robust sets characterizing different environment shifts.
Prior theoretical works on robust RL have developed algorithms with provable sample complexity guarantees, but they typically rely on either a generative model\footnote{A generative model here means a mechanism that when queried at some state, action, and time step, returns a sample of next state. Here we distinguish this notion with the notion of simulator or simulated environment which generally refers to a human-made training environment that mimics the real-world environment.
} \citep{yang2022toward, panaganti2022sample, xu2023improved, shi2023curious} or offline datasets with good coverage of the deployment environment \citep{zhou2021finite, panaganti2022robust, shi2022distributionally,ma2022distributionally,blanchet2023double}. Notably, the current literature does not explicitly address the \emph{exploration} problem, which stands as one of the fundamental challenges in reinforcement learning through trial-and-error \citep{sutton2018reinforcement}.
Meanwhile, the empirical success of robust RL methods \citep{pinto2017robust, kuang2022learning, moos2022robust} typically relies on reinforcement learning through interactive data collection in the training environment, where the agent iteratively and actively interacts with the environment, collecting data, optimizing and robustifying its policy.
Given that all of the existing literature on the theory of robust RL relies on a generative model or a pre-collected offline dataset,  it is natural to ask:
\begin{center}
   \emph{Can we design a provably sample-efficient robust RL algorithm that relies on \\ interactive data collection in the training environment?}
\end{center}

Answering this question faces a fundamental challenge: during interactive data collection, the learner no longer has oracle control over the training data distributions that are induced by the policy learned through the interaction process.
In particular, it could be the case that certain data patterns that are crucial for the policy to be robust across all testing environments are not accessible through interactive data collection, even with a sophisticated exploration mechanism.
For example, specific states may not be accessible within the training environment dynamics but could be reached in the testing environment dynamics.

In contrast, previous work has demonstrated that robust RL through a generative model or a pre-collected offline dataset with good coverage does not face such difficulties.
For the generative model setup, fortunately, the learner can directly query any state-action pair and observe the sampled next state from the generator.
Intuitively, once the states that could appear in the testing environment trajectory are queried enough times, it is then possible to guarantee the performance of the learned policy in testing environments.
The situation is similar if one has a pre-collected offline dataset that enjoys good coverage of testing environments.
In this work, we make the initial step towards studying the theory and applications of robust RL with interactive data collection.
At a high level, our results and contributions are three-fold.
\begin{itemize}
    \item (\emph{Fundamental hardness.}) We first prove a hardness result for robust RL with interactive data collection. Precisely, certain RMDPs that are solvable sample-efficiently with a generative model or with sufficient offline data with good coverage properties are, in contrast, \emph{intractable} for robust RL through interactive data collection. This shows a gap between robust RL with these two different kinds of data-type oracles.
    \item (\emph{Solvable class and sample-efficient algorithm.}) We  identify a tractable subclass of RMDPs, for which we further propose a novel robust RL algorithm that can provably learn a near-optimal robust policy through interactive data collection.
    This implies that robust RL with interactive data collection is still possible for certain subclasses of RMDPs.
    \item (\emph{Extensions and applications.}) To demonstrate the breadth of our theory, we extend it to two additional robust RL settings: a different robust-set formulation and a multi-agent extension. Finally, to showcase the practical relevance of our framework, we instantiate it in a representative operations research application that naturally exhibits sim-to-real (model-shift) concerns.
\end{itemize}
Together, our work answers the above question and shows that robust RL with interactive data collection is not only theoretically characterizable, but also directly applicable to canonical operations problems. In the following section, we explain more explicitly the problem setup and the contributions we make.


\subsection{Contributions}

This work primarily studies robust RL in a finite-horizon RMDP with an $\cS\times\cA$-rectangular total-variation distance (TV) robust set (see Assumption~\ref{ass: sa} and Definition~\ref{def: tv})\footnote{We notice that all of the previous work on sample-efficient robust RL in RMDPs with TV robust sets \citep{yang2022toward, panaganti2022sample, panaganti2022robust, xu2023improved, blanchet2023double, shi2023curious} relies on defining the TV distance through the general $f$-divergence so that a strong duality representation holds.
But this implicitly requires the testing environment transition probability to be absolutely continuous w.r.t. the training environment transition probability.
In this paper, we do not make such a restriction.
We prove the same strong duality even if the absolute continuity does not hold.
In fact, all the previous work can be directly extended to such TV distance definition via our more general strong duality result.} with interactive data collection.

\paragraph{Fundamental hardness.} We construct a class of hard-to-learn RMDPs (see Example~\ref{exp: hard} and Figure~\ref{fig:hardmdp}) and demonstrate that \emph{any} learning algorithm inevitably incurs an $\Omega(\rho  \cdot HK)$-online regret (see \eqref{eq: regret}) under at least one RMDP instance.
Here, $\rho$ signifies the radius of the TV robust uncertainty set, $H$ is the horizon, and $K$ denotes the number of interactive episodes. This linear regret lower bound underscores the impossibility of sample-efficient robust RL via interactive data collection in general.

\paragraph{Identifying a tractable case.} Upon close examination of the challenging instance, we recognize that the primary obstacle to achieving sample-efficient learning lies in the \emph{curse of support shift}, i.e., the disjointedness of distributional support between the training environment and the testing environments.
In a broader sense, the curse of support shift also refers to situations where states that often appear in testing environments are extremely hard to reach in the training environment\footnote{We remark that an existing work of \cite{dong2022online} also studies the problem of robust RL with interactive data collection.
They consider $\cS\times\cA$-rectangular RMDPs with a TV robust set, assuming that the support of the training environment transition is the full state space.
They claim the existence of an algorithm that enjoys a $\widetilde{\cO}(\sqrt{K})$-online regret. We point out that their proof exhibits an essential flaw (misuse of Lemma 12 therein) and therefore the regret they claim is invalid.
}.

To rule out these pathological instances, we propose the \emph{vanishing minimal value} assumption (Assumption~\ref{ass: zero min}), positing that the optimal robust value function reaches zero at a specific state. Such an assumption naturally applies to the sparse reward RL paradigm and offers a broader scope compared to the ``fail-state" assumption utilized in prior studies on offline RMDPs with function approximation \citep{panaganti2022robust}. For a comprehensive discussion on this comparison, please refer to Remark~\ref{remark:assumption:compare}. On the theoretical front, we establish that the vanishing minimal value assumption effectively mitigates the support shift issues between training and the testing environments (Proposition~\ref{prop: equivalent robust set}), rendering robust RL with interactive data collection feasible for RMDPs equipped with TV robust sets.

\paragraph{Efficient algorithm with sharp sample complexity.} Under the vanishing minimal value assumption, we develop an algorithm named \underline{OP}timistic \underline{RO}bust \underline{V}alue \underline{I}teration for \underline{TV} Robust Set (\texttt{OPROVI-TV}, Algorithm~\ref{alg: tv}).
We first prove that \texttt{OPROVI-TV} achieves sublinear online robust regret of order $\widetilde{\cO}(\sqrt{K})$ over $K$ episodes of interactive data collection (Theorem~\ref{thm: regret tv}).
By a standard online-to-batch conversion, this regret guarantee further implies that \texttt{OPROVI-TV} can find an $\varepsilon$-optimal robust policy within
\begin{align}
    \widetilde{\mathcal{O}}\Bigg(\min\{H, \rho^{-1}\} \cdot \frac{H^{2}SA }{\varepsilon^2}\Bigg)\label{eq: sample complexity}
\end{align}
interactive samples (Corollary~\ref{cor: sample complexity tv}).
Here $S$ and $A$ denote the number of states and actions, $\rho$ represents the radius of the TV robust set, and $H$ is the horizon length of each episode.
To the best of our knowledge, this is the first provably sample-efficient algorithm for robust RL with interactive data collection.

According to \eqref{eq: sample complexity}, the sample complexity of finding an $\varepsilon$-optimal robust policy decreases as the radius $\rho$ of the robust set increases.
This coincides with the findings of \cite{shi2023curious} who consider robust RL in infinite-horizon discounted RMDPs with TV robust sets within the generative model setup.
When the radius $\rho=0$, an RMDP reduces to a standard MDP, and the sample complexity \eqref{eq: sample complexity} recovers the minimax-optimal sample complexity for online RL in standard MDPs up to logarithm factors, i.e., $\widetilde{\cO}(H^3SA/\varepsilon^2)$.
At the other extreme, when $\rho\rightarrow 1$\footnote{We do not signify the situation when $\rho = 1$ since in that case the TV robust set  contains all possible transition probabilities, making  the problem statistically trivial. In that case, no sample is needed.}, finding an $\varepsilon$-optimal robust policy turns out to require nearly a factor of $H$ fewer samples, up to logarithmic factors, than finding the optimal policy in a standard MDP.

\paragraph{Extensions to other robust RL setups.}
Going beyond the main results on robust RL in finite-horizon RMDPs with $\cS\times\cA$-rectangular TV robust sets, we further extend our algorithm and theory to other types of robust decision-making setups.
Specifically:
\begin{itemize}
    \item (\emph{Robust RL with other robust sets.}) We first study the problem of robust RL in another type of RMDPs,  $\cS\times\cA$-rectangular discounted RMDPs equipped with robust sets consisting of transition probabilities with bounded ratio to the nominal kernel (Section~\ref{subsec: extentions}).
    This class of RMDPs naturally does not suffer from the support shift issue, and we prove that it is equivalent to the $\cS\times\cA$-rectangular RMDP with TV robust set and vanishing minimal value assumption in an appropriate sense due to Proposition~\ref{prop: equivalent robust set}.
    Consequently, using Algorithm~\ref{alg: tv} through the auxiliary construction, we can also solve robust RL for this new model sample-efficiently (Corollary~\ref{cor: regret discount}).
    Such a result also echoes our intuition on the curse of support shift.
    \item (\emph{Robust multi-agent RL in robust Markov games.})
    We further extend our framework to robust Markov games \citep{kardes2005robust} that jointly capture strategic opponents and environment ambiguity (Section~\ref{sec:rmg_extension}).
    This setting is motivated by multi-agent decision-making problems in operations research (e.g., security games, competitive resource allocation), where one must simultaneously hedge against the worst-case opponent behavior and the worst-case transition.
    We introduce the robust Nash value with Bellman--Shapley recursion and establish its existence and strong duality under $\cS\times\cA\times\cB$-rectangular TV robust set (Proposition~\ref{prop:robust_rmg_minimax}).
    Here $\cA$ and $\cB$ denote the action spaces of the two players in the game  with $A$ and $B$ denoting their cardinalities.
    Building on this structure, we develop \texttt{OPROVI-TV-MG} (Algorithm~\ref{alg:onesided_oprovitv_mg}), a game-theoretic extension of \texttt{OPROVI-TV} that estimates the joint-action transition kernels and performs optimistic robust max--min planning by solving per-state matrix games in each backward pass.
    Under the vanishing minimal value assumption for RMGs, our Theorem~\ref{thm:onesided_regret} proves that \texttt{OPROVI-TV-MG} achieves a sublinear online robust regret over $K$ episodes of interactive data collection of order
    $${\widetilde{\cO}\left(\min\{H,\rho^{-1}\}\,H S\sqrt{A B K}\right)}$$
    against any adaptive Markov opponent sequence, showing that Player~1 can approach the robust Nash value through online interaction.
\end{itemize}

\paragraph{Applications to data-driven inventory control.}
To demonstrate the practical relevance of our framework, we apply our algorithm and theory to a canonical operations-management problem:
\begin{itemize}
    \item (\emph{Data-driven robust inventory control under demand distribution shifts.})
    We apply our framework to data-driven robust inventory control under demand distribution shift (Section~\ref{sec:inventory}).
    Motivated by the fact that the actual demand law may deviate from the training demand law simulating the real-world inventory or representing the historical pattern, we model demand perturbations via an $\cS\times\cA$-rectangular TV ambiguity set, which induces a conservative TV ball on the transition kernel (Lemma~\ref{lem:tv:contraction}).
    Exploiting that the induced inventory system forms a finite-horizon RMDP equipped with $\cS\times\cA$-rectangular TV robust set and satisfies the vanishing minimal value condition due to the existence of absorbing aggregated emergency (fail) state, we are able to directly apply \texttt{OPROVI-TV} (Algorithm~\ref{alg: tv}) to learn a distributionally robust ordering policy from interactive data collection in the training environment only.
    In particular, Theorem~\ref{thm:inventory_control} establishes that \texttt{OPROVI-TV} finds an $\varepsilon$-optimal robust inventory policy within
    $$\widetilde{\cO}\left(\min\{H,\rho^{-1}\} \cdot \frac{H^{2} (B+I) Q}{\varepsilon^2}\right)$$
    episodes of interactive data collection in the training environment. Here $\rho$ signifies the robust set size reflecting the demand shift, $I$ denotes the inventory capacity, $B$ denotes the backlog threshold, and $Q$ is the order capacity (see more concrete definitions in Section~\ref{sec:inventory}).
\end{itemize}

\subsection{Related Works}\label{subsec: related works}

\paragraph{Robust reinforcement learning in robust Markov decision processes.}
Robust RL is usually framed as a robust Markov decision process (RMDP) \citep{iyengar2005robust,el2005robust,wiesemann2013robust}.
There is a long line of work dedicated to the problem of how to solve for the optimal robust policy of a given RMDP, i.e., planning  \citep{iyengar2005robust,el2005robust,xu2010distributionally,wang2022policy,wang2022convergence, kuang2022learning, pmlr-v202-wang23i, yu2023fast, zhou2023natural, li2023first, wang2023foundation, ding2024seeing}.
Recently, the community has also witnessed a growing body of work on sample-efficient robust RL in RMDPs with different data collection oracles, including the generative model setup \citep{yang2022toward,panaganti2022sample,si2023distributionally,wang2023finite,yang2023avoiding, xu2023improved, clavier2023towards, wang2023sample, shi2023curious}, offline setting \citep{zhou2021finite, panaganti2022robust, shi2022distributionally,ma2022distributionally, blanchet2023double,liu2024minimax,wang2024sample}, and interactive data collection setting \citep{badrinath2021robust, wang2021online, liu2024distributionally}.

Our work falls into the paradigm of sample-efficient robust RL through interactive data collection.
\cite{wang2021online} and \cite{badrinath2021robust} propose efficient online learning algorithms to obtain the optimal robust policy of an infinite horizon RMDP, but none of them handle the challenge of exploration in online RL by assuming the access to \emph{explorative policies}. This assumption enables the learner to collect high-quality data essential for effective learning and decision-making. In contrast, our work focuses on developing efficient algorithms for the fully online setting, where there is no predefined exploration policy to use. Under this more challenging setting, we address the exploration challenge through algorithmic design rather than relying on assumed access to explorative policies.

During the preparation of this work, we are aware of several concurrent and independent works \citep{liu2024distributionally,liu2024minimax,wang2024sample}, which study a different type of RMDPs known as $d$-rectangular linear MDPs \citep{ma2022distributionally,blanchet2023double}.
In particular, \citet{liu2024minimax} and \citet{wang2024sample} consider the offline setting, while \citet{liu2024distributionally} investigate robust RL through interactive data collection (off-dynamics learning), thus bearing closer relevance to our work.
More specifically, under the existence of a ``fail-state", the algorithm in \citet{liu2024distributionally} can learn an $\varepsilon$-optimal robust policy with provable sample efficiency.
In contrast, our work first explicitly uncovers the fundamental hardness of robust RL in RMDPs with TV robust set and without additional assumptions.
To overcome the inherent difficulty, we adopt a vanishing minimal value assumption that strictly generalizes the ``fail-state" assumption used in \cite{liu2024distributionally}.
Moreover, our focus is on tabular $\cS\times\cA$-rectangular RMDPs, with customized algorithmic design and theoretical analysis which allow us to obtain a sharp sample complexity bound.

Finally, in Table~\ref{table: tv}, we compare the learning guarantees of our algorithms with those of prior work on robust RL for RMDPs with $\cS\times\cA$-rectangular TV robust sets under various settings (generative model/offline dataset), and report the regret guarantee for our robust Markov game extension.

\begin{table}[!t]
    \centering
    \begingroup
    \small
    \renewcommand\arraystretch{1.25}
    \setlength{\tabcolsep}{3pt}
    \begin{tabular}{ | c | c | c | c | }
    \hline
    Model Assump. & Algorithm
    & Data oracle
    & \makecell[c]{Sample complexity \\
    / regret}
      \\
      \hline
      \multirow{5}*{general case}&\texttt{RPVL} \citep{xu2023improved} & generative model & $\widetilde{\mathcal{O}}\left(\frac{H^{5}SA}{\varepsilon^2}\right)$ \\
      \cline{2-4}
      & \texttt{DRVI} \citep{shi2023curious}& generative model &$\widetilde{\mathcal{O}}\left(\frac{\min\{H_{\gamma}, \rho^{-1}\}H^{2}_{\gamma}SA}{\varepsilon^2}\right)$ \\
      \cline{2-4}
      & lower bound \citep{shi2023curious} & generative model &$\Omega\left(\frac{\min\{H_{\gamma}, \rho^{-1}\}H^{2}_{\gamma}SA}{\varepsilon^2}\right)$ \\
      \cline{2-4}
      &$\texttt{P}^2$\texttt{MPO} \citep{blanchet2023double} &\makecell[c]{offline dataset 
      } &$\widetilde{\cO}\left(\frac{\cC^{\star}_{\mathrm{rob}}H^4S^2A}{\varepsilon^2}\right)$\\
      \cline{2-4}
      &lower bound (this work) &interactive data collection & intractable\\
    \hline
    \hline
    \makecell[c]{``fail-state" \\ assumption}&\texttt{RFQI} \citep{panaganti2022robust} &\makecell[c]{offline dataset 
    }& $\widetilde{\cO}\left(\frac{\cC_{\mathrm{full}}H_{\gamma}^4SA}{\rho^2\varepsilon^2}\right)$\\
    \hline
    \makecell[c]{vanishing \\ minimal value \\
    (Assumption~\ref{ass: zero min})}& \texttt{OPROVI-TV} (this work) & interactive data collection  & $\widetilde{\mathcal{O}}\left(\frac{\min\{H, \rho^{-1}\}H^{2}SA}{\varepsilon^2}\right)$    \\
    \hline
    \hline
    \makecell[c]{vanishing \\ minimal value \\
    (Assumption~\ref{ass:onesided_vmv})}&\texttt{OPROVI-TV-MG} (this work) &interactive data collection
    & {$\widetilde{\cO}\left(\min\{H, \rho^{-1}\}H S\sqrt{ABK}\right)$}\\
    \hline
    \end{tabular}
    \endgroup
    \caption[Learning guarantees for robust RL with TV ambiguity and Markov games]{\small
    Comparison of learning guarantees for robust RL with TV ambiguity and its Markov-game extension. For the RMDP rows, the last column reports the sample complexity for learning an $\varepsilon$-optimal robust policy; for the RMG row, it reports online regret against adaptive Markov opponents. The rows involving infinite-horizon $\gamma$-discounted RMDPs use $H_{\gamma}:=(1-\gamma)^{-1}$ as the effective horizon. The quantities $\cC_{\mathrm{rob}}^\star$ and $\cC_{\mathrm{full}}$ denote robust partial and full coverage coefficients.
    }
    \label{table: tv}
\end{table}

\paragraph{Sample-efficient online non-robust reinforcement learning.} Our work is also closely related to online non-robust RL, which is often formulated as a Markov decision process (MDP) with online data collection. For non-robust online RL, the key challenge is the exploration-exploitation tradeoff.
There has been a long line of work \citep{azar2017minimax,dann2017unifying,jin2018q,zanette2019tighter,zhang2020almost,zhang2021reinforcement,menard2021ucb,wu2022nearly,li2023q,zhang2023settling} addressing this challenge in the context of tabular MDPs, where the state space and action space are finite and also relatively small.
In particular, many algorithms (e.g., \texttt{UCBVI} in \citet{azar2017minimax}) have been proven capable of finding an $\varepsilon$-optimal policy within $\widetilde{\cO}(H^3SA/\varepsilon^2)$ sample complexity. Notably, a standard MDP corresponds to an RMDP with a TV robust set and $\rho = 0$, suggesting that \texttt{OPROVI-TV} can naturally achieve nearly minimax-optimality for non-robust RL.
Moving beyond the tabular setups, recent works also investigate online non-robust RL with linear function approximation \citep{jin2020provably,ayoub2020model,zhou2021nearly,zhong2023theoretical,huang2023tackling,he2023nearly,agarwal2023vo} and even general function approximations \citep{jiang2017contextual,sun2019model,du2021bilinear,jin2021bellman,foster2021statistical,liu2022welfare, zhong2022gec,liu2023one,huang2023horizon,xu2023bayesian,agarwal2023vo}.


\paragraph{Robust RL in robust Markov games.}
Robust Markov games, also known as robust stochastic games, extend robust Markov decision process to multi-agent settings with both strategic interaction between agents and environment ambiguity \citep{kardes2005robust}.
On the offline and generative model side, \citet{blanchet2023double} study offline RL in robust Markov games under general function approximations and propose the robust Nash equilibrium gap (RNE gap) as the performance criterion for a learned joint policy profile.
\citet{shi2024sampleefficient} consider the tabular robust Markov games with a generative model oracle, studying the sample complexity of learning robust variants of Nash, correlated, and coarse correlated equilibria using equilibrium-gap metrics.
\citet{shi2024breaking} propose a different model of robust Markov games and obtain improved sample complexity guarantees that avoid the curse of dimensionality in the joint action space.
All these works differ from ours in both data access and learning objective:
they focus on robust equilibrium learning for a jointly controlled policy profile, whereas we study online interactive data collection with unknown transitions and an external adversarial opponent.
Our goal is not to learn an approximate RNE for all players, but to control the regret of Player~1 against robust Nash value.
On the online side, recent works \citep{farhat2025sample,zheng2025distributionally} also study learning in robust Markov games. However, as discussed in Remark~\ref{rmk:online_rmg_diff}, they use cumulative equilibrium-gap objectives that are closer to the online analogue of the RNE gap, whereas our formulation uses a robust regret benchmark against the realized opponent policies, tailored to the setting where only Player~1 is under our control.

\paragraph{Data-driven inventory control.} Learning-based inventory control for unknown demand has been studied extensively in non-robust settings \citep[e.g.,][]{huh2011adaptive,shi2016nonparametric,agrawal2019learning,zhang2020closing,yuan2021marrying,lyu2024ucb,fan2024don}.
On the robustness side, the classical work goes back to Scarf's minimax formulation \citep{scarf1957min} and subsequent robust and distributionally robust approaches to inventory control \citep{bertsimas2006robust,klabjan2013robust,xin2022distributionally}, which primarily focus on modeling demand ambiguity and solving the corresponding robust optimization or control problem. Our setting is different from both lines of prior work: we consider finite-sample learning from interactive data collected in a training environment, seeking a policy robust to demand shifts at deployment.

\paragraph{Additional follow-up works.}
Following the initial version of this paper, several subsequent works extend the scope of the problem into different directions.
\citet{liu2024upper} study the linear function approximation setting and improve upon the earlier result from \citet{liu2024distributionally}.
\citet{he2025sample,ghosh2025orvit} consider online distributionally robust RL under other robust sets beyond the TV-based model studied here.
\citet{ghosh2025scaling} further investigates the problem under general function approximation.
Finally, \citet{zheng2025distributionally,farhat2025sample} extend the interactive robust RL problem to robust Markov games.

\subsection{Notations}
For any positive integer $H\in\mathbb{N}_+$, we denote $\{1, 2, \ldots, H\}$ by $[H]$.  Given a set $\cX$, we denote $\Delta(\cX)$ as the set of probability distributions over $\cX$. For any distribution $p\in\Delta(\cX)$, we define the shorthand for expectation and variance as
\begin{align}
    \mathbb{E}_{p(\cdot)}[f] := \mathbb{E}_{X\sim p(\cdot)}[f(X)],\quad \mathbb{V}_{p(\cdot)}[f] = \mathbb{E}_{p(\cdot)}[f^2] - (\mathbb{E}_{p(\cdot)}[f])^2.
\end{align}
For any set $\mathcal{Q} \subseteq \Delta(\cX)$, we define the robust expectation operator as
\begin{align}
    \mathbb{E}_{\mathcal{Q}}[f] := \inf_{p(\cdot)\in\mathcal{Q}}\mathbb{E}_{X\sim p(\cdot)}[f(X)].
\end{align}
For any $x, a\in\mathbb{R}$, we denote $(x)_+ =\max\{x, 0\}$ and $x\vee a = \max\{x, a\}$.
We use $\cO(\cdot)$ to hide absolute constant factors and use $\widetilde{\cO}$ to further hide logarithmic factors.

\section{Preliminaries}

\subsection{Robust Markov Decision Processes}\label{subsec: robust MDP}

We first introduce our underlying model for doing robust RL, the episodic robust Markov decision process (RMDP), denoted by a tuple $(\mathcal{S}, \mathcal{A}, H, P^{\star}, R, \mathbf{\Phi})$.
Here the set $\mathcal{S}$ is the state space and the set $\mathcal{A}$ is the action space, both with finite cardinality.
The integer $H$ is the length of each episode.
The set $P^{\star}=\{P_h^{\star}\}_{h=1}^H$ is the collection of \emph{nominal} transition kernels where $P_h^{\star}:\mathcal{S}\times\mathcal{A}\mapsto\Delta(\mathcal{S})$.
The set $R=\{R_h\}_{h=1}^H$ is the collection of reward functions where $R_h:\mathcal{S}\times\mathcal{A}\mapsto[0,1]$.
For simplicity, we denote $\mathcal{P} = \{P(\cdot|\cdot,\cdot):\mathcal{S}\times\mathcal{A}\mapsto\Delta(\mathcal{S})\}$ as the space of all possible transition kernels, and we denote $S = |\cS|$ and $A = |\cA|$.

Most importantly and different from standard MDPs, the RMDP is equipped with a mapping $\mathbf{\Phi}:\mathcal{P}\mapsto 2^{\mathcal{P}}$ that characterizes the \emph{robust set} of any transition kernel in $\mathcal{P}$.
Formally, for any transition kernel $P\in\mathcal{P}$, we call $\mathbf{\Phi}(P)$ the \emph{robust set} of $P$.
One could interpret the nominal transition kernel $P^{\star}_h$ as the transition of the training environment, while $\boldsymbol{\Phi}(P^{\star}_h)$ contains all possible transitions of the testing environments.

Given an RMDP $(\mathcal{S}, \mathcal{A}, H, P^{\star}, R, \mathbf{\Phi})$, we consider using a Markovian policy to make decisions.
A Markovian decision policy (or simply, policy) is defined as $\pi=\{\pi_h\}_{h=1}^H$ with $\pi_h:\mathcal{S}\mapsto\Delta(\mathcal{A})$ for each step $h\in[H]$.
To measure the performance of a policy $\pi$ in the RMDP, we introduce its \emph{robust value function}, defined as
\begin{align}
    V_{h, P^{\star}, \mathbf{\Phi}}^{\pi}(s)&:= \inf_{\widetilde{P}_h\in\mathbf{\Phi}(P_h^{\star}), 1\leq h\leq H} \mathbb{E}_{\{\widetilde{P}_h\}_{h=1}^H,\{\pi_h\}_{h=1}^H}\left[\sum_{i=h}^HR_{i}(s_i,a_i)\, \middle|\, s_h=s\right],\quad \forall s\in\mathcal{S},\label{eq: robust V}\\
    Q_{h, P^{\star}, \mathbf{\Phi}}^{\pi}(s, a)&:= \inf_{\widetilde{P}_h\in\mathbf{\Phi}(P_h^{\star}), 1\leq h\leq H}\mathbb{E}_{\{\widetilde{P}_h\}_{h=1}^H,\{\pi_h\}_{h=1}^H}\left[\sum_{i=h}^HR_{i}(s_i,a_i)\, \middle|\, s_h=s,a_h=a\right],\quad \forall (s,a)\in\mathcal{S}\times\mathcal{A}.\label{eq: robust Q}
\end{align}
Here the expectation is taken w.r.t. the state-action trajectories induced by policy $\pi$ under the transition $\widetilde{P}$.
One can also extend the definition of the robust value functions in terms of any collection of transition kernel $P = \{P_h\}_{h=1}^H\subset\mathcal{P}$ as $ V_{h, P, \mathbf{\Phi}}^{\pi}$ and $Q_{h, P, \mathbf{\Phi}}^{\pi}$, which we usually use in the sequel.

Among all the policies, we define the optimal robust policy $\pi^{\star}$ as the policy that can maximize the robust value function at the initial time step $h=1$, i.e.,
\begin{align}\label{eq: optimal robust policy}
    \pi^{\star} = \argmax_{\pi=\{\pi_h\}_{h=1}^H} V_{1,P^{\star},\boldsymbol{\Phi}}^{\pi} (s_1),\quad \forall s_1\in\mathcal{S}.
\end{align}
In other words, the optimal robust policy $\pi^{\star}$ maximizes the worst case expected total rewards in all possible testing environments.
For simplicity and without loss of generality, we assume in the sequel that the initial state $s_1\in\cS$ is fixed.
Our results could be directly generalized to $s_1\sim p_0(\cdot) \in\Delta(\cS)$.
Similarly, we can also define the optimal robust policy associated with a given stochastic process defined through any collection of transition kernels $P = \{P_h\}_{h=1}^H\subset\mathcal{P}$ in the same way as \eqref{eq: optimal robust policy}.
We denote the optimal robust value functions associated with $P$ as $V_{h, P, \mathbf{\Phi}}^{\star}$ and $ Q_{h, P, \mathbf{\Phi}}^{\star}$ respectively.

\paragraph{$\cS\times\cA$-rectangularity and robust Bellman equations.}
We consider robust sets $\boldsymbol{\Phi}$ that have the $\cS\times\cA$-rectangular structure \citep{iyengar2005robust}, which requires that the robust set is decoupled and independent across different $(s,a)$-pairs.
This kind of structure results in a dynamic programming representation of the robust value functions (efficient planning), and is thus commonly adopted in the literature of distributionally robust RL.
More specifically, we assume the following.
\begin{assumption}[$\cS\times\cA$-rectangularity]\label{ass: sa}
    We assume that, for any transition kernel $P\in\cP$, the robust set $\boldsymbol{\Phi}(P)$ takes the form
    \begin{align}
        \mathbf{\Phi}(P) = \bigotimes_{(s,a)\in\mathcal{S}\times\mathcal{A}} \mathcal{P}(s,a; P),\quad where\quad \mathcal{P}(s,a; P)\subseteq\Delta(\cS).
    \end{align}
\end{assumption}

Under the $\cS\times\cA$-rectangularity (Assumption~\ref{ass: sa}), we have the so-called robust Bellman equation \citep{iyengar2005robust,blanchet2023double} which gives a dynamic programming representation of robust value functions.

\begin{proposition}[Robust Bellman equation]\label{prop: robust bellman equation}
    Under Assumption~\ref{ass: sa}, for any transition $P=\{P_h\}_{h=1}^H\subseteq\mathcal{P}$ and any policy $\pi=\{\pi_h\}_{h=1}^H$ with $\pi_h:\mathcal{S}\mapsto\Delta(\mathcal{A})$, it holds that for any $(s,a,h)\in\cS\times\cA\times[H]$,
    \begin{align}
        V_{h, P, \mathbf{\Phi}}^{\pi}(s) = \mathbb{E}_{\pi_h(\cdot|s)}\big[Q_{h, P, \mathbf{\Phi}}^{\pi}(s, \cdot)\big],\quad Q_{h, P, \mathbf{\Phi}}^{\pi}(s, a) = R_h(s,a) + \mathbb{E}_{\mathcal{P}(s,a;P_h)}\big[V_{h+1, P, \mathbf{\Phi}}^{\pi}\big].\label{eq: robust bellman V Q}
    \end{align}
\end{proposition}

For the robust value functions of the optimal robust policy, we also have the following dynamic programming solution which plays a key role in our algorithm design and theoretical analysis.

\begin{proposition}[Robust Bellman optimal equation]\label{prop: robust bellman optimal equation}
    Under Assumption~\ref{ass: sa}, for any $P=\{P_h\}_{h=1}^H\subseteq\mathcal{P}$, the robust value functions of any optimal robust policy of $P$ satisfy that, for any $(s,a,h)\in\cS\times\cA\times[H]$,
    \begin{align}
        V_{h, P, \mathbf{\Phi}}^{\star}(s) = \max_{a\in\cA}Q_{h, P, \mathbf{\Phi}}^{\star}(s, a),\quad
        Q_{h, P, \mathbf{\Phi}}^{\star}(s, a) = R_h(s,a) +\mathbb{E}_{\mathcal{P}(s,a;P_h)}\big[V_{h+1, P, \mathbf{\Phi}}^{\star}\big].\label{eq: robust bellman optimal V Q}
    \end{align}
    By taking $\pi^{\star}_h(\cdot|s) = \argmax_{a\in\cA}Q_{h, P, \mathbf{\Phi}}^{\star}(s, a)$, then $\pi^{\star} = \{\pi^{\star}_h\}_{h=1}^H$ is an optimal robust policy under $P$.
\end{proposition}

We remark that the original version of the robust Bellman equation \citep{iyengar2005robust} is for infinite horizon RMDPs and a customized proof of robust Bellman equation for finite horizon RMDPs (Proposition~\ref{prop: robust bellman equation}) can be found in Appendix A.1 of \cite{blanchet2023double}.
The robust Bellman optimal equation (Proposition~\ref{prop: robust bellman optimal equation}) is then a corollary or can be directly proved in a similar manner.

\paragraph{Total-variation distance robust set.}
In Assumption~\ref{ass: sa}, the robust set $\cP(s,a;P)$ is often modeled as a ``distribution ball" centered at $P(\cdot|s,a)$.
In this paper, we mainly consider this type of robust sets specified by a \emph{total-variation distance} ball.
We put it in the following definition.

\begin{definition}[Total-variation distance robust set]\label{def: tv}
    Total-variation distance (TV) robust set is defined as
    \begin{align}
        \mathcal{P}_{\rho}(s,a; P) := \left\{\widetilde{P}(\cdot)\in\Delta(\mathcal{S}):D_{\mathrm{TV}}\big(\widetilde{P}(\cdot)\big\|P(\cdot|s,a)\big)\leq \rho\right\},
        \end{align}
    for some $\rho\in[0,1)$, where $D_{\mathrm{TV}}(\cdot\|\cdot)$ denotes the total variation distance defined as
    \begin{align}
        D_{\mathrm{TV}}\big(p(\cdot)\|q(\cdot)\big) := \frac{1}{2}\sum_{s\in\mathcal{S}}\big|p(s) - q(s)\big|,\quad\forall p(\cdot),q(\cdot)\in\Delta(\cS).\label{eq: tv}
    \end{align}
\end{definition}

Throughout the paper, when $\rho=0$ we use the convention $1/0:=+\infty$. Hence $\min\{H,\rho^{-1}\}=H$ at $\rho=0$; products such as $\rho\min\{H,\rho^{-1}\}$ are interpreted as $0$.

The TV robust set has recently been extensively studied by \cite{yang2022toward, panaganti2022sample, panaganti2022robust, xu2023improved, blanchet2023double, shi2023curious}, which all focus on robust RL with a generative model or with a pre-collected offline dataset.
Our work follows this RMDP setup and studies robust RL via interactive data collection (see Section~\ref{subsec: interactive data}).

More importantly, we emphasize that by \eqref{eq: tv} in Definition~\ref{def: tv}, we \emph{do not} define the TV distance through the notion of $f$-divergence which requires that the distribution $p$ is absolutely continuous w.r.t. $q$, as is generally adopted by the above previous works on RMDP with TV robust sets.
According to \eqref{eq: tv}, we \textcolor{blue!75}{\emph{allow $p$ to have a different support than $q$}}.
That is, there might exist an $s\in\cS$ such that $p(s)>0$ and $q(s)=0$.
Given that, the TV robust set in Definition~\ref{def: tv} could contain transition probabilities that have different supports than the nominal transition probability $P^{\star}(\cdot|s,a)$.

An essential property of the TV robust set is that the robust expectation involved in the robust Bellman equations (Propositions~\ref{prop: robust bellman equation} and \ref{prop: robust bellman optimal equation}) has a duality representation that only uses the expectation under the nominal transition kernel.
Previous works, e.g., \cite{yang2022toward}, have proved such a result when the TV distance is defined through $f$-divergence.
Here we extend such a result to the TV distance defined directly through \eqref{eq: tv} that allows different supports between $p$ and $q$.

\begin{proposition}[Strong duality representation]\label{prop: strong duality}
    Under Definition~\ref{def: tv}, the following duality representation for the robust expectation holds,
    for any $V:\mathcal{S}\mapsto[0,H]$ and $P_h:\cS\times\cA\mapsto\Delta(\cS)$,
    \begin{align}
        \mathbb{E}_{\mathcal{P}_{\rho}(s,a;P_h)}\big[V\big]
        = \sup_{\eta\in[0,H]} \left\{ - \mathbb{E}_{P_h(\cdot|s,a)}\big[(\eta-V)_+\big] - \rho\cdot \left(\eta-\min_{s\in\cS} V(s)\right)_+ + \eta \right\}.\label{eq: strong duality main}
    \end{align}
\end{proposition}

\begin{proof}[Proof of Proposition~\ref{prop: strong duality}]
    Please refer to Appendix~\ref{subsec: proof prop strong duality} for a detailed proof of Proposition~\ref{prop: strong duality}.
\end{proof}

\begin{remark}\label{rmk: extension}
    Despite all previous works on RMDPs with TV robust sets relying on the definition of TV distance $D_{\mathrm{TV}}(p(\cdot)\|q(\cdot))$ with absolute continuity of $p$ with respect to $q$ to obtain the strong duality representation in the form of \eqref{eq: strong duality main}, their results can be directly extended to TV distance that allows for different support between $p$ and $q$ thanks to Proposition~\ref{prop: strong duality}.
\end{remark}

Finally, another useful property of the robust value functions of an RMDP with TV robust sets is a fine characterization of the gap between the maximum and the minimum of the robust value function,
which is first identified and utilized by \cite{shi2023curious} for an infinite horizon RMDP with TV robust sets.
In this work, we prove and use a similar result for the finite horizon case, concluded in the following proposition.

\begin{proposition}[Gap between maximum and minimum]\label{prop: gap}
    Under Assumption~\ref{ass: sa} with the robust set specified by Definition~\ref{def: tv}, the robust value functions satisfy that
    \begin{align}
        \max_{(s,a)\in\cS\times\cA} Q_{h,P,\boldsymbol{\Phi}}^{\pi} (s,a) - \min_{(s,a)\in\cS\times\cA}Q_{h,P,\boldsymbol{\Phi}}^{\pi} (s,a) &\leq \min\big\{H,\rho^{-1}\big\},\\
        \max_{s\in\cS} V_{h,P,\boldsymbol{\Phi}}^{\pi} (s) - \min_{s\in\cS}V_{h,P,\boldsymbol{\Phi}}^{\pi} (s) &\leq \min\big\{H,\rho^{-1}\big\},
    \end{align}
    for any transition $P = \{P_h\}_{h=1}^H\subset\cP$, any policy $\pi$, and any step $h\in[H]$.
\end{proposition}

\begin{proof}[Proof of Proposition~\ref{prop: gap}]
    Please refer to Appendix~\ref{subsec: proof prop gap} for a detailed proof of Proposition~\ref{prop: gap}.
\end{proof}

We note that in the proof of Proposition~\ref{prop: gap}, we actually show a tighter form of bound of the gap between the maximum and minimum as
$$
    \frac{1}{\rho}\cdot\Big(1 - (1-\rho)^{H}\Big).
$$
But in the sequel, we mainly use the form of $\min\{H,\rho^{-1}\}$ for its brevity and the fact of $(1-(1-\rho)^H)/\rho = \Theta(\min\{H,\rho^{-1}\})$ in the sense that
$$
    c\cdot \min\big\{H,\rho^{-1}\big\}\leq (1-(1-\rho)^H)/\rho \leq \min\big\{H,\rho^{-1}\big\}
$$
for any $H\geq H_0\in\mathbb{N}_+$ and $\rho\in[0,1)$ with some absolute constant $c>0$ that is independent of $(H,\rho)$; when $\rho=0$, the ratio $(1-(1-\rho)^H)/\rho$ is interpreted as its limit $H$.

In contrast with a crude bound of $H$, such a fine upper bound decreases when $\rho$ is large, which is essential to understanding the statistical limits of doing robust RL in RMDPs with TV robust sets.

\subsection{Robust RL with Interactive Data Collection}\label{subsec: interactive data}

\vspace{3mm}
\noindent In this paper, we study how to learn the optimal robust policy $\pi^{\star}$ in \eqref{eq: optimal robust policy} from interactive data collection.
Specifically, the learner is required to interact with \emph{only} the \emph{training environment}, i.e., $P^{\star}$, for some $K\in\mathbb{N}$ episodes.
In each episode $k\in[K]$, the learner adopts a policy $\pi^k$ to interact with the training environment $P^{\star}$ and to collect data.
When the $k$-th episode ends, the learner updates its policy to $\pi^{k+1}$ based on historical data and proceeds to the subsequent $k+1$-th episode. The learning process ends after a total of $K$ episodes.

\paragraph*{Sample complexity.} We use the notion of \emph{sample complexity} as the key evaluation metric.
For any given algorithm and predetermined accuracy level $\varepsilon>0$, the sample complexity is the minimum number of episodes $K$ required for the algorithm to output an $\varepsilon$-optimal robust policy $\widehat{\pi}$ which satisfies
\begin{align}
    V_{1,P^{\star},\mathbf{\Phi}}^{\star}(s_1)-V_{1,P^{\star},\mathbf{\Phi}}^{\widehat{\pi}}(s_1) \le \varepsilon.
\end{align}
The goal is to design algorithms whose sample complexity has small or even optimal dependence on $S, A, H, \rho$, and $1/\varepsilon$.
Such a metric is connected with the sample complexity used in robust RL with generative models and offline settings (see \hyperref[subsec: related works]{related works} for the references), wherein the sample complexity means the minimum number of generative samples or pre-collected offline data required to achieve $\varepsilon$-optimality.
In contrast, here the sample complexity is measuring the least number of interactions with the training environment needed to learn $\pi^{\star}$, where no generative or offline sample is available.
Such a learning protocol casts unique challenges on the algorithmic design and theoretical analysis to get the optimal sample complexity.

\paragraph{Online regret.}
Another evaluation metric that is related to the minimization of sample complexity is the \emph{online regret}.
For online RL in standard non-robust MDPs, the notion of regret refers to the cumulative gaps between the non-robust optimal value functions and the non-robust value functions of the policies executed during each episode \citep{auer2008near}.
Here for robust RL in RMDPs, we similarly define the regret as the cumulative difference between the optimal robust policy $\pi^\star$ and the executed policies $\{\pi^k\}_{k=1}^K$, but in terms of their robust value functions $V_{1, P^{\star},\boldsymbol{\Phi}}^{\pi}$. Its formal definition is given as follows:
\begin{align}\label{eq: regret}
    \mathrm{Regret}_{\boldsymbol{\Phi}}(K) := \sum_{k=1}^KV_{1,P^{\star},\mathbf{\Phi}}^{\star}(s_1)-V_{1,P^{\star},\mathbf{\Phi}}^{\pi^k}(s_1).
\end{align}
The goal is to design algorithms that can achieve a sublinear-in-$K$ regret  with small dependence on $S,A,H, \rho$.
Intuitively, a sublinear-regret algorithm would approximately learn the optimal robust policy $\pi^{\star}$ purely from interacting with the training environment $P^{\star}$.
It turns out that any sublinear-regret algorithm can be easily converted to a polynomial-sample complexity algorithm by applying the standard online-to-batch conversion \citep{jin2018q}, which we show in detail in our theoretical analysis part.

\section{A Hardness Result: The Curse of Support Shift}\label{sec: hardness}

Unfortunately, we show in this section that in general such a problem of robust RL with online data collection is \emph{impossible} -- there exists a simple class of two RMDPs such that any algorithm suffers an $\Omega(K)$ online regret lower bound.
However, previous works on robust RL with a generative model or offline data with good coverage do provide sample-efficient ways to find the optimal robust policy for this class of RMDPs.
This is a separation between robust RL with interactive data collection and generative model/offline data.

We first explicitly present the hard example, which is a two-state, two-action RMDP with total-variation distance robust set.
Please see also Figure~\ref{fig:hardmdp} for an illustration of the example.

\begin{example}[Hard example of robust RL with interactive data collection]\label{exp: hard}
    Consider two RMDPs $\cM_0$ and $\cM_1$ which only differ in their nominal transition kernels.
    The state space is $\cS = \{s_{\mathrm{good}}, s_{\mathrm{bad}}\}$, and the action space is $\cA = \{0,1\}$.
    The horizon length $H=3$.
    The reward function $R$  is always $1$ at the good state $s_{\mathrm{good}}$ and is $0$ at the bad state $s_{\mathrm{bad}}$, i.e.,
    \begin{align}
        R_h(s,a) = \left\{\begin{aligned}
            &1,\quad s = s_{\mathrm{good}}\\
             &0,\quad s = s_{\mathrm{bad}}
        \end{aligned}\right.,
        \quad \forall (a,h)\in\cA\times[H].
    \end{align}
    For the good state $s_{\mathrm{good}}$, the next state is always $s_{\mathrm{good}}$.
    For the bad state $s_{\mathrm{bad}}$, there is a chance to get to the good state $s_{\mathrm{good}}$, with the transition probability depending on the action it takes.
    Formally,
    \begin{align}
P_h^{\star, \cM_{\theta}}(s_{\mathrm{good}}|s_{\mathrm{good}}, a)&=1,\quad \forall (a,h)\in\cA\times\{1,2\},\quad \forall \theta\in\{0,1\},\\
P_2^{\star, \cM_{\theta}}(s_{\mathrm{good}}|s_{\mathrm{bad}}, a) &= \left\{\begin{aligned}
            &p,\quad a = \theta \\
             &q,\quad a=1-\theta
        \end{aligned}\right.,
        \quad \forall \theta\in\{0,1\},
    \end{align}
    where $p,q$ are two constants satisfying $0<q<p<1$.
    Intuitively, when at the bad state, the optimal action would result in a higher transition probability $p$ to the good state than the transition probability $q$ induced by the other action.
    Finally, we consider the robust set being specified by a total-variation distance ball centered at the nominal transition kernel, that is, for any $P$,
    \begin{align}\label{eq: tv robust set}
        \mathbf{\Phi}(P) =\!\! \bigotimes_{(s,a)\in\mathcal{S}\times\mathcal{A}} \mathcal{P}_{\rho}(s,a; P),\quad \text{where}\quad \mathcal{P}_{\rho}(s,a; P) = \left\{\widetilde{P}(\cdot)\in\Delta(\mathcal{S}):D_{\mathrm{TV}}\big(\widetilde{P}(\cdot)\big\|P(\cdot|s,a)\big)\leq \rho\right\},
    \end{align}
    where $\rho\in[0,q]$ is the parameter characterizing the size of the robust set.
    We set $s_1 = s_{\mathrm{good}}$.
\end{example}

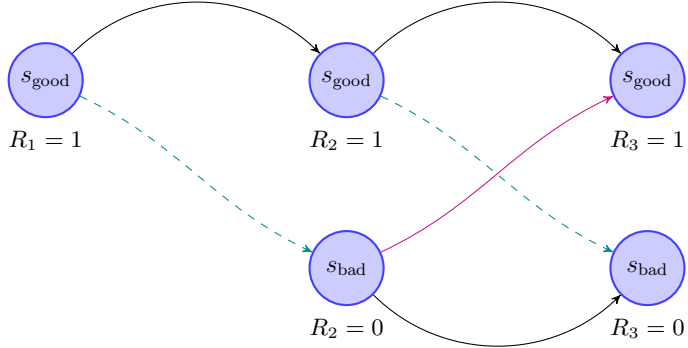
\begin{figure}[!t]
    \centering
    \begin{tikzpicture}[node distance=1.5cm,>=stealth',bend angle=45,auto]

  \tikzstyle{place}=[circle,thick,draw=blue!75,fill=blue!20,minimum size=6mm]
  \tikzstyle{hold}=[circle,draw=white,fill=white,minimum size=6mm]
  \tikzstyle{absorb}=[rectangle,rounded corners=.1cm, thick,draw=blue!75,fill=blue!20,minimum height=6mm,minimum width=9.6cm]
\tikzset{every loop/.style={min distance=15mm, font=\footnotesize}}
  \begin{scope}[xshift=-3.5cm]
        \node[place, text=black, minimum width =35pt, minimum height =35pt, scale = 0.8] (0) at  (0, 0) {\large $s_{\mathrm{good}}$};

    \node[place, text=black, minimum width =35pt, minimum height =35pt, scale = 0.8] (1) at  (4,0) {\large $s_{\mathrm{good}}$};

    \node[place, text=black, minimum width =35pt, minimum height =35pt, scale = 0.8] (2) at  (8,0) {\large $s_{\mathrm{good}}$};

    \node[place, text=black, minimum width =35pt, minimum height =35pt, scale = 0.8] (3) at  (4, -2.5) {\large $s_{\mathrm{bad}}$};

    \node[place, text=black, minimum width =35pt, minimum height =35pt, scale = 0.8] (4) at  (8,-2.5) {\large $s_{\mathrm{bad}}$};

    \draw[->] (0) edge [out=45, in=135, draw=black] (1);

    \draw[->] (1) edge [out=45, in=135, draw=black] (2);

    \draw[->] (3) edge [out=315, in=225, draw=black] (4);

    \draw[->] (3) edge [out=25, in=205, draw=luolan] (2);

    \draw[->] (0) edge [out=335, in=155, dashed, draw=tianqin] (3);

    \draw[->] (1) edge [out=335, in=155, dashed, draw=tianqin] (4);

    \node[text=black] at  (0, -0.8) {\small $R_1 = 1$};

    \node[text=black] at  (4, -0.8) {\small $R_2 = 1$};

    \node[text=black] at  (8, -0.8) {\small $R_3 = 1$};

    \node[text=black] at  (4, -3.3) {\small $R_2 = 0$};

    \node[text=black] at  (8, -3.3) {\small $R_3 = 0$};
  \end{scope}

\end{tikzpicture}
\vspace*{-1mm}
    \caption{
Illustration of the hard example in Example~\ref{exp: hard}.
The solid lines represent possible transitions of the nominal transition kernel.
The \textcolor{tianqin}{dashed lines} represent the transitions induced by the worst case transition kernel in the robust set.
The \textcolor{luolan}{red solid line} represents the transition where the two RMDP instances differ in that different actions lead to higher transition probability from $s_{\mathrm{bad}}$ to $s_{\mathrm{good}}$.
We notice that when starting from $s_1 = s_{\mathrm{good}}$, the nominal transition kernel keeps the agent at $s_{\mathrm{good}}$ and no information at $s_{\mathrm{bad}}$ is revealed.
	}
    \label{fig:hardmdp}
\end{figure}

For this class of RMDPs, we have the following hardness result for doing robust RL with interactive data collection, an $\Omega(\rho\cdot K)$-online regret lower bound.

\begin{theorem}[Hardness result (based on Example~\ref{exp: hard})]\label{thm: hard example}
    There exist two RMDPs $\{\cM_0, \cM_1\}$ such that the following regret lower bound holds:
    \begin{align}
        \inf_{\mathcal{ALG}}\sup_{\theta\in\{0,1\}}\mathbb{E}\left[\mathrm{Regret}^{\cM_{\theta},\mathcal{ALG}}_{\boldsymbol{\Phi}}(K)\right] \geq  \Omega\big(\rho \cdot HK\big),
    \end{align}
    where $\mathrm{Regret}^{\cM_{\theta},\mathcal{ALG}}_{\boldsymbol{\Phi}}(K)$ refers to the online regret of algorithm $\mathcal{ALG}$ for RMDP $\mathcal{M}_{\theta}$.
\end{theorem}

\begin{proof}[Proof of Theorem~\ref{thm: hard example}]
    We intuitively explain why robust RL with interactive data collection may fail in the Example~\ref{exp: hard} in this section. We refer the readers to a rigorous proof of Theorem~\ref{thm: hard example} in Appendix~\ref{subsec: proof thm hard example}.
\end{proof}

The reason why any algorithm fails for this class of RMDPs is the \emph{support shift} of the worst-case transition kernel.
In robust RL, the performance of a policy $\pi$ is evaluated via the robust expected total rewards, or equivalently, the expected return under the most adversarial transition kernel $P^{\dagger,\pi}$.
In this example, as we explicitly show in the proof, when in the good state $s_{\mathrm{good}}$, the worst-case transition kernel $P^{\dagger,\pi}$ would transit the state to $s_{\mathrm{bad}}$ with a constant probability $\rho$.
But the state $s_{\mathrm{bad}}$ is out of the scope of the data collection process because starting from $s_1=s_{\mathrm{good}}$ the nominal transition kernel always transits the state to $s_{\mathrm{good}}$.
As a result, the performance of the learned policy at the bad state $s_{\mathrm{bad}}$ is not guaranteed, and inevitably incurs an $\Omega(\rho\cdot K)$-lower bound of regret, a hardness result. Furthermore, by strategically constructing RMDPs with the horizon $3H$ based on Example~\ref{exp: hard}, we can derive a lower bound of $\Omega(\rho \cdot HK)$.

In contrast, doing robust RL with a generative model or an offline dataset with good coverage properties does not face such difficulty.
It turns out that any RMDP with $\cS\times\cA$-rectangular total-variation robust set (including Example~\ref{exp: hard}) can be solved in a sample-efficient manner therein, see \cite{yang2022toward, panaganti2022sample, panaganti2022robust,  xu2023improved, blanchet2023double, shi2023curious} and Remark~\ref{rmk: extension}.
The intuitive reason is that, for the generative model setting, the learner can directly query any state-action pair to estimate the nominal transition kernel $P^{\star}$, and thus no support shift problem happens.
The same reason holds for the offline setup with a good-coverage dataset.

There is a broader understanding of the curse of  support shift that hinders the tractability of robust RL via interactive data collection.
The concept of support shift can be comprehended within a broader context beyond the disjointness of certain parts of the support sets of the training and testing environments.
Instead, ensuring a ``high probability of disjointness" is enough to maintain the integrity of the hardness result. For instance, we can modify the state $s_{\mathrm{good}}$ in Example~\ref{exp: hard} so that it is no longer an absorbing state. Rather, $s_{\mathrm{good}}$ could transit to $s_{\mathrm{bad}}$ with a small probability, such as $2^{-H}$.
This modification expands the support of the training environment to encompass the entire state space.
Nevertheless, acquiring information about $s_{\mathrm{bad}}$ necessitates exponential samples, thereby preserving the hardness result.

In the next section of this paper, we aim to figure out that for specific types of RMDPs, e.g., the RMDP with total-variation robust set as in Example~\ref{exp: hard}, under what kind of structural assumptions can we perform sample-efficient robust RL with interactive data collection.

\section{A Solvable Case, Efficient Algorithm, and Sharp Analysis}\label{sec: tv}

Motivated by the hard instance (Example~\ref{exp: hard}) in the previous section, in this section, we consider a special subclass of RMDP with $\cS\times\cA$-rectangular total variation robust set that we show allows for sample-efficient robust RL through interactive data collection.
In Section~\ref{subsec: fail state assumption}, we introduce the assumption we impose on the RMDP we consider.
We propose our algorithm design in Section~\ref{subsec: alg tv}, with theoretical analysis in Section~\ref{subsec: theory tv}.
Throughout this section, our choice of the mapping $\boldsymbol{\Phi}$ is always given by \eqref{eq: tv robust set}.

\subsection{Vanishing Minimal Value: Eliminating Support Shift}\label{subsec: fail state assumption}

To overcome the difficulty of support shift identified in Section~\ref{sec: hardness}, we make the following \emph{vanishing minimal value} assumption on the underlying RMDP.

\begin{assumption}[Vanishing minimal value]\label{ass: zero min}
    We assume that the underlying RMDP satisfies that $$\min_{s\in\cS}V_{1,P^{\star},\boldsymbol{\Phi}}^{\star}(s) = 0.$$
    Also, without loss of generality, we assume that the initial state $s_1\notin \argmin_{s\in\cS}V_{1,P^{\star},\boldsymbol{\Phi}}^{\star}(s)$.
\end{assumption}

Assumption~\ref{ass: zero min} imposes that the minimal robust expected total rewards over all possible initial states is $0$.
Assuming that the initial state $s_1\notin \argmin_{s\in\cS}V_{1,P^{\star},\boldsymbol{\Phi}}^{\star}(s)$ avoids making the problem trivial.
A close look at Assumption~\ref{ass: zero min} actually gives that the minimal robust value function of any policy $\pi$ at any step is zero, that is, $\min_{s\in\cS}V_{h,P^{\star},\boldsymbol{\Phi}}^{\pi}(s) = 0$ for any policy $\pi$ and any step $h\in[H]$.
With this observation, the following proposition explains why such an assumption can help to overcome the difficulty.

\begin{proposition}[Equivalent expression of TV robust set with vanishing minimal value] \label{prop: equivalent robust set}
    For any function $V : \cS \mapsto [0, H]$ with $\min_{s \in \cS}V(s) = 0$, we have that
    \$
    \EE_{\cP_\rho(s, a; P_h^\star)}\left[ V \right] = \rho' \cdot \EE_{ \cB_{\rho'}(s, a; P_h^{\star}) }   [V],\quad \text{where}\quad  \rho' = 1-\rho \in (0,1],
    \$
    where the total-variation robust set $\cP_\rho(s, a; P_h^\star)$ is defined in \eqref{eq: tv robust set} and the set $\cB_{\rho'}(s, a; P_h^{\star})$ is defined as\footnote{Here we implicitly define $\frac{0}{0} = 0$ and $\frac{a}{0}=\infty$ for any $a>0$.}
    \$
\cB_{\rho'}(s, a; P_{h}^{\star}) = \left\{\widetilde{P}(\cdot) \in \Delta(\cS):  \sup_{s' \in \cS}\frac{\widetilde{P}(s')}{P_h^\star(s' | s, a)} \le \frac{1}{\rho'} \right\}.
\$
\end{proposition}

\begin{proof}[Proof of Proposition~\ref{prop: equivalent robust set}]
    Please refer to Appendix~\ref{subsec: proof prop equivalent robust set} for a detailed proof of Proposition~\ref{prop: equivalent robust set}.
\end{proof}

As Proposition~\ref{prop: equivalent robust set} indicates, under Assumption~\ref{ass: zero min}, the robust Bellman equations (Propositions~\ref{prop: robust bellman equation} and \ref{prop: robust bellman optimal equation}) at step $h\in[H]$ are equivalent to taking an infimum over another robust set $\cB_{\rho'}(s, a; P_{h}^{\star})$ that shares the \emph{same} support as the nominal transition kernel $P^{\star}(\cdot|s,a)$, discounted by a constant $\rho'\leq 1$.
Intuitively, this new robust set rules out the difficulty originated in unseen states in training environments and the discount factor $\rho'$ hedges the difficulty from prohibitively small probability of reaching certain states that may appear often in the testing environments.
This renders robust RL with interactive data collection possible.

To understand this from another perspective, it could be shown that under the conclusions of Proposition~\ref{prop: equivalent robust set},
the robust value functions of any policy $\pi$ are equivalent to the robust value functions of this policy under another \emph{discounted} RMDP $(\cS,\cA,H, P^{\star}, R', \boldsymbol{\Phi}')$ with $R_h' (s,a) = (\rho')^{h-1}R_h(s,a)$ and $\boldsymbol{\Phi}'$ given by
\begin{align}
    \boldsymbol{\Phi}'(P) = \bigotimes_{(s,a)\in\mathcal{S}\times\mathcal{A}} \mathcal{B}_{\rho'}(s,a; P). \label{eq: bounded ratio robust set}
\end{align}
And therefore we are equivalently considering this new type of RMDPs.
Please refer to Section~\ref{subsec: extentions} for more discussions on the connections between the two types of RMDPs.

\paragraph{Examples of Assumption~\ref{ass: zero min}.} In the sequel, we provide a concrete condition that makes Assumption~\ref{ass: zero min} hold, which imposes that the state space of the RMDP has a ``closed" subset of ``fail-states" with zero rewards.

\begin{condition}[Fail-states] \label{assumption:fail:state}
    There exists a subset $\cS_f\subset\cS$ of fail states such that
    \$
    R_h(s, a) = 0, \quad P_h^\star(\cS_f| s, a) = 1, \quad \forall (s, a, h) \in \cS_f\times \mathcal{A} \times [H].
    \$
\end{condition}

This type of ``fail-states" condition is first proposed by \citet{panaganti2022robust} (with $|\cS_f|=1$) to handle the computational issues for robust offline RL under function approximations (out of the scope of our work).
In contrast, here we make the vanishing minimal value assumption in order to tackle the \emph{support shift} or \emph{extrapolation} issue for the interactive data collection setup.
The comparison between the vanishing minimal value assumption (Assumption~\ref{ass: zero min}) and the ``fail-states" condition (Condition~\ref{assumption:fail:state}) is given below.

\begin{remark}[Comparison between Assumption~\ref{ass: zero min} and Condition~\ref{assumption:fail:state}]\label{remark:assumption:compare}
We first observe that Condition~\ref{assumption:fail:state} implies that $\min_{s\in\cS} V_{h, P^\star, \boldsymbol{\Phi}}^\pi(s) = 0$ for any policy $\pi$ and step $h \in [H]$, therefore satisfying the minimal value assumption (Assumption~\ref{ass: zero min}).
Conversely, the vanishing minimal value assumption in Assumption~\ref{ass: zero min} is strictly more general than the fail-state condition in Condition~\ref{assumption:fail:state}. To illustrate, one can consider an RMDP characterized by the state space $\cS = \{s_1, s_2\}$, action space $\cA = \{a_1\}$, time horizon $H = 2$, reward function $R_h(s, a) = \mathbf{1}\{s = s_2\}$, and transition probabilities defined as follows:
    \$
    P_1^\star(s_1 |s_1, a_1) = 1 - \rho, \quad P_1^\star(s_2 | s_1, a_1) = \rho, \quad P_1^\star(s_1 | s_2, a_1) = 0, \quad P_1^\star(s_2 | s_2, a_1) = 1,
    \$
    where $\rho$ is the radius of the robust set. It is evident that no fail-state emerges within such an RMDP structure. However, this RMDP satisfies the vanishing minimal value assumption since $V_{1,P^{\star},\boldsymbol{\Phi}}^{\star}(s_1) = 0$.
\end{remark}

\begin{remark}[Reduction to non-robust MDP without loss of generality]\label{rmk: reduction}
    It is noteworthy that assuming the vanishing minimal value (Assumption~\ref{ass: zero min}) or the presence of fail-states (Condition~\ref{assumption:fail:state}) in the non-robust case ($\rho = 0$) is without loss of generality. This is achievable by expanding the prior state space $\mathcal{S}$ of MDP to include an additional state $s_f$, denoted as the fail-state. More importantly, this augmentation does not alter the optimal value or the optimal value function of the original MDP. Consequently, it becomes sufficient to seek the optimal policy within the augmented MDP, which satisfies the conditions of vanishing minimal value (Assumption~\ref{ass: zero min}) or the existence of fail-states (Condition~\ref{assumption:fail:state}).
    This indicates that our algorithm and theoretical analysis in the sequel can be directly reduced to non-robust MDPs without additional assumptions.
\end{remark}

\subsection{Algorithm Design: OPROVI-TV}\label{subsec: alg tv}

In this section, we propose our algorithm that solves robust RL with interactive data collection for RMDPs with $\cS\times\cA$-rectangular total-variation (TV) robust sets (Assumption~\ref{ass: sa} and Definition~\ref{def: tv}) and satisfying the vanishing minimal value assumption (Assumption~\ref{ass: zero min}).
Our algorithm, \underline{OP}timistic \underline{RO}bust \underline{V}alue \underline{I}teration for \underline{TV} Robust Set (\texttt{OPROVI-TV}, Algorithm~\ref{alg: tv}), can automatically balance exploitation and exploration during the interactive data collecting process while managing the distributional robustness of the learned policy.

In each episode $k$, the algorithm operates in three stages: (i) training environment transition estimation (Line~\ref{line: part 1 start} to \ref{line: part 1 end}); (ii) optimistic robust planning based on the training environment transition estimator (Line~\ref{line: part 2 start} to \ref{line: part 2 end}); and finally (iii) executing the policy in the training environment and collecting data (Line~\ref{line: part 3 start} to \ref{line: part 3 end}).
In the following, we elaborate more on the first two parts of Algorithm~\ref{alg: tv}.

\begin{algorithm}[t]
    \caption{\underline{OP}timistic \underline{RO}bust \underline{V}alue \underline{I}teration for TV Robust Set (\texttt{OPROVI-TV})}\label{alg: tv}
    \begin{algorithmic}[1]
        \STATE \textbf{Initialize:} dataset $\mathbb{D} = \emptyset$.
        \FOR{episode $k=1,\cdots,K$}
        \STATE \textcolor{blue!55}{\texttt{Training environment transition estimation:}}\label{line: part 1 start}
        \STATE Update the count functions $N_h^k(s,a,s')$ and $N_h^k(s,a)$ based on $\mathbb{D}$ according to \eqref{eq: count function}.
        \STATE Calculate the transition kernel estimator $\widehat{P}_h^k$ according to \eqref{eq: hat p tv}.\label{line: part 1 end}
        \STATE \textcolor{blue!55}{\texttt{Optimistic robust planning:}}\label{line: part 2 start}
        \STATE Set $\overline{V}_{H+1}^k = \underline{V}_{H+1}^k = 0$.
        \FOR{step $h=H,\cdots,1$}
        \STATE Set $\overline{Q}_h^k(\cdot,\cdot)$ and $\underline{Q}_h^k(\cdot,\cdot)$ as \eqref{eq: Q overline} and \eqref{eq: Q underline}, with the bonus function $\texttt{bonus}_h^k(\cdot,\cdot)$ defined in \eqref{eq: bernstein bonus}.
        \STATE Set $\pi_h^k(\cdot|\cdot) = \argmax_{a\in\mathcal{A}}\,\overline{Q}_h^k(\cdot,a)$, $\overline{V}_h^k(\cdot) = \mathbb{E}_{\pi_h^k(\cdot|\cdot)} [\overline{Q}_h^k(\cdot,\cdot)]$, and $\underline{V}_h^k(\cdot) =  \mathbb{E}_{\pi_h^k(\cdot|\cdot)} [ \underline{Q}_h^k(\cdot,\cdot)]$.\label{line: V}
        \ENDFOR\label{line: part 2 end}
        \STATE \textcolor{blue!55}{\texttt{Execute the policy in training environment and collect data:}}\label{line: part 3 start}
        \STATE Receive the initial state $s_1^k\in\cS$.
        \FOR{step $h=1,\cdots,H$}
        \STATE Take action $a_h^k\sim \pi^k_h(\cdot|s_h^k)$, observe reward $R_h(s_h^k,a_h^k)$ and the next state $s_{h+1}^k$.
        \ENDFOR
        \STATE Set $\mathbb{D}$ as $\mathbb{D}\cup\{(s_h^k,a_h^k,s_{h+1}^k)\}_{h=1}^H$.
        \ENDFOR\label{line: part 3 end}
        \STATE \textbf{Output:} Randomly (uniformly) return a policy from $\{\pi^k\}_{k=1}^K$.
    \end{algorithmic}
\end{algorithm}

\subsubsection{Training Environment Transition Estimation}
At the beginning of each episode $k\in[K]$, we maintain an estimate of the transition kernel $P^{\star}$ of the training environment by using the historical data $\mathbb{D} = \{(s_h^{\tau}, a_h^{\tau}, s_{h+1}^{\tau})\}_{\tau=1, h = 1}^{k-1, H}$ collected from the interaction with the training environment.
Specifically, we simply adopt a vanilla empirical estimator, defined as
\begin{align}\label{eq: hat p tv}
    \widehat{P}_h^k(s'|s,a)
    =
    \begin{cases}
    \dfrac{N_h^k(s,a,s')}{N_h^k(s,a)}, & N_h^k(s,a)>0,\\[1ex]
    \dfrac{1}{S}, & N_h^k(s,a)=0,
    \end{cases}
    \quad \forall (s,a,h,s')\in\cS\times\cA\times[H]\times\cS,
\end{align}
where the count functions $N_h^k(s,a,s')$ and $N_h^k(s,a)$ are calculated on the current dataset $\mathbb{D}$ by
\begin{align}
    N_h^k(s,a,s')= \sum_{\tau=1}^{k-1}\mathbf{1}\big\{(s_h^{\tau}, a_h^{\tau},s_{h+1}^{\tau}) = (s,a,s')\big\}, \quad N_h^k(s,a) = \sum_{s'\in\cS}N_h^k(s,a,s'),\label{eq: count function}
\end{align}
for any $(s,a,h,s')\in\cS\times\cA\times[H]\times\cS$.
This just coincides with the transition estimator adopted by existing non-robust online RL algorithms \citep{auer2008near, azar2017minimax, zhang2021reinforcement}.

\subsubsection{Optimistic Robust Planning}

Given $\widehat{P}^k$ that estimates the training environment, we perform an  optimistic robust planning to construct the policy $\pi^k$ to execute.
Basically, the optimistic robust planning follows the robust Bellman optimal equation (Proposition~\ref{prop: robust bellman optimal equation}) to approximate the optimal robust policy, but differs in that it maintains an upper bound and a lower bound of the optimal robust value function and chooses the policy that maximizes the optimistic estimate to incentivize exploration during data collection.
Here the purpose of maintaining the lower bound estimate is to facilitate the construction of the variance-aware optimistic bonus (see following), which helps to sharpen our theoretical analysis.

\paragraph{Simplifying the robust expectation.}
To better utilize the vanishing minimal value condition (Assumption~\ref{ass: zero min}), we take a closer look into the robust Bellman equation.
Due to the strong duality (Proposition~\ref{prop: strong duality}), the robust expectation $\mathbb{E}_{\mathcal{P}_{\rho}(s,a;P)}[V]$ for any $V\in[0,H]$ satisfying $\min_{s\in\cS}V(s) = 0$ is equivalent to
\begin{align}
    \mathbb{E}_{\mathcal{P}_{\rho}(s,a;P)}\big[V\big] &= \sup_{\eta\in[0,H]} \bigg\{ - \mathbb{E}_{P(\cdot|s,a)}\Big[\big(\eta-V\big)_+\Big] - \rho\cdot\Big(\eta-\min_{s'\in\cS} V(s')\Big)_+ + \eta \bigg\} \\
    &=\sup_{\eta\in[0,H]} \bigg\{- \mathbb{E}_{P(\cdot|s,a)}\Big[\big(\eta-V\big)_+\Big] + \left(1-\rho\right) \cdot \eta\bigg\}.\label{eq: duality tv}
\end{align}
{Consequently, in the algorithmic recursions below, whenever the notation $\mathbb{E}_{\mathcal{P}_{\rho}(s,a;P)}[V]$ is applied to an estimated value function, it is understood as the dual-form backup on the right hand side of \eqref{eq: duality tv}. When $\min_{s\in\cS}V(s)=0$, Proposition~\ref{prop: strong duality} implies that this dual-form backup coincides with the true TV robust expectation. Thus, under Assumption~\ref{ass: zero min}, the robust Bellman equations for the true robust value functions remain unchanged, while the same dual-form backup is used for the optimistic and pessimistic value estimates in the algorithm.}

\paragraph{Optimistic robust planning.}
With this in mind, the optimistic robust planning goes as follows.
Starting from $\overline{V}_{H+1}^k = \underline{V}_{H+1}^k = 0$, we recursively define that
\begin{align}
    \overline{Q}_h^k(s,a) &= \min\left\{R_h(s,a) + \mathbb{E}_{\mathcal{P}_{\rho}(s,a;\widehat{P}_h^k)}\Big[\overline{V}_{h+1}^k\Big] + \texttt{bonus}_h^k(s,a),\min\big\{H, \rho^{-1}\big\}\right\},\quad \forall (s,a)\in\cS\times\cA,\label{eq: Q overline}\\
    \underline{Q}_h^k(s,a) &= \max\left\{R_h(s,a) +  \mathbb{E}_{\mathcal{P}_{\rho}(s,a;\widehat{P}_h^k)}\Big[\underline{V}_{h+1}^k\Big] - \texttt{bonus}_h^k(s,a), 0\right\},\quad \forall (s,a)\in\cS\times\cA,\label{eq: Q underline}
\end{align}
{where $\mathbb{E}_{\mathcal{P}_{\rho}(s,a;\widehat{P}_h^k)}$ is understood in the dual-form sense specified above, and the bonus function $\texttt{bonus}_h^k(s,a)\geq 0$ is defined later.}
Here we truncate the optimistic estimate $\overline{Q}_h^k$ via the upper bound $\min\{H,\rho^{-1}\}$ of the true optimal robust value function $Q^{\star}_{h,P^{\star}, \boldsymbol{\Phi}}$. This truncation arises from the combined implication of Proposition~\ref{prop: gap} and the fact that $\min_{(s,a)\in\cS\times\cA}Q_{h,P^{\star},\boldsymbol{\Phi}}^{\star}(s,a) = 0$ under Assumption~\ref{ass: zero min}.

As we establish in Lemma~\ref{lem: optimism and pessimism restate}, $\overline{Q}^k_h$ and $\underline{Q}^k_h$ form upper and lower bounds for $Q^{\star}_{h,P^{\star},\boldsymbol{\Phi}}$ and $Q^{\pi^k}_{h,P^{\star},\boldsymbol{\Phi}}$ under a proper choice of the bonus.
After performing \eqref{eq: Q overline} and \eqref{eq: Q underline}, we choose the data collection policy $\pi^k_h$ to be the optimal policy with respect to the optimistic estimator $\overline{Q}^k_h$ and define $\overline{V}_h^k$ and $\underline{V}_h^k$ accordingly by
\begin{align}
    \pi_h^k(\cdot|\cdot) = \argmax_{a\in\mathcal{A}}\,\overline{Q}_h^k(\cdot,a), \quad \overline{V}_h^k(s) = \mathbb{E}_{\pi_h^k(\cdot|s)} \Big[\overline{Q}_h^k(s,\cdot)\Big],\quad \underline{V}_h^k(s) =  \mathbb{E}_{\pi_h^k(\cdot|s)} \Big[ \underline{Q}_h^k(s,\cdot)\Big].\label{eq: V}
\end{align}
We remark that the purpose of maintaining the lower bound estimate \eqref{eq: Q underline} is to facilitate the construction of the bonus and to help to sharpen our theoretical analysis.
The construction of the policy $\pi^k$ is still based on the optimistic estimator, which is why we name it optimistic robust planning.
As indicated by theory, the optimistic robust planning effectively guides the policy to explore uncertainty \emph{robust} value function estimates, striking a balance between exploration and exploitation while managing distributional robustness.

\paragraph{Bonus function.}
In Algorithm~\ref{alg: tv}, the bonus function $\texttt{bonus}_h^k(s,a)$ is a Bernstein-style bound defined as
\begin{align}
\!\!\!\!\!\!\!\!\!\!\!\!\texttt{bonus}_h^k(s,a)=\sqrt{\frac{\mathbb{V}_{\widehat{P}_h^k(\cdot|s,a)}\Big[\Big(\overline{V}_{h+1}^k+\underline{V}_{h+1}^k\Big) / 2\Big] c_1\iota}{N_h^k(s,a)\vee 1}} + \frac{2\mathbb{E}_{\widehat{P}_h^k(\cdot|s,a)}\Big[\overline{V}_{h+1}^{k} - \underline{V}_{h+1}^{k}\Big]}{H}  + \frac{c_2H^2S\iota}{N_h^k(s,a)\vee 1} + \frac{1}{\sqrt{K}}\label{eq: bernstein bonus}
\end{align}
where $\iota = \log(S^3AH^2K^{3/2}/\delta)$, $c_1,c_2>0$ are absolute constants, and $\delta$ signifies a pre-selected fail probability.
Under \eqref{eq: bernstein bonus}, $\overline{Q}_h^k$ and $\underline{Q}_h^k$ become upper and lower bounds of the optimal robust value functions (Lemma~\ref{lem: optimism and pessimism restate}).
More importantly, the bonus \eqref{eq: bernstein bonus} is carefully designed for robust value functions such that the summation of this bonus term (especially the leading variance term in \eqref{eq: bernstein bonus}) over time steps is well controlled, for which we also develop new analysis methods.
This is critical for obtaining a sharp sample complexity of Algorithm~\ref{alg: tv}.

\subsection{Theoretical Guarantees}\label{subsec: theory tv}

This section establishes the online regret and the
sample complexity of \texttt{OPROVI-TV} (Algorithm~\ref{alg: tv}).
Our main result is the following theorem, upper bounding the online regret of Algorithm~\ref{alg: tv}.

\begin{theorem}[Online regret of \texttt{OPROVI-TV}]\label{thm: regret tv}
    Given an RMDP with $\cS\times\cA$-rectangular total-variation robust set of radius $\rho\in[0,1)$ (Assumption~\ref{ass: sa} and Definition~\ref{def: tv}) satisfying Assumptions~\ref{ass: zero min}, choosing the bonus function as \eqref{eq: bernstein bonus} with sufficiently large $c_1,c_2>0$, then  with probability at least $1-\delta$, Algorithm~\ref{alg: tv} satisfies
    \begin{align}
        \mathrm{Regret}_{\boldsymbol{\Phi}}(K) \leq \mathcal{O}\bigg(\sqrt{\min\big\{H,\rho^{-1}\big\} H^2SAK\iota'}\,\bigg),
    \end{align}
    where $\iota' = \log^2(SAHK/\delta)$
    and $\mathcal{O}(\cdot)$ hides absolute constants and lower order terms in $K$.
\end{theorem}

\begin{proof}[Proof of Theorem~\ref{thm: regret tv}]
    See Appendix~\ref{sec: proof tv} for a detailed proof of Theorem~\ref{thm: regret tv}.
\end{proof}

Theorem~\ref{thm: regret tv} shows that Algorithm~\ref{alg: tv} enjoys a sublinear online regret of $\widetilde{\cO}(\sqrt{K})$, meaning that it is able to approximately find the optimal robust policy through interactive data collection.
This is in contrast with the general hardness result in Section~\ref{sec: hardness} where sample-efficient learning is impossible in the worst case.
Thus we show the effectiveness of the minimal value assumption for robust RL with interactive data collection.

As a corollary, we have the following sample complexity bound for Algorithm~\ref{alg: tv}.

\begin{corollary}[Sample complexity of \texttt{OPROVI-TV}]\label{cor: sample complexity tv}
Under the same setup and conditions as in Theorem~\ref{thm: regret tv}, with probability at least $1-\delta$, 
    Algorithm~\ref{alg: tv} can output an $\varepsilon$-optimal policy within
    \begin{align}
    \mathcal{O}\left(\frac{\min\big\{H,\rho^{-1}\big\} H^2SA\iota''}{\varepsilon^2}\right)\label{eq: complexity}
    \end{align}
    episodes,
    where $\iota'' = \log(SAH/\varepsilon \delta)$
    and $\mathcal{O}(\cdot)$ hides absolute constants.
    Here the valid range of $\varepsilon$ is given by $\varepsilon\in(0,c\cdot \min\{1,1/(\rho H)\}]$ for some constant $c>0$.
\end{corollary}
\begin{proof}[Proof of Corollary~\ref{cor: sample complexity tv}]
    This follows from Theorem~\ref{thm: regret tv} and a standard online to batch conversion.
\end{proof}

This further shows that Algorithm~\ref{alg: tv} is able to find $\varepsilon$-optimal robust policy within polynomial interactive samples in $H$, $S$, $A$, and $\varepsilon^{-1}$.
We note that as the radius $\rho$ of the TV robust set increases, the sample needed to be $\varepsilon$-optimal decreases.
When $\rho$ tends to $1$, the sample complexity reduces to nearly $\widetilde{\cO}(H^2SA/\varepsilon^2)$.
Thus, we observe that robust RL through interactive data collection for this RMDP example is statistically easier when the radius $\rho$ increases, which matches the conclusion in the generative model setup \citep{yang2022toward, shi2023curious} as well as the offline learning setup \citep{panaganti2022robust}.

Finally, we compare the sample complexity \eqref{eq: complexity} with prior arts on non-robust online RL and robust RL with a generative model.
On the one hand, \eqref{eq: complexity} with $\rho=0$ equals to $$\widetilde{\cO}\left(\frac{H^3SA}{\varepsilon^2}\right),$$ which matches the minimax sample complexity lower bound for online RL in non-robust MDPs \citep{azar2017minimax}.
This means that our algorithm design can naturally handle non-robust MDPs as a special case (please also see Remark~\ref{rmk: reduction} for why one can reduce Algorithm~\ref{alg: tv} to general non-robust MDPs under Assumption~\ref{ass: zero min}).
On the other hand, the previous work of \cite{shi2023curious} for robust RL in infinite horizon RMDPs with a TV robust set and a generative model showcases a minimax optimal sample complexity of $$\widetilde{\mathcal{O}}\left(\frac{\min\big\{H_\gamma,\rho^{-1}\big\} H_\gamma^2 SA}{\varepsilon^2}\right),$$
for $\rho\in[0,1)$, where we define $H_\gamma := 1/(1-\gamma)$ as the effective horizon of the infinite $\gamma$-discounted RMDPs.
As a result, the sample complexity \eqref{eq: complexity} of Algorithm~\ref{alg: tv} matches their result.
We highlight that our algorithm does not rely on a generative model and operates purely through interactive data collection.

\section{Extension \texorpdfstring{I}{I}: Robust Set with Bounded Transition Probability Ratio}\label{subsec: extentions}

In this section, we show that our algorithm design (Algorithm~\ref{alg: tv}) can also be applied to $\cS\times\cA$-rectangular discounted RMDPs with robust sets given by \eqref{eq: bounded ratio robust set} (i.e., bounded ratio between training and testing transition probabilities).
We establish that our main theoretical result in Section~\ref{subsec: theory tv} can imply a sublinear regret upper bound for this model, which means that this type of RMDPs can also be solved sample-efficiently through the auxiliary construction based on Algorithm~\ref{alg: tv}.
This coincides with our intuition on support shift in Section~\ref{subsec: fail state assumption}.

\paragraph{$\cS\times\cA$-rectangular discounted RMDPs with robust set \eqref{eq: bounded ratio robust set}.}
We first define the model formally.
A finite-horizon discounted RMDP is specified by $\cM_{\gamma} = (\cS,\cA,H,P^{\star}, R_{\gamma},\boldsymbol{\Phi}')$, where the robust set $\boldsymbol{\Phi}'$ is given by \eqref{eq: bounded ratio robust set}, i.e.,
\begin{align}
     \boldsymbol{\Phi}'(P) = \bigotimes_{(s,a)\in\mathcal{S}\times\mathcal{A}} \left\{\widetilde{P}(\cdot) \in \Delta(\cS):  \sup_{s' \in \cS}\frac{\widetilde{P}(s')}{P(s' | s, a)} \le \frac{1}{\rho'} \right\}:= \bigotimes_{(s,a)\in\mathcal{S}\times\mathcal{A}} \cB_{\rho'}(s,a;P).\label{eq: robust set new}
\end{align}
This robust set contains transition probabilities that share the same support as the nominal transition kernel.
The reward function $R_{\gamma} = \{\gamma^{h-1}\cdot R_h\}_{h=1}^H$, where $\gamma\in(0,1)$ is the discount factor and $R_h\in[0,1]$ is the true reward at step $h$.
That is, the robust value function is now the worst case expected discounted total reward.

\paragraph{Algorithm and regret bound.}
Now we theoretically show that we can apply Algorithm~\ref{alg: tv} to solve robust RL in $\cS\times\cA$-rectangular discounted RMDPs with robust set \eqref{eq: robust set new} via interactive data collection.

As motivated by the discussions under Proposition~\ref{prop: equivalent robust set}, we define an auxiliary finite-horizon TV-RMDP $\widetilde{\cM}$ as
$\widetilde{\cM} = (\widetilde{\cS},\cA,H, \widetilde{P}^{\star}, \widetilde{R}, \widetilde{\boldsymbol{\Phi}})$ which includes an additional ``fail-state" $s_f$.
More specifically, the state space $\widetilde{\cS} = \cS\cup\{s_f\}$.
The transition kernel $\widetilde{P}^{\star}$ is defined as, for any step $h\in[H]$,
    \begin{align}
        \widetilde{P}_h^{\star}(\cdot|s,a) = P_h^{\star}(\cdot|s,a),\quad\forall (s,a)\in\cS\times\cA \quad \text{and}\quad  \widetilde{P}_h^{\star}(\cdot|s_f,a) = \delta_{s_f}(\cdot),\quad \forall a\in\cA.\label{eq: transition new}
    \end{align}
    The reward function $\widetilde{R}$ is defined as, for any step $h\in[H]$,
    \begin{align}
        \widetilde{R}_h(s,a) = \left(\frac{\gamma}{\rho'}\right)^{h-1}\cdot R_h(s,a),\quad \forall (s,a)\in\cS\times\cA \quad \text{and}\quad \widetilde{R}_h(s_f,a) = 0,\quad \forall  a\in\cA.
    \end{align}
    We suppose that the discount factor $\gamma\leq \rho'$ so that the reward function $\widetilde{R}_h\in[0,1]$.
    The robust mapping $\widetilde{\boldsymbol{\Phi}}$ is defined as, for any $P:\widetilde{\cS}\times\cA\mapsto\Delta(\widetilde{\cS})$,
    \begin{align}
        \widetilde{\boldsymbol{\Phi}}(P) = \bigotimes_{(s,a)\in\widetilde{\cS}\times\cA}\left\{\widetilde{P}(\cdot)\in\Delta(\widetilde{\cS}):D_{\mathrm{TV}}\big(\widetilde{P}(\cdot)\big\|P(\cdot|s,a)\big)\leq \rho\right\} := \bigotimes_{(s,a)\in\widetilde{\cS}\times\cA}\widetilde{\cP}_{\rho}(s,a; P),\quad \rho = 1-\rho'.
    \end{align}
Therefore, $\widetilde{\cM}$ is an RMDP with $\widetilde{\cS}\times\cA$-rectangular TV robust set of radius $\rho$ and satisfying Assumption~\ref{ass: zero min} (because it satisfies the ``fail-state" Condition~\ref{assumption:fail:state}).
Furthermore, for any initial state $s_1\in\widetilde{\cS}\setminus\{s_f\} = \cS$, the interaction with the transition kernel $\widetilde{P}^{\star}$ is equivalent to the interaction with the transition kernel $P^{\star}$ of the original RMDP $\cM_{\gamma}$, since by the definition \eqref{eq: transition new}, starting from any $s\neq s_f$ the agent would follow the same dynamics as $P^{\star}$.
    What's more, for any policy $\widetilde{\pi}_h:\widetilde{\cS}\mapsto\Delta(\cA)$ for $\widetilde{\cM}$, it naturally induces the unique policy $\widetilde{\pi}_{\cS, h}:\cS\mapsto\Delta(\cA)$ for the original RMDP $\cM_{\gamma}$.

Therefore, we can
run Algorithm~\ref{alg: tv} on the auxiliary RMDP $\widetilde{\cM}$, starting from the initial state $s_1\in\widetilde{\cS}\setminus\{s_f\}$, which only needs the interaction with $P^{\star}$.
Suppose the output policy by the algorithm is $\{\widetilde{\pi}^k\}_{k=1}^K$, then the following corollary shows the induced policy $\{\widetilde{\pi}^k_\cS\}_{k=1}^K$ for the original RMDP $\cM_{\gamma}$ enjoys a sublinear regret.

\begin{corollary}[Online regret of Algorithm~\ref{alg: tv} for discounted RMDPs with robust sets \eqref{eq: robust set new}]\label{cor: regret discount}
    Consider an $\cS\times\cA$-rectangular $\gamma$-discounted RMDP with robust set \eqref{eq: robust set new} satisfying $0\leq \gamma\leq \rho'\in(0,1]$.
    There exists an algorithm $\mathcal{ALG}$ (specified by the above discussion) such that its online regret for this RMDP is bounded by
    \begin{align}
        \mathrm{Regret}_{\boldsymbol{\Phi}'}^{\mathcal{ALG}}(K) \leq \mathcal{O}\bigg(\sqrt{\min\big\{H,(1-\rho')^{-1}\big\} H^2SAK\iota'}\,\bigg),
    \end{align}
    where $\iota' = \log^2(SAHK/\delta)$ and $(1-\rho')^{-1}$ is interpreted as $+\infty$ when $\rho'=1$
    and $\mathcal{O}(\cdot)$ hides absolute constants and lower order terms in $K$.
\end{corollary}

\begin{proof}[Proof of Corollary~\ref{cor: regret discount}]
    See Appendix~\ref{subsec: proof cor regret discount} for a detailed proof of Corollary~\ref{cor: regret discount}.
\end{proof}

Corollary~\ref{cor: regret discount} shows that besides $\cS\times\cA$-rectangular RMDPs with TV robust set and vanishing minimal value assumption, the $\cS\times\cA$-rectangular discounted RMDP with robust set of bounded transition probability ratio \eqref{eq: robust set new} can also be solved sample-efficiently by robust RL via interactive data collection.
This also echoes our intuition on the support shift issue in Section~\ref{subsec: fail state assumption}.
Furthermore, the regret decays as $\rho'$ decays in which case the transition probability ratio bound becomes higher, i.e., the robust set becomes larger.

\begin{remark}
    The upper bound in Corollary~\ref{cor: regret discount} does not depend on the discount factor $\gamma$ since Algorithm~\ref{alg: tv} adopts a coarse bound of $\widetilde{R}_h\leq 1$.
    The upper bound can be directly improved to be $\gamma$-dependent using a tighter truncation in step \eqref{eq: Q overline} of Algorithm~\ref{alg: tv}.
\end{remark}


\section{Extension II: Robust Decision Making in Multi-Agent Systems}
\label{sec:rmg_extension}

Many operations research problems involve strategic interaction among opponents, e.g., competition, security, and markets, and the underlying model may shift between training and deployment environments.
This section extends our distributionally robust RL framework from single-agent
RMDPs to multi-agent robust Markov games.

\subsection{Learning against an Adversarial Opponent under Environment Ambiguity}
\label{subsec:rmg_onesided}

Consider a sequential decision-making problem where a learning agent (Player~1) competes against an adversarial opponent (Player~2) while simultaneously facing ambiguity in the underlying environment dynamics.
This is motivated by a wide range of operations research applications---including robust operations planning, security games, and competitive resource allocation---where a decision-maker must account for both strategic behavior of other agents and the uncertainty stemming from exogenous factors such as demand volatility, model misspecification, or incomplete system knowledge.

Existing formulations of multi-agent reinforcement learning (MARL) and Markov games (MG) typically presume a known or stationary environment and thus fail to capture robustness requirements against model uncertainty.
In contrast, our formulation explicitly integrates the adversarial interaction and the environment ambiguity within a unified robust Markov game (RMG).
This integration enables the study of policies that are resilient to worst-case opponent strategies as well as worst-case perturbations of the transition dynamics, thereby providing a principled benchmark for robustness, stability, and performance guarantees in complex and uncertain multi-agent operational systems.

\paragraph{Robust Markov game.}
Specifically, we consider a two-player robust Markov game $(\mathcal S,\mathcal A,\mathcal B,H,P^\star,R,\mathbf{\Phi})$,
where $\mathcal S$ is the state space, $\mathcal A$ and $\mathcal B$ are the action spaces of Player~1 and Player~2, respectively, $H$ denotes the horizon, and $R_h:\mathcal S\times\mathcal A\times\mathcal B\to[0,1]$ denotes the stage-$h$ reward.
Player~1 wants to maximize its expected reward, while Player~2 wants to minimize it.
The training environment is governed by the transition kernel $P_h^\star(\cdot| s,a,b)$, which is unknown to the learner but accessible through online interaction, as in the single-agent setup.
The robust set mapping is denoted by $\mathbf{\Phi}$.
Following the single-agent setup, we assume the following.
\begin{assumption}[$\cS\times\cA\times \cB$-rectangularity and TV robust set]\label{ass: sa rmg}
    We assume that, for any transition kernel $P\in\cP$, the robust set $\boldsymbol{\Phi}(P)$ takes the form
    \begin{align}
        \mathbf{\Phi}(P) = \bigotimes_{(s,a,b)\in\mathcal{S}\times\mathcal{A}\times\mathcal{B}} \mathcal{P}_{\rho}(s,a,b; P),
    \end{align}
    where $\mathcal P_\rho(s,a,b;P)$ is the TV-robust set for the transition kernel $P(\cdot | s,a,b)$,  defined as
\[
\mathcal P_\rho(s,a,b;P)
:=
\Big\{\widetilde P(\cdot)\in\Delta(\mathcal S):
D_{\mathrm{TV}}\big(\widetilde P(\cdot)\big\|P(\cdot| s,a,b)\big)\le\rho\Big\}. \label{eq: game_TV-robust set}
\]
\end{assumption}
When Player~2 has a singleton action set $\mathcal{B}=\{b_0\}$, $\mathbf{\Phi}$ reduces to the $\cS\times\cA$-rectangular TV-robust set of the RMDP as defined in Section~\ref{subsec: robust MDP}. Moreover, when $\rho=0$, each ambiguity set collapses to the nominal transition kernel, and the model reduces to the standard zero-sum Markov game \citep{shapley1953stochastic,littman1994markov}.

\paragraph{Robust Nash value and Bellman-Shapley recursion.}
For any Markovian policy pair $(\pi,\nu)$ with $\pi = \{\pi_h: \mathcal{S} \to \Delta(\mathcal A)\}_{h=1}^H$ and $\nu = \{\nu_h: \mathcal{S} \to \Delta(\mathcal B)\}_{h=1}^H$, we define the robust value function as
\begin{align}
    V_{h, P^{\star}, \mathbf{\Phi}}^{\pi, \nu}(s)
    &:= \inf_{\{\widetilde{P}_i\in\mathbf{\Phi}(P_i^{\star})\}_{i=h}^H}
    \mathbb{E}_{\{\widetilde{P}_i\}_{i=h}^H,\pi,\nu}\left[\sum_{i=h}^HR_{i}(s_i,a_i,b_i)\, \middle|\, s_h=s\right],\\
    Q_{h, P^{\star}, \mathbf{\Phi}}^{\pi, \nu}(s, a, b)
    &:= \inf_{\{\widetilde{P}_i\in\mathbf{\Phi}(P_i^{\star})\}_{i=h}^H}
    \mathbb{E}_{\{\widetilde{P}_i\}_{i=h}^H,\pi,\nu}\left[\sum_{i=h}^HR_{i}(s_i,a_i,b_i)\, \middle|\, s_h=s,a_h=a,b_h=b\right].
\end{align}
Here the expectation is taken w.r.t. the state-action trajectories induced by the policies $\pi$ and $\nu$ under the transition $\widetilde{P} = \{\widetilde{P}_h\}_{h=1}^H$.
The associated robust Bellman recursion is given by
\begin{equation}
Q^{\pi,\nu}_{h,P^\star,\mathbf{\Phi}}(s,a,b)
=
R_h(s,a,b)
+
\mathbb{E}_{\mathcal P_\rho(s,a,b;P_h^\star)}
\big[V^{\pi,\nu}_{h+1,P^\star,\mathbf{\Phi}}\big],
\qquad
V^{\pi,\nu}_{h,P^\star,\mathbf{\Phi}}(s)
=
\mathbb{E}_{\pi_h,\nu_h}
\big[Q^{\pi,\nu}_{h,P^\star,\mathbf{\Phi}}(s,\cdot,\cdot)\big],
\label{eq:game_robust_bellman}
\end{equation}
where the worst-case expectation over the robust set $\mathcal P_\rho$ corresponds to the worst-case transition within the TV ball \eqref{eq: game_TV-robust set}.
Now we define the robust Nash value \citep{zhang2020robust,blanchet2023double},
which serves as the natural performance benchmark under simultaneous adversarial opposition and environment ambiguity. The following proposition shows the existence of the robust Nash value and the strong duality property.

\begin{proposition}[Robust Bellman equation, minimax value, and Markov perfect robust saddle]
    \label{prop:robust_rmg_minimax}
    Define $V^\star_{H+1,P^\star,\mathbf{\Phi}}(\cdot)\equiv 0$ and, for $h=H,H-1,\ldots,1$,
    \begin{equation}
    \begin{aligned}
    Q^\star_{h,P^\star,\mathbf{\Phi}}(s,a,b)
    &:=
    R_h(s,a,b)
    +
    \mathbb E_{\mathcal P_\rho(s,a,b;P_h^\star)}
    \!\left[V^\star_{h+1,P^\star,\mathbf{\Phi}}\right],
    \label{eq:robust_shapley_Q}\\
    V^\star_{h,P^\star,\mathbf{\Phi}}(s)
    &:=
    \max_{\pi\in\Delta(\mathcal A)}
    \min_{\nu\in\Delta(\mathcal B)}
    \mathbb E_{a\sim\pi,\,b\sim\nu}
    \!\left[
    Q^\star_{h,P^\star,\mathbf{\Phi}}(s,a,b)
    \right].
    \end{aligned}
    \end{equation}
    Then the following holds:
    \begin{enumerate}
    \item (Robust Nash equilibrium.) There exists a Markovian policy pair $(\pi^\star,\nu^\star)$ such that
    for all $h\in[H]$ and $s\in\mathcal S$, and any Markovian policy pair $(\pi,\nu)$,
    \begin{equation}
    V^{\pi,\nu^\star}_{h,P^\star,\mathbf{\Phi}}(s)~\le~
    V^{\pi^\star,\nu^\star}_{h,P^\star,\mathbf{\Phi}}(s)
    ~\le~
    V^{\pi^\star,\nu}_{h,P^\star,\mathbf{\Phi}}(s),
    \label{eq:robust_saddle_ineq}
    \end{equation}
    and $V^{\pi^\star,\nu^\star}_{h,P^\star,\mathbf{\Phi}}(s)=V^\star_{h,P^\star,\mathbf{\Phi}}(s)$.
    Moreover, for each $(h,s)\in[H]\times \cS$, $(\pi_h^\star(\cdot|s),\nu_h^\star(\cdot|s))$ is a Nash equilibrium of the one-shot
    zero-sum matrix game with payoff matrix $Q^\star_{h,P^\star,\mathbf{\Phi}}(s,\cdot,\cdot)$.
    \item (Strong duality.) For every $h\in[H]$ and $s\in\mathcal S$, it holds that
    \begin{equation}
    \max_{\pi}\min_{\nu}\,V^{\pi,\nu}_{h,P^\star,\mathbf{\Phi}}(s)
    =
    V^\star_{h,P^\star,\mathbf{\Phi}}(s)
    =
    \min_{\nu}\max_{\pi}\,V^{\pi,\nu}_{h,P^\star,\mathbf{\Phi}}(s).
    \label{eq:robust_global_minimax}
    \end{equation}
    \end{enumerate}
    \end{proposition}

\begin{proof}[Proof of Proposition~\ref{prop:robust_rmg_minimax}]
    See Appendix \ref{subsec:proof_rmg_minimax} for a detailed proof.
\end{proof}

\paragraph{Learning objective.}
The learner interacts only with the training kernel $P^\star$ for $K$ episodes. At the beginning of episode $k$, Player~2 may choose a Markov policy $\nu^k$ based on the history before episode $k$. At step $h$, Player~1 samples $a_h^k\sim \pi_h^k(\cdot|s_h^k)$, Player~2 samples $b_h^k\sim \nu_h^k(\cdot|s_h^k)$ without observing Player~1's current sampled action, and the next state $s_{h+1}^k$ is generated from $P_h^\star(\cdot| s_h^k,a_h^k,b_h^k)$.
Player~2's actions are observable to the learner.

For any adaptive Markov opponent sequence $\{\nu^k\}_{k=1}^K$ and any learner policy sequence $\{\pi^k\}_{k=1}^K$ executed during the interaction, we measure the performance of $\{\pi^k\}_{k=1}^K$ via the following definition of robust regret
\begin{equation}
\mathrm{Regret}_{\mathbf{\Phi}, \{\nu^k\}_{k=1}^K}(K)
:=
\sum_{k=1}^K
V^\star_{1,P^\star,\mathbf{\Phi}}(s_1)
-
V^{\pi^k,\nu^k}_{1,P^\star,\mathbf{\Phi}}(s_1).
\label{eq:onesided_regret}
\end{equation}
Intuitively, it quantifies the cumulative regret of the learner (Player~1) relative to the robust Nash value.
This benchmark is the robust analogue of the online regret notions used in non-robust zero-sum Markov games, e.g., \cite{xie2020learning,tian2020provably}.
The key point is that the learner controls only Player~1, while Player~2 may adapt its Markov policy to past episodes, so the comparator should be a value level that the learner could secure before seeing the opponent's future policies.
By Proposition~\ref{prop:robust_rmg_minimax},
\[
V^\star_{1,P^\star,\mathbf{\Phi}}(s_1)
=
\max_{\pi}\min_{\nu}V^{\pi,\nu}_{1,P^\star,\mathbf{\Phi}}(s_1),
\]
namely, the robust Nash value is exactly the largest reward guarantee that Player~1 can secure simultaneously against the worst-case opponent policies and the worst-case transition kernels in the robust set.
In particular, for the robust Nash policy $\pi^\star$, we have
$V^{\pi^\star,\nu}_{1,P^\star,\mathbf{\Phi}}(s_1)\ge V^\star_{1,P^\star,\mathbf{\Phi}}(s_1)$ for any $\nu$.
Thus the comparator is valid and is achievable by a fixed policy, instead of relying on hindsight knowledge of the realized opponent sequence. The term $V^{\pi^k,\nu^k}_{1,P^\star,\mathbf{\Phi}}(s_1)$ then evaluates the learner's policy under the actual opponent policy $\nu^k$ faced in episode $k$, while still allowing nature to pick the worst-case transition kernel in $\mathbf{\Phi}(P^\star)$.
Therefore, the regret in \eqref{eq:onesided_regret} characterizes the two aspects of our problem: strategic opposition from Player~2 and distributional ambiguity in the environment.
This is consistent with regret notions in prior non-robust Markov games.
Specifically, \cite{xie2020learning} studies the online learning setting against an external opponent using a similar value benchmark, while \cite{tian2020provably} explicitly advocates the minimax-value benchmark as a statistically meaningful weakening of the stronger best-policy-in-hindsight regret.

Although the robust Nash value can be informally viewed as a $\max_{\text{Player~1}}\min_{\text{Player~2}}\min_{\text{environment}}$ problem, the last minimization over the environment should not be absorbed into the minimization of Player~2 so as to reduce the problem to a standard zero-sum Markov game. The key point is that the environment minimization is performed entrywise: for each fixed $(s,a,b)$, the environment may choose a different worst-case kernel in $\mathcal P_\rho(s,a,b;P_h^\star)$. Therefore, if one tries to absorb the environment minimization into Player~2, the minimizing player would have to choose a transition perturbation that depends on Player~1's action $a$. This is not the same as a standard zero-sum game, where Player~2 only chooses an action $b$ without observing $a$. In this sense, absorbing the environment minimization into Player~2 changes the information structure of the game. At the same time, the environment is still restricted to the fixed rectangular ambiguity set, rather than choosing an arbitrary perturbation after observing Player~1's action, and this restriction is what keeps the problem tractable in our setting.

\begin{remark}[Comparison with recent works on online robust Markov games]
\label{rmk:online_rmg_diff}
During the preparation of this work, we became aware of two recent works \citep{farhat2025sample,zheng2025distributionally} on online learning in robust Markov games under a different performance criterion.
Their regret is closer to the online counterpart of the robust Nash equilibrium gap (RNE gap) introduced by \citet{blanchet2023double}: in each episode, it measures the gain that a player could obtain by unilaterally deviating to the robust best response against the current joint policy, and then accumulates this equilibrium-gap quantity over episodes.
By contrast, our regret in \eqref{eq:onesided_regret} is tailored to learning against an external adversarial opponent: it compares the learner's robust value under the realized opponent policy $\nu^k$ with the robust Nash value $V^\star_{1,P^\star,\mathbf{\Phi}}(s_1)$.
These two notions emphasize different objectives: their criterion is symmetric and equilibrium-gap based, requiring the centralized control of all agents, whereas ours is minimax-value based and directly evaluates the controlled learner's performance against the realized opponent sequence.
This distinction is why the analysis below targets the regret in \eqref{eq:onesided_regret} rather than the cumulative equilibrium-gap.
\end{remark}

\subsection{Algorithm and Theory}
\label{subsubsec:onesided_vmv_dual}

In this subsection, we propose a game-theoretic extension of \texttt{OPROVI-TV} for RMGs introduced in Section~\ref{subsec:rmg_onesided}.
As in the single-agent setting, the goal is to manage exploration as well as robust planning so that the learner can approach the robust Nash value purely through interactive data collection with the environment and the opponent.
To make this possible, we impose the following vanishing minimal value assumption, which is the natural counterpart of the vanishing minimal value condition in Section~\ref{subsec: fail state assumption}.

\begingroup
\begin{assumption}[Vanishing minimal value for RMG]
\label{ass:onesided_vmv}
For every Markov policy pair $(\pi,\nu)$,
\[
\min_{s\in\cS}
V^{\pi,\nu}_{1,P^\star,\boldsymbol{\Phi}}(s)=0.
\]
\end{assumption}

Because rewards are nonnegative, this step-1 condition implies
$\min_{s\in\cS}V^{\pi,\nu}_{h,P^\star,\boldsymbol{\Phi}}(s)=0$
for every Markov policy pair $(\pi,\nu)$ and every step $h\in[H]$.
This follows by a Bellman-induction argument: at a zero-value state, all nonnegative Bellman terms must vanish, and compactness of the TV ball gives a zero-valued successor state.
This condition is imposed so that the TV-dual representation used in our robust Bellman updates remains valid throughout the analysis.
In contrast to the single-agent setting, Player~2 may adapt its Markov policy across episodes.
Therefore, the analysis must apply not only to the robust Nash value, but also to the value functions induced by the learner's policy together with any realized opponent policy.
A simple sufficient condition is the existence of an absorbing zero-reward fail state $s_f$ such that
$R_h(s_f,a,b)=0$ and $P_h^\star(s_f\mid s_f,a,b)=1$ for all $h\in[H]$ and $(a,b)\in\cA\times\cB$.
Then every policy pair has value zero at $s_f$, so Assumption~\ref{ass:onesided_vmv} holds.
\endgroup

Under Assumption~\ref{ass:onesided_vmv}, we now extend \texttt{OPROVI-TV} from the single-agent robust MDP setting to the two-player zero-sum robust Markov game given in
Section~\ref{subsec:rmg_onesided}. Compared to the single-agent case, the game extension differs in two respects:
(i) we estimate the joint-action transition kernel $P_h^\star(\cdot| s,a,b)$ from observed tuples
$(s_h^k,a_h^k,b_h^k,s_{h+1}^k)$, where the opponent's actions $b_h^k$ are observable,
and (ii) we replace the state-wise $\max$ backup in value iteration by an upper-only robust minimax backup, implemented by solving a matrix game at each state.
In each episode $k$, the learner constructs an estimated nominal model $\widehat P^k$ and computes an
optimistic robust max--min plan by solving $S$ independent matrix games per stage in the backward pass.

\begin{algorithm}[t]
    \caption{\underline{OP}timistic \underline{RO}bust \underline{V}alue \underline{I}teration for TV Robust Set (\texttt{OPROVI-TV-MG})} \label{alg:onesided_oprovitv_mg}
    \begin{algorithmic}[1]
        \STATE \textbf{Initialize:} dataset $\mathbb{D} = \emptyset$.
        \FOR{episode $k=1,\cdots,K$}
        \STATE \texttt{Training environment transition estimation:}
        \STATE Set $N_h^k(s,a,b,s') := \sum_{\tau=1}^{k-1}\mathbf 1\!\left\{(s_h^\tau,a_h^\tau,b_h^\tau,s_{h+1}^\tau)=(s,a,b,s')\right\}; \,  N_h^k(s,a,b):=\sum_{s'\in\mathcal S}N_h^k(s,a,b,s')$.

        \STATE {Calculate $\widehat{P}_h^k$ by setting $\widehat{P}_h^k(s'|s,a,b)=N_h^k(s,a,b,s')/N_h^k(s,a,b)$ if $N_h^k(s,a,b)>0$, and setting $\widehat{P}_h^k(\cdot|s,a,b)$ to the uniform distribution over $\mathcal S$ otherwise.}
        \STATE \texttt{Optimistic robust planning:}
        \STATE Set $\overline{V}_{H+1}^k = 0$.
        \FOR{step $h=H,\cdots,1$}
        \STATE Set $\overline{Q}_h^k(\cdot,\cdot,\cdot)$ as \eqref{eq: game Q overline}, with the bonus function $\texttt{bonus}_h^k(\cdot,\cdot,\cdot)$ defined in \eqref{eq: mg_bonus}.
        \STATE Compute $\pi_h^k(\cdot|s)\in\argmax_{\pi\in\Delta(\mathcal A)}\min_{\nu\in\Delta(\mathcal B)}\mathbb{E}_{a\sim\pi,b\sim\nu}\big[\overline{Q}_h^k(s,a,b)\big]$ and a best response $\tilde{\nu}_h^k(\cdot|s)\in\argmin_{\nu\in\Delta(\mathcal B)}\mathbb{E}_{a\sim\pi_h^k(\cdot|s),b\sim\nu}\big[\overline{Q}_h^k(s,a,b)\big]$.
        \STATE Set $\overline{V}_h^k(s) = \mathbb{E}_{a \sim \pi_h^k(\cdot|s), b \sim \tilde{\nu}_h^k(\cdot|s)}[\overline{Q}_h^k(s,a,b)]$.
        \ENDFOR
        \STATE \texttt{Execute the policy in training environment and collect data:}
        \STATE Receive the fixed initial state $s_1^k=s_1\in\cS$.
        \FOR{step $h=1,\cdots,H$}
        \STATE Player~1 takes action $a_h^k\sim \pi^k_h(\cdot|s_h^k)$; Player~2 takes action $b_h^k\sim \nu^k_h(\cdot|s_h^k)$.
        \STATE Observe reward $R_h(s_h^k,a_h^k,b_h^k)$ and the next state $s_{h+1}^k \sim P_h^\star(\cdot|s_h^k,a_h^k,b_h^k)$.
        \ENDFOR
        \STATE Set $\mathbb{D}$ as $\mathbb{D}\cup\{(s_h^k,a_h^k,b_h^k,s_{h+1}^k)\}_{h=1}^H$.
        \ENDFOR
    \end{algorithmic}
\end{algorithm}

Our algorithm \texttt{OPROVI-TV-MG} is a direct game-theoretic extension of the single-agent algorithm
\texttt{OPROVI-TV} (Algorithm~\ref{alg: tv}).
At a high level, both algorithms share the same episodic ``estimate--plan--execute'' template.
The game-specific modifications are:
\begin{enumerate}
\item \textbf{Joint-action model estimation:} replace the empirical kernel $\widehat P_h^k(\cdot| s,a)$ by the
joint-action estimator $\widehat P_h^k(\cdot| s,a,b)$ using the observed opponent actions $b_h^k$.
\item \textbf{Robust minimax planning:} replace the state-wise greedy maximization in \texttt{OPROVI-TV} by a state-wise
zero-sum matrix game in each Bellman backup, i.e., replace $\max_{a\in\cA}$ with
$\max_{\pi\in\Delta(\cA)}\min_{\nu\in\Delta(\cB)}$ as suggested by the Bellman recursion
(Proposition~\ref{prop:robust_rmg_minimax}).
Unlike the single-agent case, where optimistic and pessimistic recursions bracket the value of the same learner policy, a game value also depends on the opponent policy. The optimistic recursion remains useful because it upper bounds the robust Nash value via monotonicity of the max--min operator; see \eqref{eq:onesided_optimism_hoeffding_new}. In contrast, a max--min pessimistic update would minimize over an internally selected opponent distribution, while the episode is played against the external policy $\nu^k$, which may adapt across episodes. We therefore specify the optimistic robust minimax backup below.
\end{enumerate}

\paragraph{Optimistic robust minimax planning.}
\begingroup
Given the estimated nominal model $\widehat P_h^k(\cdot| s,a,b)$, we use the following bonus,
\begin{equation}
    \begin{aligned}
    \mathrm{bonus}_h^k(s,a,b)
    & =
    c_b\cdot \min\{H,\rho^{-1}\}\cdot
    \sqrt{
    \frac{S\iota}{N_h^k(s,a,b)\vee 1}
    },
    \end{aligned} \label{eq: mg_bonus}
    \end{equation}
where $c_b>0$ is a sufficiently large absolute constant and $\iota=\log(SABHK/\delta)$.
\endgroup
\begingroup
We use a TV-dual convention analogous to \eqref{eq: duality tv}, with $\eta$ restricted to $[0,\min\{H,\rho^{-1}\}]$, which is without loss for the clipped value functions used below. For any transition kernel $P$, tuple $(s,a,b)$, and $V:\mathcal S\to[0,\min\{H,\rho^{-1}\}]$, set
\[
\mathbb{E}_{\mathcal{P}_{\rho}(s,a,b;P)}[V]
:=
\sup_{\eta\in[0,\min\{H,\rho^{-1}\}]}
\Big\{
- \mathbb E_{P(\cdot|s,a,b)}[(\eta-V)_+]
+(1-\rho)\eta\Big\}.
\]
Under Assumption~\ref{ass:onesided_vmv}, this notation agrees with the true TV-robust Bellman term whenever $V$ is $V^\star$ or a realized value function.
With this convention, the optimistic $Q$-estimate is
\begin{equation}
    \begin{aligned}
    \overline{Q}_h^k(s,a,b)
    =
    \min\Big\{
    &R_h(s,a,b)
    +
    \mathbb{E}_{\mathcal{P}_{\rho}(s,a,b;\widehat{P}_h^k)}
    \big[\overline{V}_{h+1}^k\big]
    +
    \texttt{bonus}_h^k(s,a,b),\min\big\{H, \rho^{-1}\big\}
    \Big\}.
    \end{aligned}
    \label{eq: game Q overline}
\end{equation}
The policy update is the state-wise matrix-game step in Algorithm~\ref{alg:onesided_oprovitv_mg}: for each $s$, solve the zero-sum game with payoff matrix $\overline Q_h^k(s,\cdot,\cdot)$ to obtain $\pi_h^k$, the planning best response $\tilde{\nu}_h^k$, and the value $\overline V_h^k$.
For finite action spaces, this matrix game reduces to a standard linear program.
\endgroup


\paragraph{Theoretical Results.}
We now state our results on the regret \eqref{eq:onesided_regret} for \texttt{OPROVI-TV-MG}.


\begin{theorem}[Online regret \eqref{eq:onesided_regret} of \texttt{OPROVI-TV-MG}]
\label{thm:onesided_regret}
Consider an RMG with $\cS\times\cA\times\cB$-rectangular TV robust set of radius $\rho\in[0,1)$ (Assumption~\ref{ass: sa rmg}) satisfying Assumption~\ref{ass:onesided_vmv}. Fix any adaptive Markov opponent sequence
$\nu=\{\nu^k\}_{k=1}^K$ that is non-anticipating with respect to Player~1's current action, and use the bonus function in \eqref{eq: mg_bonus}.
Then, with probability at least $1-\delta$, the online robust regret of \texttt{OPROVI-TV-MG} (Algorithm~\ref{alg:onesided_oprovitv_mg}) is bounded by
\[
\mathrm{Regret}_{\mathbf{\Phi}, \{\nu^k\}_{k=1}^K}(K)
\le
{
\widetilde{\mathcal O}\left(
\min\big\{H, \rho^{-1}\big\}\cdot H S\sqrt{ABK}
\right).}
\]
Here $\widetilde{\mathcal O}$ hides absolute constants and logarithmic factors in $(S,A,B,H,K,1/\delta)$ and lower order terms in $K$.
\end{theorem}

\begin{proof}[Proof of Theorem~\ref{thm:onesided_regret}]
    See Appendix \ref{sec:game_proof} for a detailed proof.
\end{proof}

\begingroup
Theorem~\ref{thm:onesided_regret} gives sublinear online robust regret against any non-anticipating adaptive Markov opponent sequence.
The rate is weaker than the single-agent bound because the regret is measured along the opponent sequence that is actually realized.
Such a sequence may adapt to past data and need not coincide with the worst response used in the max--min planning problem.
Thus a state-wise pessimistic max--min recursion would not give the trajectory-wise optimistic--pessimistic width used in the single-agent Bernstein argument.
Accordingly, \texttt{OPROVI-TV-MG} uses only the optimistic recursion and controls the robust Bellman error through a uniform $L_1$ transition-estimation bound.
The price is the extra factor coming from this uniform concentration step, rather than from a variance-sensitive width recursion.
\endgroup


\section{Application: Data-Driven Robust Inventory Control}
\label{sec:inventory}

Inventory control is a canonical operations research problem where a decision maker repeatedly trades off ordering costs against service and shortage risks under stochastic demand. In modern practice, data-driven inventory policies are often trained or tuned in a training environment (e.g., a model calibrated from historical data) and then deployed under demand shifts and model misspecification.
This motivates a distributionally robust formulation: seeking inventory policy that performs well under worst-case perturbations of the demand law around the training environment.
In this section, we present that a standard periodic-review inventory model can be written as a finite-horizon MDP; therefore, our robust RL framework and guarantees apply.

\subsection{Inventory Control as a Finite-horizon MDP}
\label{subsec:inv_mdp}

We consider a single-item, periodic-review inventory system over a planning horizon of length $H$.
The system state $x_h$ represents the inventory level.
State $x_h$ is possibly negative, representing backlog at period $h\in[H]$.
At each period $h$, the decision maker chooses an order quantity $q_h$, observes a stochastic demand realization $D_h$, incurs a cost, and transits to the next inventory level $x_{h+1}$.
We formulate this problem as a finite-horizon tabular MDP
$(\mathcal{S}_{\mathrm{inv}},\mathcal{A}_{\mathrm{inv}},H,P^\star,R)$,
whose components are defined as follows.
\begin{enumerate}
\item \textbf{State and action spaces.}
We adopt standard capacity and service-level truncations,
\[
\mathcal{S}_{\mathrm{inv}} := \{-B,-B+1,\cdots,I\} \cup \{s_f\},\quad
\mathcal{A}_{\mathrm{inv}} := \{0,1,\cdots,Q\}.
\]
Here $I\in\NN_+$ is the inventory capacity, $Q\in\NN_+$ is the order capacity, and $B\in\NN_+$ is the backlog threshold. We introduce $s_f$ as an aggregated fail-state representing all inventory positions below $-B$, namely exceptional operating regimes that are not modeled explicitly in the truncated state space.
These truncations are common in practice when inventory or backlog is constrained by storage limits, service-level requirements, or contractual considerations.

\item \textbf{Nominal transition kernel induced by demand.}
Given an inventory state $x_h\in\{-B,\ldots,I\}$ and an order quantity $q_h$, the post-order inventory is
$x'_h:=\min\{x_h+q_h,I\}$.
Let the (truncated) demand support be $\mathcal D:=\{0,1,\ldots,D\}$, and let $D_h\in\mathcal D$ denote the demand realized in the period $h$.
Then the pre-truncation next inventory state is
$x_{h+1}^{\mathrm{pre}} := x'_h - D_h$.
We then apply a service truncation,
\[
x_{h+1}=T_h(x_h,q_h,D_h):=
\begin{cases}
x_{h+1}^{\mathrm{pre}}, & \text{if } x_{h+1}^{\mathrm{pre}}\ge -B,\\
s_f, & \text{if } x_{h+1}^{\mathrm{pre}}<-B.
\end{cases}
\]
We make $s_f$ absorbing in the truncated MDP: once the backlog threshold is violated, the subsequent dynamics over the remaining horizon are no longer modeled explicitly and are instead aggregated into this sink state. Thus the model does not distinguish among different post-violation paths after the system enters this exceptional operating regime. For any $(q, D_h)\in\mathcal A_{\mathrm{inv}}\times\mathcal{D}$, we therefore set $T_h(s_f,q, D_h)=s_f$. If the conditional demand law is $d_h^\star(\cdot| x,q)\in\Delta(\mathcal D)$ for each $(h,x,q)$, then the induced transition kernel is the pushforward,
\[
P_h^\star(\cdot| x,q)= d_h^\star(\cdot| x,q)\circ T_h(x,q,\cdot)^{-1}.
\]

\item \textbf{Reward function induced by a cost transformation.}
Let the one-period cost at step $h$ be
\[
c_h(x_h,q_h,D_h)
:= c^{\mathrm{order}}\cdot \mathbf 1\{q_h>0\}
+ c^{\mathrm{hold}}\cdot (x_{h+1}^{\mathrm{pre}})^+
+ c^{\mathrm{back}}\cdot (-x_{h+1}^{\mathrm{pre}})^+,
\]
where $c^{\mathrm{order}}>0$ is a fixed ordering/setup cost, $c^{\mathrm{hold}}>0$ is a holding cost coefficient, and $c^{\mathrm{back}}>0$ is a backlog cost coefficient. When $x_{h+1}\neq s_f$, we have $x_{h+1}=x_{h+1}^{\mathrm{pre}}$, so the cost can be written exactly as a deterministic transition-cost function of $(x_h,q_h,x_{h+1})$:
\[
c_h(x_h,q_h,x_{h+1})
:= c^{\mathrm{order}}\cdot \mathbf 1\{q_h>0\}
+ c^{\mathrm{hold}}\cdot (x_{h+1})_+
+ c^{\mathrm{back}}\cdot (-x_{h+1})_+,
\quad \text{for } x_{h+1}\neq s_f.
\]
For the transition $x_{h+1}=s_f$, which represents the event of $x_{h+1}^{\mathrm{pre}}<-B$, the exact overflow magnitude and the subsequent post-violation dynamics are no longer tracked in the truncated model. To keep the finite MDP well defined without introducing additional parameters, we assign the maximal cost $c_{\max}$ to this aggregated fail-state, which serves as a conservative reduced-form representation of leaving the normal operating regime.
We define $c_h(x_h,q_h,s_f):=c_{\max}$ for all $(x_h,q_h)$, where a valid uniform cost upper bound $c_{\max}$ is
\[
c_{\max}:= c^{\mathrm{order}} + c^{\mathrm{hold}} \cdot I + c^{\mathrm{back}}\cdot (D+B).
\]
Finally, to align with the reward-maximization convention in our framework, we define the reward
\[
R_h(x_h,q_h,x_{h+1})
:= 1-\frac{c_h(x_h,q_h,x_{h+1})}{c_{\max}}\in[0,1].
\]
In particular, $R_h(\cdot,\cdot,s_f)=0$, and since $s_f$ is absorbing, the inventory MDP satisfies the vanishing minimal value condition.
Also, we assume that $x_1\neq s_f$.
Thus one can show that Assumption~\ref{ass: zero min} holds.
\end{enumerate}

\subsection{Distributional Robustness via an \texorpdfstring{$\mathcal{S}\times \mathcal{A}$-rectangular}{S x A-rectangular} TV Uncertainty Set}
\label{subsec:inv_rmdp}

In inventory control, we consider the distribution shift induced by the demand distribution.
Specifically, we model demand shift given inventory state $x$ and order quantity $q$ by a TV robust set,
\[
\mathcal U^D_{h,\rho}(x,q)
:=\Big\{\widetilde d_h(\cdot| x,q)\in\Delta(\mathcal D):
D_{\mathrm{TV}}\big(\widetilde d_h(\cdot| x,q)\big\|d_h^\star(\cdot| x,q)\big)\le\rho\Big\}.
\]
Each $\widetilde d_h(\cdot| x,q)\in\mathcal U^D_{h,\rho}(x,q)$ induces a transition kernel
\[
\widetilde P_h(\cdot| x,q)=\widetilde d_h(\cdot| x,q)\circ T_h(x,q,\cdot)^{-1}.
\]
Hence, the induced transition robust set is
\[
\mathcal U^P_{h,\rho}(x,q)
= \Big\{\widetilde P_h(\cdot| x,q)=\widetilde d_h(\cdot| x,q)\circ T_h(x,q,\cdot)^{-1}:\widetilde d_h(\cdot| x,q)\in\mathcal U^D_{h,\rho}(x,q)\Big\}.
\]
We also define the TV ball around the nominal transition kernel as follows,
\[
\mathcal{P}_{\rho}(x,q;P_h^\star)
:= \Big\{\widetilde P_h(\cdot| x,q): D_{\mathrm{TV}}\big(\widetilde P_h(\cdot| x,q)\big\| P_h^\star(\cdot| x,q)\big)\le \rho\Big\}.
\]

\begin{lemma}
\label{lem:tv:contraction}
For any $(h,x,q)\in[H]\times\cS_{\mathrm{inv}}\times\cA_{\mathrm{inv}}$, it holds that $\mathcal U^P_{h,\rho}(x,q) \subseteq \mathcal{P}_{\rho}(x,q;P_h^\star)$.
\end{lemma}

\begin{proof}[Proof of Lemma~\ref{lem:tv:contraction}]
This follows directly from the data-processing inequality for TV distance under measurable maps: for any distributions $\mu,\nu$ on $\mathcal D$ and any measurable $g$,
\[
D_{\mathrm{TV}}\big(\mu\circ g^{-1}\big\|\nu\circ g^{-1}\big)\le D_{\mathrm{TV}}(\mu\|\nu).
\]
Applying it with $g(d)=T_h(x,q,d)$ proves Lemma~\ref{lem:tv:contraction}.
\end{proof}

Next, we explain why we only model transition ambiguity, and not reward ambiguity. Under any demand law $\widetilde d_h(\cdot| x,q)$, the induced next-state transition is $\widetilde P_h(\cdot| x,q)=\widetilde d_h(\cdot| x,q)\circ T_h(x,q,\cdot)^{-1}$. Because the one-step reward is a deterministic function of the realized next state, namely $R_h(x,q,x')$ with $x'=T_h(x,q,D_h)$, the conditional law of the reward given $(x,q)$ is fully determined by $\widetilde P_h(\cdot| x,q)$.
Therefore, demand perturbations do not introduce independent degrees of freedom in the reward: all reward variability is characterized by the next-state distribution.
Meanwhile, for any bounded $V_{h+1}:\mathcal S_{\mathrm{inv}}\to\mathbb R$ and $(x,q)\in \cS_{\mathrm{inv}}\times\cA_{\mathrm{inv}}$, we have,
\begin{align*}
\inf_{\widetilde d_h\in\mathcal U^D_{h,\rho}(x,q)}
\mathbb E_{\substack{D_h\sim \widetilde d_h(\cdot| x,q) \\
x'=T_h(x,q,D_h)}}
\!\left[ R_h(x,q,x')+V_{h+1}(x')\right]
=
\inf_{\widetilde P_h\in\mathcal U^P_{h,\rho}(x,q)}
\mathbb E_{x'\sim \widetilde P_h(\cdot| x,q)}
\!\left[ R_h(x,q,x')+V_{h+1}(x')\right];
\end{align*}
and due to Lemma~\ref{lem:tv:contraction}, we have $\mathcal U^P_{h,\rho}(x,q)\subseteq \mathcal{P}_{\rho}(x,q;P_h^\star)$. Therefore, it suffices (and is conservative) to work with the $\mathcal{S}\times \mathcal{A}$-rectangular TV ball $\mathcal{P}_{\rho}(\cdot,\cdot;P_h^\star)$ at the transition level, without introducing a separate reward ambiguity set.

Returning to our framework of RMDPs in Section~\ref{subsec: robust MDP}, the corresponding $\mathcal{S}\times \mathcal{A}$-rectangular TV robust-set mapping is defined by
\[
\mathbf{\Phi}_{\mathrm{inv}}(P^\star_h)
= \bigotimes_{(x,q)\in\mathcal{S}_{\mathrm{inv}}\times\mathcal{A}_{\mathrm{inv}}}\ \mathcal{P}_\rho(x,q;P^\star_h).
\]
Thus we define the corresponding RMDP for inventory control as $\cM_{\mathrm{inv}}=(\cS_{\mathrm{inv}}, \cA_{\mathrm{inv}}, H, P^{\star}, R,\mathbf{\Phi}_{\mathrm{inv}})$, where $R=\{R_h(x,q,x')\}_{h\in[H]}$ is a known transition-dependent reward.

\paragraph{Our robust inventory control objective.}
A Markovian inventory policy $\pi=\{\pi_h\}_{h=1}^H$ with $\pi_h:\mathcal S_{\mathrm{inv}}\mapsto\Delta(\mathcal A_{\mathrm{inv}})$ induces the robust value function
\[
V_{h,P^{\star},\mathbf{\Phi}_{\mathrm{inv}}}^{\pi}(x)
= \inf_{\widetilde{P}_h\in\mathbf{\Phi}_{\mathrm{inv}}(P_h^{\star}), 1\leq h\leq H}
\mathbb{E}_{\widetilde{P},\pi}\left[\sum_{i=h}^H R_{i}(x_i,q_i,x_{i+1})\, \middle|\, x_h=x\right],
\quad \forall x\in\mathcal S_{\mathrm{inv}}.
\]
Under rectangularity, the corresponding robust Bellman recursion uses the same TV ambiguity set, with the known one-period reward kept inside the robust expectation:
\[
Q_{h,P^\star,\mathbf{\Phi}_{\mathrm{inv}}}^{\pi}(x,q)
=
\inf_{\widetilde P_h\in\mathcal P_\rho(x,q;P_h^\star)}
\mathbb E_{x'\sim \widetilde P_h(\cdot|x,q)}
\!\left[
R_h(x,q,x')+
V_{h+1,P^\star,\mathbf{\Phi}_{\mathrm{inv}}}^{\pi}(x')
\right].
\]
The optimal robust inventory policy $\pi^\star = \{\pi_h^\star\}_{h=1}^H$ is
\[
\pi^\star \in \argmax_{\pi}\ V_{1,P^{\star},\mathbf{\Phi}_{\mathrm{inv}}}^{\pi}(x_1).
\]
Let $V_{1,P^{\star},\mathbf{\Phi}_{\mathrm{inv}}}^{\star} := V_{1,P^{\star},\mathbf{\Phi}_{\mathrm{inv}}}^{\pi^\star}$ be the optimal robust value function. By construction, $V_{1,P^{\star},\mathbf{\Phi}_{\mathrm{inv}}}^{\star}(s_f)=0$, so the vanishing minimal value condition (Assumption~\ref{ass: zero min}) holds.
\paragraph{Interactive learning protocol and regret objective.}
We adopt the interactive data-collection protocol in Section~\ref{subsec: robust MDP}. Across episodes $k=1,\ldots,K$, the learner interacts only with the training inventory transition $P^\star$.
In episode $k$, it executes a policy $\pi^k$, observes realized rewards and state transitions, updates its estimate of $P^\star$, and proceeds to the next episode. The performance criterion is the online regret
\[
\mathrm{Regret}_{\mathbf{\Phi}_{\mathrm{inv}}}(K)
:=\sum_{k=1}^KV_{1,P^{\star},\mathbf{\Phi}_{\mathrm{inv}}}^{\star}(x_1)-V_{1,P^{\star},\mathbf{\Phi}_{\mathrm{inv}}}^{\pi^k}(x_1).
\]
The goal is to design an algorithm that achieves sublinear regret in $K$ (or equivalently, via an online-to-batch conversion, to output an $\varepsilon$-optimal robust inventory policy).

\subsection{Theoretical Guarantee for Robust Inventory Learning}
Since $\cM_{\mathrm{inv}}$ is a finite-horizon RMDP with $\mathcal{S}\times \mathcal{A}$-rectangular TV robust set (Assumption~\ref{ass: sa}), the vanishing minimal value condition (Assumption~\ref{ass: zero min}) also holds. The transition-dependent reward only changes the one-step planning operator: in \texttt{OPROVI-TV} (Algorithm~\ref{alg: tv}), the robust expectation of $V_{h+1}$ is replaced by the robust expectation of the known bounded function $R_h(x,q,\cdot)+V_{h+1}(\cdot)$. No additional reward parameter is learned. Moreover, in the true robust recursion, $R_h(x,q,s_f)+V_{h+1}(s_f)=0$ because $s_f$ is absorbing and has zero reward, so the TV-duality simplification under Assumption~\ref{ass: zero min} continues to apply. Hence the regret bound (Theorem~\ref{thm: regret tv}) and sample complexity guarantee (Corollary~\ref{cor: sample complexity tv}) carry over with the same order.

\begin{theorem}[Sample-efficient robust learning for inventory control]
\label{thm:inventory_control}
Let $\cM_{\mathrm{inv}}$ be the inventory RMDP with TV radius $\rho\in[0,1)$.
Then, with probability at least $1-\delta$, \texttt{OPROVI-TV} with the transition-reward Bellman update above achieves
\begin{align}
\mathrm{Regret}_{\mathbf{\Phi}_{\mathrm{inv}}}(K)
\le \widetilde{\mathcal{O}}\!\left(\sqrt{\min\big\{H,\rho^{-1}\big\}\cdot H^2 (I+B)QK}\right),
\end{align}
and outputs an $\varepsilon$-optimal robust inventory policy using
\begin{align}
\widetilde{\mathcal{O}}\!\left(\frac{\min\big\{H,\rho^{-1}\big\}\cdot H^2 (I+B)Q}{\varepsilon^2}\right)
\label{eq: inventory_complexity}
\end{align}
episodes. Here $\widetilde{\mathcal{O}}$ omits absolute constants and logarithmic factors in $(I,B,Q,H,K,1/\delta)$.
\end{theorem}

\begin{proof}[Proof of Theorem~\ref{thm:inventory_control}]
The statement follows by applying the proof of Theorem~\ref{thm: regret tv} and Corollary~\ref{cor: sample complexity tv} to $\cM_{\mathrm{inv}}$, replacing each continuation value $V_{h+1}(\cdot)$ in the one-step robust expectation by the known bounded function $R_h(x,q,\cdot)+V_{h+1}(\cdot)$.
\end{proof}

Theorem~\ref{thm:inventory_control} provides a finite-sample guarantee for learning the distributionally robust inventory policy via interactive data collection from the training inventory environment only, and without assuming access to a generative model of the inventory transition.
The complexity bound scales only with the tabular size $I+B$ and $Q$, without scaling with the demand bound $D$ due to the truncation-based transition design.
It becomes smaller as $\rho$ increases, reflecting that larger ambiguity sets lead to smaller value spans and statistically easier robust learning in our framework.

\section{Conclusions and Discussions}\label{subsec: discussion}

In this work, we show that without any structural assumptions, robust RL through interactive data collection necessarily induces a linear regret lower bound in the worst case due to the curse of support shift.
Meanwhile, under the vanishing minimal value assumption, which effectively rules out the support-shift pathology for RMDPs with a TV robust set, we develop a sample-efficient robust RL algorithm for this class of problems.
Beyond the main finite-horizon RMDP setup, we further extend our algorithm and theory to (i) discounted RMDPs with ratio-bounded robust sets, (ii) robust Markov games.
To demonstrate the operational relevance of our theory, we instantiate the framework for data-driven robust inventory control under demand shifts.
Together, these extensions and applications show that the ideas developed in this paper form a broader framework for robust sequential decision-making with interactive data collection.



\section*{Acknowledgement}
The authors would like to thank the anonymous reviewers for their helpful comments.
The authors would also like to thank Pan Xu and Zhishuai Liu for their feedback on an early draft of this work.

\bibliography{reference}

@article{boute2022deep,
  title={Deep reinforcement learning for inventory control: A roadmap},
  author={Boute, Robert N and Gijsbrechts, Joren and Van Jaarsveld, Willem and Vanvuchelen, Nathalie},
  journal={European journal of operational research},
  volume={298},
  number={2},
  pages={401--412},
  year={2022},
  publisher={Elsevier}
}

@article{maurer2009empirical,
  title={Empirical bernstein bounds and sample variance penalization},
  author={Maurer, Andreas and Pontil, Massimiliano},
  journal={arXiv preprint arXiv:0907.3740},
  year={2009}
}

@article{blanchet2023double,
  title={Double pessimism is provably efficient for distributionally robust offline reinforcement learning: Generic algorithm and robust partial coverage},
  author={Blanchet, Jose and Lu, Miao and Zhang, Tong and Zhong, Han},
  journal={arXiv preprint arXiv:2305.09659},
  year={2023}
}

@article{shi2023curious,
  title={The Curious Price of Distributional Robustness in Reinforcement Learning with a Generative Model},
  author={Shi, Laixi and Li, Gen and Wei, Yuting and Chen, Yuxin and Geist, Matthieu and Chi, Yuejie},
  journal={arXiv preprint arXiv:2305.16589},
  year={2023}
}

@article{wang2023sample,
  title={Sample Complexity of Variance-reduced Distributionally Robust Q-learning},
  author={Wang, Shengbo and Si, Nian and Blanchet, Jose and Zhou, Zhengyuan},
  journal={arXiv preprint arXiv:2305.18420},
  year={2023}
}

@inproceedings{azar2017minimax,
  title={Minimax regret bounds for reinforcement learning},
  author={Azar, Mohammad Gheshlaghi and Osband, Ian and Munos, R{\'e}mi},
  booktitle={International Conference on Machine Learning},
  pages={263--272},
  year={2017},
  organization={PMLR}
}

@article{auer2008near,
  title={Near-optimal regret bounds for reinforcement learning},
  author={Auer, Peter and Jaksch, Thomas and Ortner, Ronald},
  journal={Advances in neural information processing systems},
  volume={21},
  year={2008}
}

@article{agarwal2019reinforcement,
  title={Reinforcement learning: Theory and algorithms},
  author={Agarwal, Alekh and Jiang, Nan and Kakade, Sham M and Sun, Wen},
  journal={CS Dept., UW Seattle, Seattle, WA, USA, Tech. Rep},
  pages={10--4},
  year={2019}
}

@article{yang2022toward,
  title={Toward theoretical understandings of robust Markov decision processes: Sample complexity and asymptotics},
  author={Yang, Wenhao and Zhang, Liangyu and Zhang, Zhihua},
  journal={The Annals of Statistics},
  volume={50},
  number={6},
  pages={3223--3248},
  year={2022},
  publisher={Institute of Mathematical Statistics}
}

@article{ouyang2022training,
  title={Training language models to follow instructions with human feedback},
  author={Ouyang, Long and Wu, Jeffrey and Jiang, Xu and Almeida, Diogo and Wainwright, Carroll and Mishkin, Pamela and Zhang, Chong and Agarwal, Sandhini and Slama, Katarina and Ray, Alex and others},
  journal={Advances in Neural Information Processing Systems},
  volume={35},
  pages={27730--27744},
  year={2022}
}

@article{silver2017mastering,
  title={Mastering the game of go without human knowledge},
  author={Silver, David and Schrittwieser, Julian and Simonyan, Karen and Antonoglou, Ioannis and Huang, Aja and Guez, Arthur and Hubert, Thomas and Baker, Lucas and Lai, Matthew and Bolton, Adrian and others},
  journal={nature},
  volume={550},
  number={7676},
  pages={354--359},
  year={2017},
  publisher={Nature Publishing Group}
}

@article{sadeghi2016cad2rl,
  title={Cad2rl: Real single-image flight without a single real image},
  author={Sadeghi, Fereshteh and Levine, Sergey},
  journal={arXiv preprint arXiv:1611.04201},
  year={2016}
}

@article{jin2018q,
  title={Is Q-learning provably efficient?},
  author={Jin, Chi and Allen-Zhu, Zeyuan and Bubeck, Sebastien and Jordan, Michael I},
  journal={Advances in neural information processing systems},
  volume={31},
  year={2018}
}

@article{kiran2021deep,
  title={Deep reinforcement learning for autonomous driving: A survey},
  author={Kiran, B Ravi and Sobh, Ibrahim and Talpaert, Victor and Mannion, Patrick and Al Sallab, Ahmad A and Yogamani, Senthil and P{\'e}rez, Patrick},
  journal={IEEE Transactions on Intelligent Transportation Systems},
  volume={23},
  number={6},
  pages={4909--4926},
  year={2021},
  publisher={IEEE}
}

@inproceedings{panaganti2022sample,
  title={Sample Complexity of Robust Reinforcement Learning with a Generative Model},
  author={Panaganti, Kishan and Kalathil, Dileep},
  booktitle={International Conference on Artificial Intelligence and Statistics},
  pages={9582--9602},
  year={2022},
  organization={PMLR}
}

@inproceedings{zhou2021finite,
  title={Finite-Sample Regret Bound for Distributionally Robust Offline Tabular Reinforcement Learning},
  author={Zhou, Zhengqing and Zhou, Zhengyuan and Bai, Qinxun and Qiu, Linhai and Blanchet, Jose and Glynn, Peter},
  booktitle={International Conference on Artificial Intelligence and Statistics},
  pages={3331--3339},
  year={2021},
  organization={PMLR}
}

@article{wiesemann2013robust,
  title={Robust Markov decision processes},
  author={Wiesemann, Wolfram and Kuhn, Daniel and Rustem, Ber{\c{c}}},
  journal={Mathematics of Operations Research},
  volume={38},
  number={1},
  pages={153--183},
  year={2013},
  publisher={INFORMS}
}

@article{iyengar2005robust,
  title={Robust dynamic programming},
  author={Iyengar, Garud N},
  journal={Mathematics of Operations Research},
  volume={30},
  number={2},
  pages={257--280},
  year={2005},
  publisher={INFORMS}
}

@article{ma2022distributionally,
  title={Distributionally robust offline reinforcement learning with linear function approximation},
  author={Ma, Xiaoteng and Liang, Zhipeng and Xia, Li and Zhang, Jiheng and Blanchet, Jose and Liu, Mingwen and Zhao, Qianchuan and Zhou, Zhengyuan},
  journal={arXiv preprint arXiv:2209.06620},
  year={2022}
}

@article{shi2022distributionally,
  title={Distributionally robust model-based offline reinforcement learning with near-optimal sample complexity},
  author={Shi, Laixi and Chi, Yuejie},
  journal={arXiv preprint arXiv:2208.05767},
  year={2022}
}

@article{panaganti2022robust,
  title={Robust reinforcement learning using offline data},
  author={Panaganti, Kishan and Xu, Zaiyan and Kalathil, Dileep and Ghavamzadeh, Mohammad},
  journal={arXiv preprint arXiv:2208.05129},
  year={2022}
}

@inproceedings{jin2020provably,
  title={Provably efficient reinforcement learning with linear function approximation},
  author={Jin, Chi and Yang, Zhuoran and Wang, Zhaoran and Jordan, Michael I},
  booktitle={Conference on Learning Theory},
  pages={2137--2143},
  year={2020},
  organization={PMLR}
}

@article{dong2022online,
  title={Online Policy Optimization for Robust MDP},
  author={Dong, Jing and Li, Jingwei and Wang, Baoxiang and Zhang, Jingzhao},
  journal={arXiv preprint arXiv:2209.13841},
  year={2022}
}

@inproceedings{sun2019model,
  title={Model-based rl in contextual decision processes: Pac bounds and exponential improvements over model-free approaches},
  author={Sun, Wen and Jiang, Nan and Krishnamurthy, Akshay and Agarwal, Alekh and Langford, John},
  booktitle={Conference on learning theory},
  pages={2898--2933},
  year={2019},
  organization={PMLR}
}

@inproceedings{liu2022welfare,
  title={Welfare maximization in competitive equilibrium: Reinforcement learning for markov exchange economy},
  author={Liu, Zhihan and Lu, Miao and Wang, Zhaoran and Jordan, Michael and Yang, Zhuoran},
  booktitle={International Conference on Machine Learning},
  pages={13870--13911},
  year={2022},
  organization={PMLR}
}

@book{sutton2018reinforcement,
  title={Reinforcement learning: An introduction},
  author={Sutton, Richard S and Barto, Andrew G},
  year={2018},
  publisher={MIT press}
}

@inproceedings{wang2018supervised,
  title={Supervised reinforcement learning with recurrent neural network for dynamic treatment recommendation},
  author={Wang, Lu and Zhang, Wei and He, Xiaofeng and Zha, Hongyuan},
  booktitle={Proceedings of the 24th ACM SIGKDD international conference on knowledge discovery \& data mining},
  pages={2447--2456},
  year={2018}
}

@article{kober2013reinforcement,
  title={Reinforcement learning in robotics: A survey},
  author={Kober, Jens and Bagnell, J Andrew and Peters, Jan},
  journal={The International Journal of Robotics Research},
  volume={32},
  number={11},
  pages={1238--1274},
  year={2013},
  publisher={SAGE Publications Sage UK: London, England}
}

@inproceedings{peng2018sim,
  title={Sim-to-real transfer of robotic control with dynamics randomization},
  author={Peng, Xue Bin and Andrychowicz, Marcin and Zaremba, Wojciech and Abbeel, Pieter},
  booktitle={2018 IEEE international conference on robotics and automation (ICRA)},
  pages={3803--3810},
  year={2018},
  organization={IEEE}
}

@inproceedings{zhao2020sim,
  title={Sim-to-real transfer in deep reinforcement learning for robotics: a survey},
  author={Zhao, Wenshuai and Queralta, Jorge Pe{\~n}a and Westerlund, Tomi},
  booktitle={2020 IEEE Symposium Series on Computational Intelligence (SSCI)},
  pages={737--744},
  year={2020},
  organization={IEEE}
}

@article{hu2022provable,
  title={Provable Sim-to-real Transfer in Continuous Domain with Partial Observations},
  author={Hu, Jiachen and Zhong, Han and Jin, Chi and Wang, Liwei},
  journal={arXiv preprint arXiv:2210.15598},
  year={2022}
}

@article{el2005robust,
  title={Robust solutions to markov decision problems with uncertain transition matrices},
  author={El Ghaoui, Laurent and Nilim, Arnab},
  journal={Operations Research},
  volume={53},
  number={5},
  pages={780--798},
  year={2005}
}

@inproceedings{pinto2017robust,
  title={Robust adversarial reinforcement learning},
  author={Pinto, Lerrel and Davidson, James and Sukthankar, Rahul and Gupta, Abhinav},
  booktitle={International Conference on Machine Learning},
  pages={2817--2826},
  year={2017},
  organization={PMLR}
}

@article{xie2020learning,
  title={Learning Zero-Sum Simultaneous-Move Markov Games Using Function Approximation and Correlated Equilibrium},
  author={Xie, Qiaomin and Chen, Yudong and Wang, Zhaoran and Yang, Zhuoran},
  journal={arXiv preprint arXiv:2002.07066},
  year={2020}
}

@article{tian2020provably,
  title={Provably Efficient Online Agnostic Learning in Markov Games},
  author={Tian, Yi and Wang, Yuanhao and Yu, Tiancheng and Sra, Suvrit},
  journal={arXiv preprint arXiv:2010.15020},
  year={2020}
}

@article{farhat2025sample,
  title={Sample-Efficient Distributionally Robust Multi-Agent Reinforcement Learning via Online Interaction},
  author={Farhat, Zain Ulabedeen and Ghosh, Debamita and Atia, George K. and Wang, Yue},
  journal={arXiv preprint arXiv:2508.02948},
  year={2025}
}

@article{zheng2025distributionally,
  title={Distributionally Robust Online Markov Game with Linear Function Approximation},
  author={Zheng, Zewu and Lin, Yuanyuan},
  journal={arXiv preprint arXiv:2511.07831},
  year={2025}
}

@article{shi2024sampleefficient,
  title={Sample-Efficient Robust Multi-Agent Reinforcement Learning in the Face of Environmental Uncertainty},
  author={Shi, Laixi and Mazumdar, Eric and Chi, Yuejie and Wierman, Adam},
  journal={arXiv preprint arXiv:2404.18909},
  year={2024}
}

@article{xu2010distributionally,
  title={Distributionally robust Markov decision processes},
  author={Xu, Huan and Mannor, Shie},
  journal={Advances in Neural Information Processing Systems},
  volume={23},
  year={2010}
}

@inproceedings{badrinath2021robust,
  title={Robust reinforcement learning using least squares policy iteration with provable performance guarantees},
  author={Badrinath, Kishan Panaganti and Kalathil, Dileep},
  booktitle={International Conference on Machine Learning},
  pages={511--520},
  year={2021},
  organization={PMLR}
}

@article{wang2021online,
  title={Online robust reinforcement learning with model uncertainty},
  author={Wang, Yue and Zou, Shaofeng},
  journal={Advances in Neural Information Processing Systems},
  volume={34},
  pages={7193--7206},
  year={2021}
}

@inproceedings{wang2023finite,
  title={A Finite Sample Complexity Bound for Distributionally Robust Q-learning},
  author={Wang, Shengbo and Si, Nian and Blanchet, Jose and Zhou, Zhengyuan},
  booktitle={International Conference on Artificial Intelligence and Statistics},
  pages={3370--3398},
  year={2023},
  organization={PMLR}
}

@article{si2023distributionally,
  title={Distributionally Robust Batch Contextual Bandits},
  author={Si, Nian and Zhang, Fan and Zhou, Zhengyuan and Blanchet, Jose},
  journal={Management Science},
  year={2023},
  publisher={INFORMS}
}

@inproceedings{wang2022policy,
  title={Policy gradient method for robust reinforcement learning},
  author={Wang, Yue and Zou, Shaofeng},
  booktitle={International Conference on Machine Learning},
  pages={23484--23526},
  year={2022},
  organization={PMLR}
}

@article{wang2022convergence,
  title={On the Convergence of Policy Gradient in Robust MDPs},
  author={Wang, Qiuhao and Ho, Chin Pang and Petrik, Marek},
  journal={arXiv preprint arXiv:2212.10439},
  year={2022}
}

@article{yang2023avoiding,
  title={Avoiding Model Estimation in Robust Markov Decision Processes with a Generative Model},
  author={Yang, Wenhao and Wang, Han and Kozuno, Tadashi and Jordan, Scott M and Zhang, Zhihua},
  journal={arXiv preprint arXiv:2302.01248},
  year={2023}
}

@inproceedings{kuang2022learning,
  title={Learning robust policy against disturbance in transition dynamics via state-conservative policy optimization},
  author={Kuang, Yufei and Lu, Miao and Wang, Jie and Zhou, Qi and Li, Bin and Li, Houqiang},
  booktitle={Proceedings of the AAAI Conference on Artificial Intelligence},
  volume={36},
  pages={7247--7254},
  year={2022}
}

@article{zhang2020robust,
  title={Robust deep reinforcement learning against adversarial perturbations on state observations},
  author={Zhang, Huan and Chen, Hongge and Xiao, Chaowei and Li, Bo and Liu, Mingyan and Boning, Duane and Hsieh, Cho-Jui},
  journal={Advances in Neural Information Processing Systems},
  volume={33},
  pages={21024--21037},
  year={2020}
}

@article{clavier2023towards,
  title={Towards Minimax Optimality of Model-based Robust Reinforcement Learning},
  author={Clavier, Pierre and Pennec, Erwan Le and Geist, Matthieu},
  journal={arXiv preprint arXiv:2302.05372},
  year={2023}
}

@inproceedings{xu2023improved,
  title={Improved Sample Complexity Bounds for Distributionally Robust Reinforcement Learning},
  author={Xu, Zaiyan and Panaganti, Kishan and Kalathil, Dileep},
  booktitle={International Conference on Artificial Intelligence and Statistics},
  pages={9728--9754},
  year={2023},
  organization={PMLR}
}

@article{kardes2005robust,
  title={Robust stochastic games and applications to counter-terrorism strategies},
  author={Kardes, Erim},
  journal={CREATE report},
  year={2005},
  publisher={Citeseer}
}

@inproceedings{zhang2021reinforcement,
  title={Is reinforcement learning more difficult than bandits? a near-optimal algorithm escaping the curse of horizon},
  author={Zhang, Zihan and Ji, Xiangyang and Du, Simon},
  booktitle={Conference on Learning Theory},
  pages={4528--4531},
  year={2021},
  organization={PMLR}
}

@article{zhang2023settling,
  title={Settling the sample complexity of online reinforcement learning},
  author={Zhang, Zihan and Chen, Yuxin and Lee, Jason D and Du, Simon S},
  journal={arXiv preprint arXiv:2307.13586},
  year={2023}
}

@article{zhong2022gec,
  title={Gec: A unified framework for interactive decision making in mdp, pomdp, and beyond},
  author={Zhong, Han and Xiong, Wei and Zheng, Sirui and Wang, Liwei and Wang, Zhaoran and Yang, Zhuoran and Zhang, Tong},
  journal={arXiv preprint arXiv:2211.01962},
  year={2022}
}

@inproceedings{wu2022nearly,
  title={Nearly optimal policy optimization with stable at any time guarantee},
  author={Wu, Tianhao and Yang, Yunchang and Zhong, Han and Wang, Liwei and Du, Simon and Jiao, Jiantao},
  booktitle={International Conference on Machine Learning},
  pages={24243--24265},
  year={2022},
  organization={PMLR}
}

@article{zhong2023theoretical,
  title={A theoretical analysis of optimistic proximal policy optimization in linear markov decision processes},
  author={Zhong, Han and Zhang, Tong},
  journal={arXiv preprint arXiv:2305.08841},
  year={2023}
}

@article{liu2023one,
  title={One Objective to Rule Them All: A Maximization Objective Fusing Estimation and Planning for Exploration},
  author={Liu, Zhihan and Lu, Miao and Xiong, Wei and Zhong, Han and Hu, Hao and Zhang, Shenao and Zheng, Sirui and Yang, Zhuoran and Wang, Zhaoran},
  journal={arXiv preprint arXiv:2305.18258},
  year={2023}
}

@article{foster2021statistical,
  title={The statistical complexity of interactive decision making},
  author={Foster, Dylan J and Kakade, Sham M and Qian, Jian and Rakhlin, Alexander},
  journal={arXiv preprint arXiv:2112.13487},
  year={2021}
}

@article{jin2021bellman,
  title={Bellman eluder dimension: New rich classes of rl problems, and sample-efficient algorithms},
  author={Jin, Chi and Liu, Qinghua and Miryoosefi, Sobhan},
  journal={Advances in neural information processing systems},
  volume={34},
  pages={13406--13418},
  year={2021}
}

@inproceedings{du2021bilinear,
  title={Bilinear classes: A structural framework for provable generalization in rl},
  author={Du, Simon and Kakade, Sham and Lee, Jason and Lovett, Shachar and Mahajan, Gaurav and Sun, Wen and Wang, Ruosong},
  booktitle={International Conference on Machine Learning},
  pages={2826--2836},
  year={2021},
  organization={PMLR}
}

@inproceedings{xu2023bayesian,
  title={Bayesian design principles for frequentist sequential learning},
  author={Xu, Yunbei and Zeevi, Assaf},
  booktitle={International Conference on Machine Learning},
  pages={38768--38800},
  year={2023},
  organization={PMLR}
}

@inproceedings{jiang2017contextual,
  title={Contextual decision processes with low bellman rank are pac-learnable},
  author={Jiang, Nan and Krishnamurthy, Akshay and Agarwal, Alekh and Langford, John and Schapire, Robert E},
  booktitle={International Conference on Machine Learning},
  pages={1704--1713},
  year={2017},
  organization={PMLR}
}

@inproceedings{he2023nearly,
  title={Nearly minimax optimal reinforcement learning for linear markov decision processes},
  author={He, Jiafan and Zhao, Heyang and Zhou, Dongruo and Gu, Quanquan},
  booktitle={International Conference on Machine Learning},
  pages={12790--12822},
  year={2023},
  organization={PMLR}
}

@inproceedings{agarwal2023vo,
  title={VO $ Q $ L: Towards Optimal Regret in Model-free RL with Nonlinear Function Approximation},
  author={Agarwal, Alekh and Jin, Yujia and Zhang, Tong},
  booktitle={The Thirty Sixth Annual Conference on Learning Theory},
  pages={987--1063},
  year={2023},
  organization={PMLR}
}

@inproceedings{zhou2021nearly,
  title={Nearly minimax optimal reinforcement learning for linear mixture markov decision processes},
  author={Zhou, Dongruo and Gu, Quanquan and Szepesvari, Csaba},
  booktitle={Conference on Learning Theory},
  pages={4532--4576},
  year={2021},
  organization={PMLR}
}

@inproceedings{ayoub2020model,
  title={Model-based reinforcement learning with value-targeted regression},
  author={Ayoub, Alex and Jia, Zeyu and Szepesvari, Csaba and Wang, Mengdi and Yang, Lin},
  booktitle={International Conference on Machine Learning},
  pages={463--474},
  year={2020},
  organization={PMLR}
}

@article{huang2023horizon,
  title={Horizon-Free and Instance-Dependent Regret Bounds for Reinforcement Learning with General Function Approximation},
  author={Huang, Jiayi and Zhong, Han and Wang, Liwei and Yang, Lin F},
  journal={arXiv preprint arXiv:2312.04464},
  year={2023}
}

@article{huang2023tackling,
  title={Tackling Heavy-Tailed Rewards in Reinforcement Learning with Function Approximation: Minimax Optimal and Instance-Dependent Regret Bounds},
  author={Huang, Jiayi and Zhong, Han and Wang, Liwei and Yang, Lin F},
  journal={arXiv preprint arXiv:2306.06836},
  year={2023}
}

@InProceedings{pmlr-v202-wang23i,
  title = 	 {Policy Gradient in Robust {MDP}s with Global Convergence Guarantee},
  author =       {Wang, Qiuhao and Ho, Chin Pang and Petrik, Marek},
  booktitle = 	 {Proceedings of the 40th International Conference on Machine Learning},
  pages = 	 {35763--35797},
  year = 	 {2023},
  editor = 	 {Krause, Andreas and Brunskill, Emma and Cho, Kyunghyun and Engelhardt, Barbara and Sabato, Sivan and Scarlett, Jonathan},
  volume = 	 {202},
  series = 	 {Proceedings of Machine Learning Research},
  month = 	 {23--29 Jul},
  publisher =    {PMLR},
  url = 	 {https://proceedings.mlr.press/v202/wang23i.html}
}

@inproceedings{yu2023fast,
  title={Fast Bellman Updates for Wasserstein Distributionally Robust MDPs},
  author={Yu, Zhuodong and Dai, Ling and Xu, Shaohang and Gao, Siyang and Ho, Chin Pang},
  booktitle={Thirty-seventh Conference on Neural Information Processing Systems},
  year={2023}
}

@inproceedings{zhou2023natural,
  title={Natural Actor-Critic for Robust Reinforcement Learning with Function Approximation},
  author={Zhou, Ruida and Liu, Tao and Cheng, Min and Kalathil, Dileep and Kumar, Panganamala and Tian, Chao},
  booktitle={Thirty-seventh Conference on Neural Information Processing Systems},
  year={2023}
}

@article{wang2023foundation,
  title={On the foundation of distributionally robust reinforcement learning},
  author={Wang, Shengbo and Si, Nian and Blanchet, Jose and Zhou, Zhengyuan},
  journal={arXiv preprint arXiv:2311.09018},
  year={2023}
}

@article{li2023first,
  title={First-order Policy Optimization for Robust Policy Evaluation},
  author={Li, Yan and Lan, Guanghui},
  journal={arXiv preprint arXiv:2307.15890},
  year={2023}
}

@inproceedings{zanette2019tighter,
  title={Tighter problem-dependent regret bounds in reinforcement learning without domain knowledge using value function bounds},
  author={Zanette, Andrea and Brunskill, Emma},
  booktitle={International Conference on Machine Learning},
  pages={7304--7312},
  year={2019},
  organization={PMLR}
}

@article{dann2017unifying,
  title={Unifying PAC and regret: Uniform PAC bounds for episodic reinforcement learning},
  author={Dann, Christoph and Lattimore, Tor and Brunskill, Emma},
  journal={Advances in Neural Information Processing Systems},
  volume={30},
  year={2017}
}

@article{zhang2020almost,
  title={Almost optimal model-free reinforcement learningvia reference-advantage decomposition},
  author={Zhang, Zihan and Zhou, Yuan and Ji, Xiangyang},
  journal={Advances in Neural Information Processing Systems},
  volume={33},
  pages={15198--15207},
  year={2020}
}

@article{li2023q,
  title={Is Q-learning minimax optimal? a tight sample complexity analysis},
  author={Li, Gen and Cai, Changxiao and Chen, Yuxin and Wei, Yuting and Chi, Yuejie},
  journal={Operations Research},
  year={2023},
  publisher={INFORMS}
}

@inproceedings{menard2021ucb,
  title={UCB Momentum Q-learning: Correcting the bias without forgetting},
  author={M{\'e}nard, Pierre and Domingues, Omar Darwiche and Shang, Xuedong and Valko, Michal},
  booktitle={International Conference on Machine Learning},
  pages={7609--7618},
  year={2021},
  organization={PMLR}
}

@article{moos2022robust,
  title={Robust reinforcement learning: A review of foundations and recent advances},
  author={Moos, Janosch and Hansel, Kay and Abdulsamad, Hany and Stark, Svenja and Clever, Debora and Peters, Jan},
  journal={Machine Learning and Knowledge Extraction},
  volume={4},
  number={1},
  pages={276--315},
  year={2022},
  publisher={MDPI}
}

@article{liu2024distributionally,
  title={Distributionally Robust Off-Dynamics Reinforcement Learning: Provable Efficiency with Linear Function Approximation},
  author={Liu, Zhishuai and Xu, Pan},
  journal={arXiv preprint arXiv:2402.15399},
  year={2024}
}

@article{ding2024seeing,
  title={Seeing is not believing: Robust reinforcement learning against spurious correlation},
  author={Ding, Wenhao and Shi, Laixi and Chi, Yuejie and Zhao, Ding},
  journal={Advances in Neural Information Processing Systems},
  volume={36},
  year={2024}
}

@article{liu2024minimax,
  title={Minimax Optimal and Computationally Efficient Algorithms for Distributionally Robust Offline Reinforcement Learning},
  author={Liu, Zhishuai and Xu, Pan},
  journal={arXiv preprint arXiv:2403.09621},
  year={2024}
}

@article{wang2024sample,
  title={Sample Complexity of Offline Distributionally Robust Linear Markov Decision Processes},
  author={Wang, He and Shi, Laixi and Chi, Yuejie},
  journal={arXiv preprint arXiv:2403.12946},
  year={2024}
}

@article{shapley1953stochastic,
  title={Stochastic games},
  author={Shapley, Lloyd S},
  journal={Proceedings of the national academy of sciences},
  volume={39},
  number={10},
  pages={1095--1100},
  year={1953},
  publisher={National Academy of Sciences}
}

@incollection{littman1994markov,
  title={Markov games as a framework for multi-agent reinforcement learning},
  author={Littman, Michael L},
  booktitle={Machine learning proceedings 1994},
  pages={157--163},
  year={1994},
  publisher={Elsevier}
}

@article{huh2011adaptive,
  title={Adaptive data-driven inventory control with censored demand based on Kaplan-Meier estimator},
  author={Huh, Woonghee Tim and Levi, Retsef and Rusmevichientong, Paat and Orlin, James B},
  journal={Operations Research},
  volume={59},
  number={4},
  pages={929--941},
  year={2011},
  publisher={INFORMS}
}

@article{shi2016nonparametric,
  title={Nonparametric data-driven algorithms for multiproduct inventory systems with censored demand},
  author={Shi, Cong and Chen, Weidong and Duenyas, Izak},
  journal={Operations Research},
  volume={64},
  number={2},
  pages={362--370},
  year={2016},
  publisher={INFORMS}
}

@article{fan2024don,
  title={Don’t follow rl blindly: Lower sample complexity of learning optimal inventory control policies with fixed ordering costs},
  author={Fan, Xiaoyu and Chen, Boxiao and Lennon Olsen, Tava and Qin, Hanzhang and Zhou, Zhengyuan},
  journal={Available at SSRN 4828001},
  year={2024}
}

@article{lyu2024ucb,
  title={Ucb-type learning algorithms with kaplan--meier estimator for lost-sales inventory models with lead times},
  author={Lyu, Chengyi and Zhang, Huanan and Xin, Linwei},
  journal={Operations Research},
  volume={72},
  number={4},
  pages={1317--1332},
  year={2024},
  publisher={INFORMS}
}

@article{yuan2021marrying,
  title={Marrying stochastic gradient descent with bandits: Learning algorithms for inventory systems with fixed costs},
  author={Yuan, Hao and Luo, Qi and Shi, Cong},
  journal={Management Science},
  volume={67},
  number={10},
  pages={6089--6115},
  year={2021},
  publisher={Informs}
}

@article{zhang2020closing,
  title={Closing the gap: A learning algorithm for lost-sales inventory systems with lead times},
  author={Zhang, Huanan and Chao, Xiuli and Shi, Cong},
  journal={Management Science},
  volume={66},
  number={5},
  pages={1962--1980},
  year={2020},
  publisher={INFORMS}
}

@inproceedings{agrawal2019learning,
  title={Learning in structured mdps with convex cost functions: Improved regret bounds for inventory management},
  author={Agrawal, Shipra and Jia, Randy},
  booktitle={Proceedings of the 2019 ACM Conference on Economics and Computation},
  pages={743--744},
  year={2019}
}

@article{xin2022distributionally,
  title={Distributionally robust inventory control when demand is a martingale},
  author={Xin, Linwei and Goldberg, David Alan},
  journal={Mathematics of Operations Research},
  volume={47},
  number={3},
  pages={2387--2414},
  year={2022},
  publisher={INFORMS}
}

@article{klabjan2013robust,
  title={Robust stochastic lot-sizing by means of histograms},
  author={Klabjan, Diego and Simchi-Levi, David and Song, Miao},
  journal={Production and Operations Management},
  volume={22},
  number={3},
  pages={691--710},
  year={2013},
  publisher={SAGE Publications Sage CA: Los Angeles, CA}
}

@techreport{scarf1957min,
  title={A min-max solution of an inventory problem},
  author={Scarf, Herbert E and Arrow, KJ and Karlin, S},
  year={1957},
  institution={Rand Corporation Santa Monica}
}

@article{bertsimas2006robust,
  title={A robust optimization approach to inventory theory},
  author={Bertsimas, Dimitris and Thiele, Aur{\'e}lie},
  journal={Operations research},
  volume={54},
  number={1},
  pages={150--168},
  year={2006},
  publisher={INFORMS}
}

@article{shi2024breaking,
  title={Breaking the curse of multiagency in robust multi-agent reinforcement learning},
  author={Shi, Laixi and Gai, Jingchu and Mazumdar, Eric and Chi, Yuejie and Wierman, Adam},
  journal={arXiv preprint arXiv:2409.20067},
  year={2024}
}

@article{ghosh2025scaling,
  title={Scaling Online Distributionally Robust Reinforcement Learning: Sample-Efficient Guarantees with General Function Approximation},
  author={Ghosh, Debamita and Atia, George K and Wang, Yue},
  journal={arXiv preprint arXiv:2512.18957},
  year={2025}
}

@article{ghosh2025orvit,
  title={ORVIT: Near-Optimal Online Distributionally Robust Reinforcement Learning},
  author={Ghosh, Debamita and Atia, George K and Wang, Yue},
  journal={arXiv preprint arXiv:2508.03768},
  year={2025}
}

@article{he2025sample,
  title={Sample complexity of distributionally robust off-dynamics reinforcement learning with online interaction},
  author={He, Yiting and Liu, Zhishuai and Wang, Weixin and Xu, Pan},
  journal={arXiv preprint arXiv:2511.05396},
  year={2025}
}

@article{liu2024upper,
  title={Upper and lower bounds for distributionally robust off-dynamics reinforcement learning},
  author={Liu, Zhishuai and Wang, Weixin and Xu, Pan},
  journal={arXiv preprint arXiv:2409.20521},
  year={2024}
}

\bibliographystyle{ims}

\newpage 

\appendix 

\section{Proofs for Properties of RMDPs with TV Robust Sets}

\subsection{Proof of Proposition~\ref{prop: strong duality}}\label{subsec: proof prop strong duality}

To simplify the notations, we present the following lemma, which directly implies Proposition~\ref{prop: strong duality}.

\begin{lemma}[Strong duality for TV robust set]\label{lem: tv}
    The following duality for total variation robust set holds,
    for $f:\mathcal{S}\mapsto[0,H]$,
    \begin{align*}
        \inf_{Q(\cdot):D_{\mathrm{TV}}(Q(\cdot)\|Q^{\star}(\cdot))\leq \sigma}\mathbb{E}_{Q(\cdot)}[f]
        = \sup_{\eta\in[0,H]} \left\{ - \mathbb{E}_{Q^{\star}(\cdot)}\big[(\eta-f)_+\big] - \sigma\cdot \left(\eta-\min_{s\in\cS} f(s)\right)_+ + \eta \right\},
    \end{align*}
    where $ \sigma\in[0,1]$ and the TV distance $D_{\mathrm{TV}}(Q(\cdot)\|Q^{\star}(\cdot))$ is defined as
    \begin{align}
        D_{\mathrm{TV}}(Q(\cdot)\|Q^{\star}(\cdot)) = \frac{1}{2}\sum_{s\in\cS}|Q(s) - Q^{\star}(s)|.
    \end{align}
\end{lemma}

\begin{proof}[Proof of Lemma \ref{lem: tv}]
    First, we note that when $Q^{\star}(s)>0$ for any $s\in\cS$, i.e., any $Q(\cdot)\in\Delta(\cS)$ is absolute continuous w.r.t. $Q^{\star}(\cdot)$, adapting the TV convention in \cite{yang2022toward} to Definition~\ref{def: tv}, we have that
    \begin{align*}
        \inf_{Q(\cdot):D_{\mathrm{TV}}(Q(\cdot)\|Q^{\star}(\cdot))\leq \sigma}\mathbb{E}_{Q(\cdot)}[f]
        = \sup_{\eta\in\mathbb{R}} \left\{ - \mathbb{E}_{Q^{\star}(\cdot)}\big[(\eta-f)_+\big] - \sigma\cdot \left(\eta-\min_{s\in\cS} f(s)\right)_+ + \eta \right\}.
    \end{align*}
    Furthermore, as is shown in Lemma H.8 in \cite{blanchet2023double}, the optimal dual variable $\eta^{\star}$ lies in $[0,H]$ when $f\in[0,H]$.
    Therefore, for $Q^{\star}(\cdot)$ such that $Q^{\star}(s)>0$ for any $s\in\cS$, we have
    \begin{align*}
        \inf_{Q(\cdot):D_{\mathrm{TV}}(Q(\cdot)\|Q^{\star}(\cdot))\leq \sigma}\mathbb{E}_{Q(\cdot)}[f]
        = \sup_{\eta\in[0,H]} \left\{ - \mathbb{E}_{Q^{\star}(\cdot)}\big[(\eta-f)_+\big] - \sigma\cdot \left(\eta-\min_{s\in\cS} f(s)\right)_+ + \eta \right\}.
    \end{align*}
    Now for any $Q^{\star}(\cdot)\in\Delta(\cS)$, we can prove the same result by averaging $Q^{\star}(\cdot)$ with a uniform distribution and taking the limit.
    More specifically, denote $U(\cdot)\in\Delta(\cS)$ as the uniform distribution on $\cS$, i.e., $U(s) = 1/|\cS|$ for any $s\in\cS$.
    Consider the following distributionally robust optimization problem, for any $\epsilon\in[0,1]$,
    \begin{align}
        \mathtt{P}(\epsilon):=\inf_{Q(\cdot):D_{\mathrm{TV}}\big(Q(\cdot)\|(1-\epsilon)Q^{\star}(\cdot) + \epsilon\cdot U(\cdot)\big)\leq \sigma}\mathbb{E}_{Q(\cdot)}[f].
    \end{align}
    By our previous discussions, since $(1-\epsilon)Q^{\star}(s) + \epsilon\cdot U(s)>0$ for any $s\in\cS$ and $\epsilon>0$, we have that
    \begin{align}
        \mathtt{P}(\epsilon) = \mathtt{D}(\epsilon),\quad \forall \epsilon\in(0,1],\label{eq: proof strong duality 1}
    \end{align}
    where the function $\mathtt{D}(\cdot):[0,1]\mapsto\mathbb{R}_+$ is defined as
    \begin{align}
        \mathtt{D}(\epsilon):=\sup_{\eta\in[0,H]} \left\{ - (1-\epsilon)\cdot \mathbb{E}_{Q^{\star}(\cdot)}\big[(\eta-f)_+\big] - \epsilon\cdot\mathbb{E}_{U(\cdot)}\big[(\eta-f)_+\big]- \sigma\cdot \left(\eta-\min_{s\in\cS} f(s)\right)_+ + \eta \right\}.
    \end{align}
    By the definition of $\mathtt{P}(\cdot)$ and $\mathtt{D}(\cdot)$, our goal is to prove that $\mathtt{P}(0) = \mathtt{D}(0)$.
    To this end, it suffices to prove that (i) $\lim_{\epsilon\rightarrow 0+}\mathtt{D}(\epsilon)$ exists and $\lim_{\epsilon\rightarrow 0+}\mathtt{D}(\epsilon) = \mathtt{D}(0)$; and (ii) $\lim_{\epsilon\rightarrow 0+}\mathtt{P}(\epsilon) = \mathtt{P}(0)$.
    To prove (i), consider that for any $\epsilon>0$, by the definition of $\mathtt{D}(\cdot)$,
    \begin{align}
        \left|\mathtt{D}(0) - \mathtt{D}(\epsilon)\right| \leq \sup_{\eta\in[0,H]}\Big\{\epsilon\cdot \mathbb{E}_{Q^{\star}(\cdot)}\big[(\eta-f)_+\big] +  \epsilon\cdot\mathbb{E}_{U(\cdot)}\big[(\eta-f)_+\big]\Big\} \leq \epsilon\cdot 2H.
    \end{align}
    Since the right hand side tends to $0$ as $\epsilon$ tends to $0$, we know that $\lim_{\epsilon\rightarrow 0+}\mathtt{D}(\epsilon)$ exists, $\lim_{\epsilon\rightarrow 0+}\mathtt{D}(\epsilon) = \mathtt{D}(0)$.
    This also indicates that $\lim_{\epsilon\rightarrow 0+}\mathtt{P}(\epsilon)$ exists due to \eqref{eq: proof strong duality 1}.
    This proves (i).
    Now we prove (ii).
    Notice that since the set
    \begin{align}
        \left\{Q(\cdot)\in\Delta(\cS):D_{\mathrm{TV}}\big(Q(\cdot)\|(1-\epsilon)Q^{\star}(\cdot) + \epsilon\cdot U(\cdot)\big)\leq \sigma\right\}
    \end{align}
    is a closed subset of $\mathbb{R}^{|\cS|}$, and $\mathbb{E}_{Q(\cdot)}[f]$ is a continuous function of $Q(\cdot)\in\mathbb{R}^{|\cS|}$ w.r.t. the $\|\cdot\|_2$-norm, we can denote the optimal solution to the optimization problem involved in $\mathtt{P}(\epsilon)$ as
    \begin{align}
        Q_{\epsilon}^{\dagger}(\cdot) = \arginf_{Q(\cdot):D_{\mathrm{TV}}\big(Q(\cdot)\|(1-\epsilon)Q^{\star}(\cdot) + \epsilon\cdot U(\cdot)\big)\leq \sigma}\mathbb{E}_{Q(\cdot)}[f],
    \end{align}
    which also gives that
    \begin{align}
        \mathtt{P}(\epsilon) = \mathbb{E}_{Q_{\epsilon}^{\dagger}(\cdot)}[f] = \sum_{s\in\cS}Q_{\epsilon}^{\dagger}(s)f(s).
    \end{align}
    With these preparations, we are able to prove (ii).
    On the one hand, consider for any $\epsilon\in(0,1]$,
    \begin{align}
        D_{\mathrm{TV}}\big((1-\epsilon)\cdot Q_{0}^{\dagger}(\cdot)+ \epsilon\cdot U(\cdot)\big\|(1-\epsilon)\cdot Q^{\star}(\cdot) + \epsilon\cdot U(\cdot)\big) \leq (1-\epsilon)\cdot \sigma \leq \sigma.
    \end{align}
    Therefore, for any $\epsilon\in(0,1]$, it holds that
    \begin{align}
        \mathtt{P}(\epsilon) = \inf_{Q(\cdot):D_{\mathrm{TV}}\big(Q(\cdot)\|(1-\epsilon)Q^{\star}(\cdot) + \epsilon\cdot U(\cdot)\big)\leq \sigma}\mathbb{E}_{Q(\cdot)}[f] \leq \mathbb{E}_{(1-\epsilon)\cdot Q_{0}^{\dagger}(\cdot)+ \epsilon\cdot U(\cdot)}[f] = (1-\epsilon)\cdot\mathbb{E}_{Q_0^{\dagger}}[f]+\epsilon\cdot \mathbb{E}_{U(\cdot)}[f],
    \end{align}
    which implies that
    \begin{align}
        \lim_{\epsilon\rightarrow 0+}\mathtt{P}(\epsilon)\leq  \mathbb{E}_{Q_0^{\dagger}}[f] = \mathtt{P}(0).\label{eq: proof strong duality 2-}
    \end{align}
    On the other hand, for any $\epsilon\in(0,1]$,
    \begin{align}
        \sigma\geq \frac{1}{2}\sum_{s\in\cS}\Big|Q^{\dagger}_{\epsilon}(s) - (1-\epsilon)\cdot Q^{\star}(s) - \epsilon\cdot U(s) \Big|\geq (1-\epsilon)\cdot D_{\mathrm{TV}}(Q^{\dagger}_{\epsilon}(\cdot)\|Q^{\star}(\cdot)) -\epsilon\cdot D_{\mathrm{TV}}(Q^{\dagger}_{\epsilon}(\cdot)\|U(\cdot)),
    \end{align}
    and by using $D_{\mathrm{TV}}(Q^{\dagger}_{\epsilon}(\cdot)\|U(\cdot))\leq 1$, we obtain that
    \begin{align}
        D_{\mathrm{TV}}(Q^{\dagger}_{\epsilon}(\cdot)\|Q^{\star}(\cdot))  \leq \frac{\sigma + \epsilon}{1-\epsilon}.\label{eq: proof strong duality 2}
    \end{align}
    Consider a sequence of $\{\epsilon_i\}_{i=1}^{\infty}$ converging to $0$, i.e., $\lim_{i\rightarrow \infty}\epsilon_i = 0$.
    Since $\{Q^{\dagger}_{\epsilon_i}(\cdot)\}_{i=1}^{\infty}$ is a sequence contained in a compact subset of $\mathbb{R}^{|\cS|}$, it has a converging (w.r.t. $\|\cdot\|_2$) subsequence denoted by $\{Q^{\dagger}_{\epsilon_{i_k}}(\cdot)\}_{k=1}^{\infty}$ whose limit is denoted as $Q^{\dagger}(\cdot)\in\Delta(\cS)$.
    By \eqref{eq: proof strong duality 2}, we know that
    \begin{align}
        D_{\mathrm{TV}}(Q^{\dagger}_{\epsilon_{i_k}}(\cdot)\|Q^{\star}(\cdot)) \leq \frac{\sigma + \epsilon_{i_k}}{1-\epsilon_{i_k}}.\label{eq: proof strong duality 3}
    \end{align}
    Taking limit on both sides of \eqref{eq: proof strong duality 3} (limit of LHS exists since the TV distance is a continuous function (w.r.t. $\|\cdot\|_2$) of its first entry and the limit of RHS obviously exists), we obtain that
    \begin{align}
        D_{\mathrm{TV}}(Q^{\dagger}(\cdot)\|Q^{\star}(\cdot)) \leq \sigma.\label{eq: proof strong duality 4}
    \end{align}
    Now we can arrive at the following,
    \begin{align}
        \lim_{\epsilon\rightarrow 0+}\mathtt{P}(\epsilon) = \lim_{\epsilon\rightarrow 0+}\mathbb{E}_{Q_{\epsilon}^{\dagger}(\cdot)}[f] = \lim_{k\rightarrow \infty}\mathbb{E}_{Q_{\epsilon_{i_k}}^{\dagger}(\cdot)}[f] = \mathbb{E}_{Q^{\dagger}(\cdot)}[f]\geq  \inf_{Q(\cdot):D_{\mathrm{TV}}(Q(\cdot)\|Q^{\star}(\cdot))\leq \sigma}\mathbb{E}_{Q(\cdot)}[f] = \mathtt{P}(0),\label{eq: proof strong duality 5}
    \end{align}
    where the first and the last equality follows from the definition of $\mathtt{P}(\cdot)$, the second equality follows from the choice of the sequence $\{\epsilon_{i_k}\}_{k=1}^{\infty}$ that converges to $0$, the third equality is due to the continuity of $\mathbb{E}_{Q(\cdot)}[f]$ of $Q(\cdot)$ (w.r.t. $\|\cdot\|_2$), and the inequality follows from \eqref{eq: proof strong duality 4}.
    Finally, with \eqref{eq: proof strong duality 2-} and \eqref{eq: proof strong duality 5}, we conclude that
    \begin{align}
        \lim_{\epsilon\rightarrow 0+}\mathtt{P}(\epsilon) = \mathtt{P}(0),
    \end{align}
    which proves (ii).
    Consequently, by (i) and (ii)
    \begin{align}
        \mathtt{P}(0) = \lim_{\epsilon\rightarrow 0+}\mathtt{P}(\epsilon) = \lim_{\epsilon\rightarrow 0+}\mathtt{D}(\epsilon) = \mathtt{D}(0).
    \end{align}
    Recalling the definitions of $\mathtt{P}(\cdot)$ and $\mathtt{D}(\cdot)$, we conclude the proof of Lemma~\ref{lem: tv}.
\end{proof}

\subsection{Proof of Proposition~\ref{prop: gap}}\label{subsec: proof prop gap}

\begin{proof}[Proof of Proposition~\ref{prop: gap}]
    Here we prove a stronger result that for any policy $\pi$ and step $h\in[H]$
    \begin{align}
        \max_{(s,a)\in\cS\times\cA} Q_{h,P,\boldsymbol{\Phi}}^{\pi} (s,a) - \min_{(s,a)\in\cS\times\cA}Q_{h,P,\boldsymbol{\Phi}}^{\pi} (s,a) &\leq \frac{1}{\rho}\cdot\Big(1 - (1-\rho)^{H-h+1}\Big),\label{eq: proof gap q}\\
        \max_{s\in\cS} V_{h,P,\boldsymbol{\Phi}}^{\pi} (s) - \min_{s\in\cS}V_{h,P,\boldsymbol{\Phi}}^{\pi} (s) &\leq \frac{1}{\rho}\cdot\Big(1 - (1-\rho)^{H-h+1}\Big).\label{eq: proof gap v}
    \end{align}
    First, we note that for the last step $h=H$, \eqref{eq: proof gap q} and \eqref{eq: proof gap v} naturally hold since $R_H\in[0,1]$.
    Now suppose that \eqref{eq: proof gap v} hold for some step $h+1$.
    By robust Bellman equation (Proposition~\ref{prop: robust bellman equation}), we have that
    \begin{align}
         \!\!\!\!\!\!\!\!\!Q_{h, P^\star, \boldsymbol{\Phi}}^{\pi}(s, a) &= R_{h}(s, a) + \EE_{\cP_\rho(s, a; P_h^\star)} \Big[V_{h+1, P^\star, \boldsymbol{\Phi}}^{\pi}\Big] \le 1 + \EE_{\cP_\rho(s, a; P_h^\star)}\Big[V_{h+1, P^\star, \boldsymbol{\Phi}}^\pi\Big],\quad\forall (s,a)\in\cS\times\cA,\label{eq: proof gap 0}
    \end{align}
    where the inequality uses the fact that $R_h\le 1$.
    Now we denote the state with the least robust value as
    \begin{align}
        s_0 \in\argmin_{s \in \cS} V_{h+1, P^\star, \boldsymbol{\Phi}}^{\pi}(s).\label{eq: proof gap 1-}
    \end{align}
    Inspired by \citet{shi2023curious}, we choose a transition kernel $\widetilde{P}_h$ satisfying that
    \begin{align}
    \Big\| \widetilde{P}_h(\cdot | s, a)\Big\|_1  = 1 - \rho,\quad  P_h^\star(s' | s, a) \ge \widetilde{P}_h(s'| s, a) \ge 0, \quad \forall (s,a,s')\in\cS\times\cA\times\cS,\label{eq: proof gap 1}
    \end{align}
    which implies that
    \begin{align}
    D_{\mathrm{TV}}\left( \widetilde{P}_h(\cdot | s, a) + \rho \cdot \delta_{s_0}(\cdot) \,\middle\|\, P_h^\star(\cdot | s, a) \right) \le \rho,\quad \forall (s,a)\in\cS\times\cA.
    \end{align}
    Here $\delta_{s_0}(\cdot)$ is the point measure centered at $s_0$ defined in \eqref{eq: proof gap 1-}.
    Combined with \eqref{eq: proof gap 0}, we have that
    \begin{align}
        Q_{h, P^\star, \boldsymbol{\Phi}}^{\pi}(s, a) &\leq 1 + \mathbb{E}_{\widetilde{P}_h(\cdot|s,a) +\rho\cdot\delta_{s_0}(\cdot)}\Big[V_{h+1, P^\star, \boldsymbol{\Phi}}^\pi\Big] \\
        &= 1 + \mathbb{E}_{\widetilde{P}_h(\cdot|s,a) }\Big[V_{h+1, P^\star, \boldsymbol{\Phi}}^\pi\Big] +\rho\cdot V_{h+1, P^\star, \boldsymbol{\Phi}}^\pi(s_0)\\
        & \leq 1+(1-\rho)\cdot\max_{s\in\cS}V_{h+1, P^\star, \boldsymbol{\Phi}}^\pi(s) + \rho\cdot \min_{s\in\cS}V_{h+1, P^\star, \boldsymbol{\Phi}}^\pi(s).\label{eq: proof gap 2}
    \end{align}
    Consequently from \eqref{eq: proof gap 2}, we further obtain that for any $(s,a)\in\cS\times\cA$,
    \begin{align}
        &Q_{h, P^\star, \boldsymbol{\Phi}}^{\pi}(s, a)  - \min_{(s,a)\in\cS\times\cA}Q_{h, P^\star, \boldsymbol{\Phi}}^{\pi}(s, a)\\
        &\quad \leq 1+(1-\rho)\cdot\max_{s\in\cS}V_{h+1, P^\star, \boldsymbol{\Phi}}^\pi(s) + \rho\cdot \min_{s\in\cS}V_{h+1, P^\star, \boldsymbol{\Phi}}^\pi(s)- \min_{(s,a)\in\cS\times\cA}Q_{h, P^\star, \boldsymbol{\Phi}}^{\pi}(s, a) \\
        &\quad=1+(1-\rho)\cdot\left(\max_{s\in\cS}V_{h+1, P^\star, \boldsymbol{\Phi}}^\pi(s) - \min_{s\in\cS}V_{h+1, P^\star, \boldsymbol{\Phi}}^\pi(s)\right) + \min_{s\in\cS}V_{h+1, P^\star, \boldsymbol{\Phi}}^\pi(s) - \min_{(s,a)\in\cS\times\cA}Q_{h, P^\star, \boldsymbol{\Phi}}^{\pi}(s, a) \\
        &\quad \leq 1+(1-\rho)\cdot\left(\max_{s\in\cS}V_{h+1, P^\star, \boldsymbol{\Phi}}^\pi(s) - \min_{s\in\cS}V_{h+1, P^\star, \boldsymbol{\Phi}}^\pi(s)\right),\label{eq: proof gap 3}
    \end{align}
    where the first inequality uses \eqref{eq: proof gap 2} and the last inequality uses the following fact,
    \begin{align}
        \min_{(s,a)\in\cS\times\cA}Q_{h, P^\star, \boldsymbol{\Phi}}^{\pi}(s, a) = \min_{(s,a)\in\cS\times\cA}\bigg\{R_{h}(s, a) + \EE_{\cP_\rho(s, a; P_h^\star)} \Big[V_{h+1, P^\star, \boldsymbol{\Phi}}^{\pi}\Big]\bigg\} \geq \min_{s\in\cS}V_{h+1, P^\star, \boldsymbol{\Phi}}^\pi(s).
    \end{align}
    Now applying the assumption that \eqref{eq: proof gap v} holds at step $h+1$ to the right hand side of \eqref{eq: proof gap 3}, we obtain that
    \begin{align}
        \max_{(s,a)\in\cS\times\cA}Q_{h, P^\star, \boldsymbol{\Phi}}^{\pi}(s, a)  - \min_{(s,a)\in\cS\times\cA}Q_{h, P^\star, \boldsymbol{\Phi}}^{\pi}(s, a) &\leq 1+ \frac{1-\rho}{\rho}\cdot\Big(1 - (1-\rho)^{H-h}\Big) \\
        &= \frac{1}{\rho}\cdot\Big(1 - (1-\rho)^{H-h+1}\Big).
    \end{align}
    Thus given \eqref{eq: proof gap v} at step $h+1$, we can derive \eqref{eq: proof gap q} at step $h$.
    Now by noticing that
    \begin{align}
        \min_{(s,a)\in\cS\times\cA}Q_{h, P^\star, \boldsymbol{\Phi}}^{\pi}(s, a) \leq \min_{s\in\cS}V_{h, P^\star, \boldsymbol{\Phi}}^{\pi}(s) \leq \max_{s\in\cS}V_{h, P^\star, \boldsymbol{\Phi}}^{\pi}(s) \leq  \max_{(s,a)\in\cS\times\cA}Q_{h, P^\star, \boldsymbol{\Phi}}^{\pi}(s, a),
    \end{align}
    we can conclude that \eqref{eq: proof gap v} also holds at step $h$.
    As a result, by an induction argument, we finish the proof of Proposition~\ref{prop: gap}.
\end{proof}

\subsection{Proof of Proposition~\ref{prop: equivalent robust set}}\label{subsec: proof prop equivalent robust set}

\begin{proof}[Proof of Proposition~\ref{prop: equivalent robust set}]
    We fix $(s, a, h) \in \cS \times \cA \times [H]$ throughout the proof.
    By Lemma~\ref{lem: tv}, we have that
    \#
     \EE_{\cP_\rho(s, a; P_h^\star)}\left[ V \right]
        & = \sup_{\eta\in\mathbb{R}} \left\{ - \mathbb{E}_{P_h^{\star}(\cdot|s, a)}\big[(\eta-V)_+\big] - \rho\cdot \left(\eta-\min_{s\in\cS} V(s)\right)_+ + \eta \right\} \notag \\
        & = \sup_{\eta\in[0, H]} \left\{ - \mathbb{E}_{ P_h^{\star}(\cdot|s, a)}\big[(\eta-V)_+\big] - \rho\cdot \left(\eta-\min_{s\in\cS} V(s)\right)_+ + \eta \right\} \notag \\
        & = \sup_{\eta\in[0, H]} \bigg\{ - \mathbb{E}_{ P_h^{\star}(\cdot|s, a)}\big[(\eta-V)_+\big] + \left(1 - \rho \right) \cdot \eta \bigg\}, \label{eq:8881}
    \#
    where the second equality follows from the fact the optimal dual variable $\eta^\star$ is in $[0, H]$ when $V\in[0,H]$ (see e.g., Lemma H.8 in \citet{blanchet2023double}), and the last equality is obtained by the fact that $\min_{s \in \cS} V(s) = 0$.
    \paragraph{Part (i).} For any $\eta \in [0, H]$ and $Q \in \cB_{\rho'}(s, a; P_h^{\star})$, we have that
    \#
    - \mathbb{E}_{P_h^{\star}(\cdot|s, a)}\big[(\eta-V)_+\big] + \left(1 - \rho \right) \cdot \eta & \le \left(1 - \rho \right) \cdot \Big( - \mathbb{E}_{ Q(\cdot)}\big[(\eta-V)_+\big] + \eta \Big) \notag \\
    & \le \left(1 - \rho \right) \cdot \Big( - \mathbb{E}_{Q(\cdot)}\big[\eta-V\big] + \eta \Big) \notag \\
    & = \left(1 - \rho \right) \cdot \mathbb{E}_{Q(\cdot)}\big[V\big],\label{eq:8881+}
    \#
    where the first inequality uses the definition of $\cB_{\rho'}(s, a; P_h^{\star})$, and the second inequality follows from $(x)_+ \ge x$.
    Furthermore, since \eqref{eq:8881+} holds for any $\eta \in [0, H]$ and $Q \in \cB_{\rho'}(s, a; P_h^{\star})$, we have that
    \#
    \sup_{\eta\in[0, H]} \bigg\{ - \mathbb{E}_{ P_h^{\star}(\cdot|s, a)}\big[(\eta-V)_+\big] + \left(1 - \rho \right) \cdot \eta \bigg\} \le \left(1 - \rho \right) \cdot \inf_{Q \in\cB_{\rho'}(s, a;P_h^{\star}) } \mathbb{E}_{ Q(\cdot)}\big[V\big].\label{eq:8882}
    \#
    Combining \eqref{eq:8881}  and \eqref{eq:8882}, we conclude that
    \$
    \EE_{\cP_\rho(s, a; P_h^\star)}\big[V\big]\le \rho' \cdot \EE_{ \cB_{\rho'}(s, a; P_h^{\star}) }   \big[V\big].
    \$
    \paragraph{Part (ii).} Since $\rho \in [0, 1)$, we know that there exists a $\widetilde{\eta} \in [0, H]$ such that
    \$
    \sum_{s': V(s') < \widetilde{\eta}} P_h^\star(s' | s, a) \le 1 - \rho \le \sum_{s': V(s') \le \widetilde{\eta}} P_h^\star(s' | s, a),
    \$
    which further implies that we have the following interpolation for some $\lambda \in [0, 1]$:
    \$
    1 - \rho = \lambda \sum_{s': V(s') < \widetilde{\eta}} P_h^\star(s' | s, a) + (1 - \lambda) \sum_{s': V(s') \le \widetilde{\eta}} P_h^\star(s' | s, a).
    \$
    We define a probability measure $\widetilde{P}_h^\star(\cdot) \in \Delta(\cS)$ as
    \# \label{eq:8883}
\widetilde{P}_h^\star(s')  = \frac{\lambda P_h^\star(s'|s, a) \cdot \mathbf{1}\{V(s') < \widetilde{\eta}\} + (1 - \lambda) P_h^\star(s'|s, a) \cdot \mathbf{1}\{V(s') \le \widetilde{\eta}\} }{1 - \rho}.
    \#
    It is not difficult to verify that $\widetilde{P}_h^\star  \in \cB_{\rho'}(s, a;P^{\star}_h)$. Hence, we have
    \#
    \left( 1 - \rho\right) \cdot \EE_{ \cB_{\rho'}(s, a;P^{\star}_h)}   [V] & \le \left( 1 - \rho\right) \cdot \EE_{ \widetilde{P}_h^\star(\cdot) }   \big[V\big] \\
    & = \left( 1 - \rho\right) \cdot \EE_{ \widetilde{P}_h^\star(\cdot) }   \big[V- \widetilde{\eta}\big] + \left( 1 - \rho\right) \cdot  \widetilde{\eta} \\
    & = - \EE_{P_h^\star(\cdot | s, a)} \big[(\widetilde{\eta}-V)_+\big] + \left( 1 - \rho\right)  \cdot\widetilde{\eta} ,\label{eq:8883+}
    \#
    where the last equality uses the definition of $\widetilde{P}^\star_h$ in \eqref{eq:8883}. Furthermore, by \eqref{eq:8883+} we have that
    \# \label{eq:8884}
    \rho' \cdot \EE_{ \cB_{\rho'}(s, a; P_h^{\star}) }   \big[V\big] &\le \sup_{\eta\in[0, H]} \bigg\{ - \mathbb{E}_{P_h^{\star}(\cdot|s, a)}\big[(\eta-V)_+\big] + \left(1 - \rho \right) \cdot \eta \bigg\} \\
    &= \EE_{\cP_\rho(s, a; P_h^\star)}\big[ V \big],
    \#
    where the equality follows from \eqref{eq:8881}.

    \paragraph{Combining Part (i) and Part (ii).}    Finally, combining \eqref{eq:8882} and \eqref{eq:8884}, we  prove Proposition~\ref{prop: equivalent robust set}.
\end{proof}

\section{Proofs for Hardness Results}

\subsection{Proof of Theorem~\ref{thm: hard example}}\label{subsec: proof thm hard example}

\begin{proof}[Proof of Theorem~\ref{thm: hard example}]
We first explicitly give the expressions of the robust value functions in Example~\ref{exp: hard}, based on which we derive the desired online regret lower bound.

\paragraph{Robust value function.}
Firstly, we can explicitly write down the expression of the robust value functions for any policy $\pi$ under Example~\ref{exp: hard}, i.e.,
$V_{h, P^{{\star},\cM_\theta}, \mathbf{\Phi}}^{\pi}$ and $Q_{h, P^{{\star},\cM_\theta}, \mathbf{\Phi}}^{\pi}$.
From now on we fix a policy $\pi$.

For step $h=3$, the robust value function is the reward received.
We can directly obtain for any $a\in\cA$,
\begin{align}\label{eq: h=3}
    Q_{3, P^{{\star},\cM_\theta}, \mathbf{\Phi}}^{\pi}(s_{\mathrm{good}}, a) = V_{3, P^{{\star},\cM_\theta}, \mathbf{\Phi}}^{\pi}(s_{\mathrm{good}}) = 1, \quad Q_{3, P^{{\star},\cM_\theta}, \mathbf{\Phi}}^{\pi}(s_{\mathrm{bad}}, a) = V_{3, P^{{\star},\cM_\theta}, \mathbf{\Phi}}^{\pi}(s_{\mathrm{bad}}) = 0.
\end{align}

For step $h=2$, by the robust Bellman equation (Proposition~\ref{prop: robust bellman equation}), we have that for the good state $s_{\mathrm{good}}$,
\begin{align}\label{eq: h=2 good}
     Q_{2, P^{{\star},\cM_\theta}, \mathbf{\Phi}}^{\pi}(s_{\mathrm{good}}, a) = 1 + \inf_{P\in \cP_{\rho}(s_{\mathrm{good}}, a; P_2^{\star,\cM_{\theta}})} \mathbb{E}_{P(\cdot)}\big[V_{3, P^{{\star},\cM_\theta}, \mathbf{\Phi}}^{\pi}\big] = 1+(1-\rho),\quad \forall a\in\cA,
\end{align}
where the last equality is because $V_{3, P^{{\star},\cM_\theta}, \mathbf{\Phi}}^{\pi}$ takes the minimal value $0$ at the bad state $s_{\mathrm{bad}}$ and thus the most adversarial transition distribution is achieved at
\begin{align}
    P^{\dagger}(s') = (1-\rho)\cdot\mathbf{1}\{s'=s_{\mathrm{good}}\} + \rho\cdot\mathbf{1}\{s'=s_{\mathrm{bad}}\}.
\end{align}
Similarly, we have that for the bad state $s_{\mathrm{bad}}$,
\begin{align}\label{eq: h=2 bad}
    Q_{2, P^{{\star},\cM_\theta}, \mathbf{\Phi}}^{\pi}(s_{\mathrm{bad}}, a) = 0 + \inf_{P\in \cP_{\rho}(s_{\mathrm{bad}}, a; P_2^{\star,\cM_{\theta}})} \mathbb{E}_{P(\cdot)}\big[V_{3, P^{{\star},\cM_\theta}, \mathbf{\Phi}}^{\pi}\big] = \left\{
    \begin{aligned}
        &p - \rho,\quad a = \theta\\
        &q - \rho,\quad a = 1-\theta
    \end{aligned}
    \right..
\end{align}
Finally by the robust Bellman equation again, we have that
\begin{align}
    V_{2, P^{{\star},\cM_\theta}, \mathbf{\Phi}}^{\pi}(s_{\mathrm{good}}) = 1+(1-\rho),\quad V_{2, P^{{\star},\cM_\theta}, \mathbf{\Phi}}^{\pi}(s_{\mathrm{bad}}) = \pi_2(\theta|s_{\mathrm{bad}})\cdot(p-\rho) + \pi_2(1-\theta|s_{\mathrm{bad}})\cdot(q-\rho).
\end{align}
Notice that by $q<p$ we know that $V_{2, P^{{\star},\cM_\theta}, \mathbf{\Phi}}^{\pi}(s_{\mathrm{bad}}) < p -\rho < 1+(1-\rho)< V_{2, P^{{\star},\cM_\theta}, \mathbf{\Phi}}^{\pi}(s_{\mathrm{good}})$.

For step $h=1$, we consider the robust values on the initial state $s_1 = s_{\mathrm{good}}$, by robust Bellman equation,
\begin{align}
    Q_{1, P^{{\star},\cM_\theta}, \mathbf{\Phi}}^{\pi}(s_{\mathrm{good}}, a) &= 1 + \inf_{P\in \cP_{\rho}(s_{\mathrm{good}}, a; P_1^{\star,\cM_{\theta}})} \mathbb{E}_{P(\cdot)}\big[V_{2, P^{{\star},\cM_\theta}, \mathbf{\Phi}}^{\pi}\big] \label{eq: h=1}\\
    &= 1 + (1-\rho)\cdot\big[1+(1-\rho)\big] + \rho\cdot\big[\pi_2(\theta|s_{\mathrm{bad}})\cdot(p-\rho) + \pi_2(1-\theta|s_{\mathrm{bad}})\cdot(q-\rho)\big],
\end{align}
for any action $a\in\cA$. By robust Bellman equation, we also derive $V_{1, P^{{\star},\cM_\theta}, \mathbf{\Phi}}^{\pi}(s_{\mathrm{good}}) = Q_{1, P^{{\star},\cM_\theta}, \mathbf{\Phi}}^{\pi}(s_{\mathrm{good}},a) $.

\paragraph{Lower bound the online regret under Example~\ref{exp: hard}.}

With all the previous preparation, we can lower bound the online regret for robust RL with interactive data collection in Example~\ref{exp: hard}.
But first, we present the following general lemma.

\begin{lemma}[Performance difference lemma for robust value function]\label{lem: performance difference}
    For any RMDP satisfying Assumption~\ref{ass: sa} and any policy $\pi$, the following inequality holds,
    \begin{align}
        V_{1,P^{\star},\boldsymbol{\Phi}}^{\pi^{\star}}(s) - V_{1,P^{\star},\boldsymbol{\Phi}}^{\pi}(s) \geq \mathbb{E}_{(P^{\pi^{\star},\dagger}, \pi^{\star})}\left[\sum_{h=1}^H\sum_{a\in\cA}\big(\pi^{\star}_h(a|s_h) - \pi_h(a|s_h)\big)\cdot Q_{h,P^{\star},\boldsymbol{\Phi}}^{\pi}(s_h,a)\middle| s_1=s \right],
    \end{align}
    where the expectation is taken with respect to the trajectories induced by policy $\pi^{\star}$, transition kernel $P^{\pi^{\star},\dagger}$.
    Here the transition kernel $P^{\pi^{\star},\dagger}$ is defined as
    \begin{align}
        P^{\pi^{\star},\dagger}_h(\cdot|s,a) = \arginf_{P\in \cP(s,a;P^{\star}_h)}\mathbb{E}_{P(\cdot)}\big[V_{h+1,P^{\star},\boldsymbol{\Phi}}^{\pi^{\star}}\big],
    \end{align}
    where $\cP(s,a;P^{\star}_h)$ is the robust set for state-action pair $(s,a)$ (see Assumption~\ref{ass: sa}).
\end{lemma}

\begin{proof}[Proof of Lemma~\ref{lem: performance difference}]
    Please refer to Appendix~\ref{subsec: proof lem performance difference} for a detailed proof of Lemma~\ref{lem: performance difference}.
\end{proof}

Now back to Example~\ref{exp: hard}, our previous calculation actually shows that, by \eqref{eq: h=3} for step $h=3$,
\begin{align}
    \sum_{a\in\cA}\big(\pi_3^{\star,\cM_{\theta}}(a|s_3) - \pi_3(a|s_3)\big)\cdot Q_{3,P^{\star,\cM_{\theta}},\boldsymbol{\Phi}}^{\pi}(s_3,a) = 0,\quad \forall s_3\in\{s_{\mathrm{good}}, s_{\mathrm{bad}}\}.\label{eq: h=3 hard}
\end{align}
and by \eqref{eq: h=1} we also have that for step $h=1$,
\begin{align}
    \sum_{a\in\cA}\big(\pi_1^{\star, \cM_{\theta}}(a|s_1) - \pi_1(a|s_1)\big)\cdot Q_{1,P^{\star, \cM_{\theta}},\boldsymbol{\Phi}}^{\pi}(s_1,a) = 0,\quad \text{where}\quad   s_1=s_{\mathrm{good}}.\label{eq: h=1 hard}
\end{align}
Finally, let's consider step $h=2$.
By \eqref{eq: h=2 good}, we have that for the good state, it holds that
\begin{align}
    \sum_{a\in\cA}\big(\pi_2^{\star, \cM_{\theta}}(a|s_{\mathrm{good}}) - \pi_2(a|s_{\mathrm{good}})\big)\cdot Q_{2,P^{\star, \cM_{\theta}},\boldsymbol{\Phi}}^{\pi}(s_{\mathrm{good}},a) = 0,\label{eq: h=2 good hard}
\end{align}
Meanwhile, by \eqref{eq: h=2 bad}, we have that for the bad state, it holds that (recall that $q<p$)
\begin{align}
     &\sum_{a\in\cA}\big(\pi_2^{\star, \cM_{\theta}}(a|s_{\mathrm{bad}}) - \pi_2(a|s_{\mathrm{bad}})\big)\cdot Q_{2,P^{\star, \cM_{\theta}},\boldsymbol{\Phi}}^{\pi}(s_{\mathrm{bad}},a) \\
     &\qquad= \max\big\{p-\rho, q-\rho\big\} - \Big(\pi_2(\theta|s_{\mathrm{bad}})\cdot(p-\rho) + \pi_2(1-\theta|s_{\mathrm{bad}})\cdot(q-\rho)\Big)\\
     &\qquad = p-\rho- \Big(\pi_2(\theta|s_{\mathrm{bad}})\cdot(p-\rho) + \pi_2(1-\theta|s_{\mathrm{bad}})\cdot(q-\rho)\Big) \\
     &\qquad = \frac{p-q}{2}\cdot\bigg(\left|\pi_2^{\star,\cM_{\theta}}(\theta|s_{\mathrm{bad}}) - \pi_2(\theta|s_{\mathrm{bad}})\right| + \left|\pi_2^{\star,\cM_{\theta}}(1-\theta|s_{\mathrm{bad}}) - \pi_2(1-\theta|s_{\mathrm{bad}})\right|\bigg)\\
     & \qquad = (p-q)\cdot D_{\mathrm{TV}}\left(\pi_2^{\star,\cM_{\theta}}(\cdot|s_{\mathrm{bad}}) \middle\|\pi_2(\cdot|s_{\mathrm{bad}})  \right),\label{eq: h=2 bad hard}
\end{align}
where according to \eqref{eq: h=2 bad} the optimal policy of $\cM_\theta$ at $h=2$ and $s_{\mathrm{bad}}$ is $\pi^{\star,\cM_{\theta}}_2(\theta|s_{\mathrm{bad}})=1$.
Now combining \eqref{eq: h=3 hard}, \eqref{eq: h=1 hard}, \eqref{eq: h=2 good hard}, and \eqref{eq: h=2 bad hard} with Lemma~\ref{lem: performance difference}, we can conclude that
\begin{align}
    &V_{1,P^{\star, \cM_{\theta}},\boldsymbol{\Phi}}^{\pi^{\star,\cM_{\theta}}}(s_{\mathrm{good}}) - V_{1,P^{\star, \cM_{\theta}},\boldsymbol{\Phi}}^{\pi}(s_{\mathrm{good}}) \\
    &\qquad\geq \mathbb{E}_{a_1\sim \pi^{\star,\cM_{\theta}}_1(\cdot|s_{\mathrm{good}}), s_2\sim P_1^{\pi^{\star,\cM_{\theta}},\dagger}(\cdot|s_{\mathrm{good}},a_1)}\left[\sum_{a\in\cA}\big(\pi_2^{\star}(a|s_2) - \pi_2(a|s_2)\big)\cdot Q_{2,P^{\star, \cM_{\theta}},\boldsymbol{\Phi}}^{\pi}(s_2,a)\right]\\
    &\qquad = P_1^{\pi^{\star,\cM_{\theta}},\dagger}(s_{\mathrm{bad}}|s_{\mathrm{good}},0)\cdot (p-q)\cdot D_{\mathrm{TV}}\left(\pi_2^{\star,\cM_{\theta}}(\cdot|s_{\mathrm{bad}}) \middle\|\pi_2(\cdot|s_{\mathrm{bad}})\right),\label{eq: lower bound}
\end{align}
where the adversarial transition kernel $P_1^{\pi^{\star,\cM_{\theta}},\dagger}$ is given by
\begin{align}
    P_1^{\pi^{\star,\cM_{\theta}},\dagger}(\cdot|s_{\mathrm{good}},0) &= \argmin_{P\in\cP(s_{\mathrm{good}},0;P_1^{\star,\cM_{\theta}})}\mathbb{E}_{P(\cdot)}\Big[V_{2,P^{\star,\cM_{\theta}},\boldsymbol{\Phi}}^{\pi^{\star,\cM_{\theta}}}\Big] \\
    &= (1-\rho)\cdot\mathbf{1}\{\cdot =s_{\mathrm{good}}\} + \rho\cdot\mathbf{1}\{\cdot =s_{\mathrm{bad}}\}.\label{eq: adversarial kernel}
\end{align}
Consequently, taking \eqref{eq: adversarial kernel} back into \eqref{eq: lower bound}, we have that
\begin{align}
    V_{1,P^{\star, \cM_{\theta}},\boldsymbol{\Phi}}^{\pi^{\star,\cM_{\theta}}}(s_{\mathrm{good}}) - V_{1,P^{\star, \cM_{\theta}},\boldsymbol{\Phi}}^{\pi}(s_{\mathrm{good}})\geq \rho\cdot (p-q)\cdot D_{\mathrm{TV}}\left(\pi_2^{\star,\cM_{\theta}}(\cdot|s_{\mathrm{bad}}) \middle\|\pi_2(\cdot|s_{\mathrm{bad}})\right).
\end{align}
This implies that for any algorithm executing $\pi^1,\cdots,\pi^K$, its online regret is lower bounded by the following,
\begin{align}
    \mathrm{Regret}^{\cM_{\theta},\mathcal{ALG}}_{\boldsymbol{\Phi}}(K) &=\sum_{k=1}^KV_{1,P^{\star, \cM_{\theta}},\boldsymbol{\Phi}}^{\pi^{\star,\cM_{\theta}}}(s_{\mathrm{good}}) - V_{1,P^{\star, \cM_{\theta}},\boldsymbol{\Phi}}^{\pi^k}(s_{\mathrm{good}}) \\
    &\geq \rho\cdot (p-q)\cdot \sum_{k=1}^KD_{\mathrm{TV}}\left(\pi_2^{\star,\cM_{\theta}}(\cdot|s_{\mathrm{bad}}) \middle\|\pi_2^k(\cdot|s_{\mathrm{bad}})\right).
\end{align}
However, since in RMDPs of Example~\ref{exp: hard}, the online interaction process is always kept in $s_{\mathrm{good}}$ and there is no information on $\theta$ which can only be accessed at $(s,h) = (s_{\mathrm{bad}},2)$.
As a result, the estimates $\pi_2^k(\cdot|s_{\mathrm{bad}})$ of $\pi^{\star,\cM_{\theta}}_2(\cdot|s_{\mathrm{bad}}) = \mathbf{1}\{\cdot=\theta\}$ can do no better than a random guess.
Put it formally, consider that
\begin{align}
    &\sup_{\theta\in\{0,1\}}\mathbb{E}_{\cM_{\theta},\mathcal{ALG}}\left[\mathrm{Regret}^{\cM_{\theta},\mathcal{ALG}}_{\boldsymbol{\Phi}}(K)\right] \\&\qquad \geq \rho\cdot (p-q)\cdot\sup_{\theta\in\{0,1\}} \mathbb{E}_{\cM_{\theta},\mathcal{ALG}}\left[\sum_{k=1}^KD_{\mathrm{TV}}\left(\pi_2^{\star,\cM_{\theta}}(\cdot|s_{\mathrm{bad}}) \middle\|\pi_2^k(\cdot|s_{\mathrm{bad}})\right)\right]\\
    &\qquad =\rho\cdot (p-q)\cdot\sup_{\theta\in\{0,1\}} \sum_{k=1}^K\mathbb{E}_{\mathcal{ALG}}\left[\pi_2^k(1-\theta|s_{\mathrm{bad}})\right].\label{eq: lower bound 2}
\end{align}
Here in the last equality we can drop the subscription of $\cM_{\theta}$ because the algorithm outputs $\pi^k_2$ independent of the $\theta$ due to our previous discussion.
Notice that
\begin{align}
    \sum_{\theta\in\{0,1\}}\sum_{k=1}^K\mathbb{E}_{\mathcal{ALG}}\left[\pi_2^k(1-\theta|s_{\mathrm{bad}})\right] =\sum_{k=1}^K\sum_{\theta\in\{0,1\}}\mathbb{E}_{\mathcal{ALG}}\left[\pi_2^k(1-\theta|s_{\mathrm{bad}})\right] = \sum_{k=1}^K 1 = K,
\end{align}
which further indicates that
\begin{align}
    \sup_{\theta\in\{0,1\}} \sum_{k=1}^K\mathbb{E}_{\mathcal{ALG}}\left[\pi_2^k(1-\theta|s_{\mathrm{bad}})\right]\geq \frac{K}{2}.\label{eq: sup lower bound}
\end{align}
Therefore, by combining \eqref{eq: lower bound 2} and \eqref{eq: sup lower bound}, we conclude that
\begin{align}
    \inf_{\mathcal{ALG}}\sup_{\theta\in\{0,1\}}\mathbb{E}_{\cM_{\theta},\mathcal{ALG}}\left[\mathrm{Regret}^{\cM_{\theta},\mathcal{ALG}}_{\boldsymbol{\Phi}}(K)\right] \geq  (p-q)\cdot\frac{\rho K}{2}.
\end{align}
This is the desired online regret lower bound of $\Omega(\rho \cdot K)$ for the RMDPs presented in Example~\ref{exp: hard}. Furthermore, we can construct two RMDPs $\{\widetilde{\mathcal{M}}_0, \widetilde{\mathcal{M}}_1\}$ with horizon $3H$ by concatenating $H$ RMDPs $\{\mathcal{M}_0, \mathcal{M}_1\}$ presented in Example~\ref{exp: hard}. Notably, at any steps $\{3i+1\}_{i=0}^{H-1}$, we define
\$
R_{3i+1}(s_{\mathrm{bad}}, a) = 1, \qquad P_{3i+1}^{\star, \widetilde{\cM}_{\theta}}(s_{\mathrm{good}}|s_{\mathrm{bad}}, a)=1,\quad \forall (a, \theta) \in\cA\times\{0,1\}.
\$
Then we have
\$
\inf_{\mathcal{ALG}}\sup_{\theta\in\{0,1\}}\mathbb{E}_{\widetilde{\cM}_{\theta},\mathcal{ALG}}\left[\mathrm{Regret}^{\widetilde{\cM}_{\theta},\mathcal{ALG}}_{\boldsymbol{\Phi}}(K)\right] \geq H \cdot \Omega(\rho \cdot K) = \Omega(\rho  \cdot HK),
\$
which completes the proof of Theorem~\ref{thm: hard example}.
\end{proof}

\subsection{Proof of Lemma~\ref{lem: performance difference}}\label{subsec: proof lem performance difference}

\begin{proof}[Proof of Lemma~\ref{lem: performance difference}]
    For any step $h\in[H]$, we have that by robust Bellman equation (Proposition~\ref{prop: robust bellman equation}),
    \begin{align}
        Q_{h,P^{\star},\boldsymbol{\Phi}}^{\pi^{\star}}(s, a) - Q_{h,P^{\star},\boldsymbol{\Phi}}^{\pi}(s, a) = \mathbb{E}_{\mathcal{P}_{\rho}(s,a;P_h^{\star})}\big[V_{h+1, P^{\star}, \mathbf{\Phi}}^{\pi^{\star}}\big] - \mathbb{E}_{\mathcal{P}_{\rho}(s,a;P_h^{\star})}\big[V_{h+1, P^{\star}, \mathbf{\Phi}}^{\pi}\big].
    \end{align}
    By the definition of the transition kernel $P^{\pi^{\star},\dagger}$ in Lemma~\ref{lem: performance difference} and the property of infimum, we have that
    \begin{align}
        Q_{h,P^{\star},\boldsymbol{\Phi}}^{\pi^{\star}}(s, a) - Q_{h,P^{\star},\boldsymbol{\Phi}}^{\pi}(s, a) &\geq \mathbb{E}_{P_h^{\pi^{\star},\dagger}(\cdot|s,a)}\big[V_{h+1, P^{\star}, \mathbf{\Phi}}^{\pi^{\star}}\big] - \mathbb{E}_{P_h^{\pi^{\star},\dagger}(\cdot|s,a)}\big[V_{h+1, P^{\star}, \mathbf{\Phi}}^{\pi}\big] \\
        & = \mathbb{E}_{P_h^{\pi^{\star},\dagger}(\cdot|s,a)}\big[V_{h+1, P^{\star}, \mathbf{\Phi}}^{\pi^{\star}} - V_{h+1, P^{\star}, \mathbf{\Phi}}^{\pi}\big].\label{eq: proof performance difference 1}
    \end{align}
    By robust Bellman equation (Proposition~\ref{prop: robust bellman equation}) and \eqref{eq: proof performance difference 1}, we further obtain that
    \begin{align}
        V_{h,P^{\star},\boldsymbol{\Phi}}^{\pi^{\star}}(s) - V_{h,P^{\star},\boldsymbol{\Phi}}^{\pi}(s) &= \mathbb{E}_{\pi_h^{\star}(\cdot|s)}\big[Q_{h, P^{\star}, \mathbf{\Phi}}^{\pi^{\star}}(s, \cdot)\big] - \mathbb{E}_{\pi_h(\cdot|s)}\big[Q_{h, P^{\star}, \mathbf{\Phi}}^{\pi}(s, \cdot)\big] \\
        & = \mathbb{E}_{\pi_h^{\star}(\cdot|s)}\big[Q_{h, P^{\star}, \mathbf{\Phi}}^{\pi}(s, \cdot)\big] - \mathbb{E}_{\pi_h(\cdot|s)}\big[Q_{h, P^{\star}, \mathbf{\Phi}}^{\pi}(s, \cdot)\big]  \\
        &\qquad + \mathbb{E}_{\pi_h^{\star}(\cdot|s)}\big[Q_{h, P^{\star}, \mathbf{\Phi}}^{\pi^{\star}}(s, \cdot)\big] - \mathbb{E}_{\pi_h^{\star}(\cdot|s)}\big[Q_{h, P^{\star}, \mathbf{\Phi}}^{\pi}(s, \cdot)\big] \\
        &\geq \sum_{a\in\cA}\big(\pi^{\star}_h(a|s) - \pi_h(a|s)\big)\cdot Q_{h,P^{\star},\boldsymbol{\Phi}}^{\pi}(s,a)  \\
        &\qquad + \mathbb{E}_{a\sim \pi_h^{\star}(\cdot|s), P_h^{\pi^{\star},\dagger}(\cdot|s,a)}\big[V_{h+1, P^{\star}, \mathbf{\Phi}}^{\pi^{\star}} - V_{h+1, P^{\star}, \mathbf{\Phi}}^{\pi}\big].\label{eq: proof performance difference 2}
    \end{align}
    Thus by recursively applying \eqref{eq: proof performance difference 2} over $h\in[H]$, we can conclude that
    \begin{align}
        V_{1,P^{\star},\boldsymbol{\Phi}}^{\pi^{\star}}(s) - V_{1,P^{\star},\boldsymbol{\Phi}}^{\pi}(s) \geq \mathbb{E}_{(P^{\pi^{\star},\dagger}, \pi^{\star})}\left[\sum_{h=1}^H\sum_{a\in\cA}\big(\pi^{\star}_h(a|s_h) - \pi_h(a|s_h)\big)\cdot Q_{h,P^{\star},\boldsymbol{\Phi}}^{\pi}(s_h,a)\middle| s_1=s \right],
    \end{align}
    which completes the proof of Lemma~\ref{lem: performance difference}.
\end{proof}

\section{Proofs for Theoretical Analysis of OPROVI-TV}\label{sec: proof tv}

In this section, we prove our main theoretical results (Theorem~\ref{thm: regret tv}).
In Appendix~\ref{subsec: proof them regret tv}, we outline the proof of the theorem.
In Appendix~\ref{subsec: key lemmas}, we list all the key lemmas used in the proof of the theorem.
We defer the proof of all the lemmas to subsequent sections (Appendices~\ref{subsec: proof optimism and pessimism restate} to \ref{subsec: other technical lemmas}).

Before presenting all the proofs, we define the typical event $\mathcal{E}$ as
\begin{align}
    \mathcal{E} &= \biggg\{\bigg|\left(\mathbb{E}_{P_h^{\star}(\cdot|s,a)} - \mathbb{E}_{\widehat{P}_h^k(\cdot|s,a)}\right)\Big[\big(\eta - V_{h+1,P^{\star},\boldsymbol{\Phi}}^{\star}\big)_+\Big]\bigg|\leq \sqrt{\frac{\mathbb{V}_{\widehat{P}_h^k(\cdot|s,a)}\Big[\big(\eta - V_{h+1,P^{\star},\boldsymbol{\Phi}}^{\star}\big)_+\Big]\cdot c_1\iota}{N_h^k(s,a)\vee 1}} + \frac{c_2H\iota}{N_h^k(s,a)\vee 1},\\
    &\qquad\left|P_h^{\star}(s'|s,a) - \widehat{P}_h^{k}(s'|s,a)\right| \leq \sqrt{\frac{\min\left\{P_h^{\star}(s'|s,a),\widehat{P}_h^k(s'|s,a)\right\}\cdot c_1\iota}{N_h^k(s,a)\vee 1}} + \frac{c_2\iota}{N_h^k(s,a)\vee 1}, \\
    &\qquad \forall (s,a,s',h,k)\in\cS\times\cA\times\cS\times[H]\times[K],\,\,\forall \eta \in\mathcal{N}_{1/(S\sqrt{K})}\big([0,H]\big)\biggg\}, \quad \iota = \log\big(S^3AH^2K^{3/2}/\delta\big),  \label{eq: typical event main}
\end{align}
where $c_1,c_2>0$ are two absolute constants, $\cN_{1/S\sqrt{K}}([0,H])$ denotes an $1/S\sqrt{K}$-cover of the interval $[0,H]$.

\begin{lemma}[Typical event]\label{lem: typical event}
    For the typical event $\mathcal{E}$ defined in \eqref{eq: typical event main}, it holds that $\mathbb{P}(\mathcal{E})\geq 1-\delta$.
\end{lemma}

\begin{proof}[Proof of Lemma~\ref{lem: typical event}]
    This is a direct application of Bernstein inequality and its empirical version \citep{maurer2009empirical}, together with a union bound over $(s,a,s',h,k, \eta)\in\cS\times\cA\times\cS\times[H]\times[K]\times\mathcal{N}_{1/(S\sqrt{K})}([0,H])$.
    Note that the size of $\mathcal{N}_{1/(S\sqrt{K})}([0,H])$ is of order $SH\sqrt{K}$.
\end{proof}

In this section, we always let the event $\cE$ hold, which by Lemma~\ref{lem: typical event} is of probability at least $1-\delta$.

\subsection{Proof of Theorem~\ref{thm: regret tv}}\label{subsec: proof them regret tv}

\begin{proof}[Proof of Theorem~\ref{thm: regret tv}]

With Lemma \ref{lem: optimism and pessimism restate} (optimism and pessimism), we can upper bound the regret as
\begin{align}\label{eq: analysis overview 1}
    \mathrm{Regret}_{\boldsymbol{\Phi}}(K) &= \sum_{k=1}^KV_{1,P^{\star},\mathbf{\Phi}}^{\star}(s_1)-V_{1,P^{\star},\mathbf{\Phi}}^{\pi^k}(s_1) \leq \sum_{k=1}^K \overline{V}_1^k(s_1) - \underline{V}_1^k(s_1).
\end{align}
In the sequel, we break our proof into three steps.

\paragraph*{Step 1: upper bounding \eqref{eq: analysis overview 1}.}
According to the choice of $\overline{Q}_h^k$, $\underline{Q}_h^k$, $\overline{V}_h^k$, $\underline{V}_h^k$ in \eqref{eq: Q overline}, \eqref{eq: Q underline}, and \eqref{eq: V}, let's consider that for any $(h,k)\in[H]\times[K]$ and $(s,a)\in\cS\times\cA$,
\allowdisplaybreaks
\begin{align}
    \overline{Q}_h^{k}(s,a) - \underline{Q}_h^{k}(s,a)
    & = \min\bigg\{R_h(s,a) + \mathbb{E}_{\mathcal{P}_{\rho}(s,a;\widehat{P}_h^k)}\Big[\overline{V}_{h+1}^k\Big] + \texttt{bonus}_h^k(s,a),\,\min\big\{H, \rho^{-1}\big\}\bigg\}\\
    &\qquad -
    \max\bigg\{R_h(s,a) +  \mathbb{E}_{\mathcal{P}_{\rho}(s,a;\widehat{P}_h^k)}\Big[\underline{V}_{h+1}^k\Big] - \texttt{bonus}_h^k(s,a),\, 0\bigg\}\\
    &\leq \mathbb{E}_{\mathcal{P}_{\rho}(s,a;\widehat{P}_h^k)}\Big[\overline{V}_{h+1}^k\Big]- \mathbb{E}_{\mathcal{P}_{\rho}(s,a;\widehat{P}_h^k)}\Big[\underline{V}_{h+1}^k\Big] + 2\cdot \texttt{bonus}_h^k(s,a)\\
    &= \underbrace{\mathbb{E}_{\mathcal{P}_{\rho}(s,a;\widehat{P}_h^k)}\Big[\overline{V}_{h+1}^k\Big]- \mathbb{E}_{\mathcal{P}_{\rho}(s,a;P_h^{\star})}\Big[\overline{V}_{h+1}^k\Big] + \mathbb{E}_{\mathcal{P}_{\rho}(s,a;P_h^{\star})}\Big[\underline{V}_{h+1}^k\Big]- \mathbb{E}_{\mathcal{P}_{\rho}(s,a;\widehat{P}_h^k)}\Big[\underline{V}_{h+1}^k\Big]}_{_{\displaystyle{\text{Term (i)}}}}\\
    & \qquad +  \underbrace{\mathbb{E}_{\mathcal{P}_{\rho}(s,a;P_h^{\star})}\Big[\overline{V}_{h+1}^{k}\Big]- \mathbb{E}_{\mathcal{P}_{\rho}(s,a;P_h^{\star})}\Big[\underline{V}_{h+1}^k\Big]}_{\displaystyle{\text{Term (ii)}}}\,\, +\, 2\cdot \texttt{bonus}_h^k(s,a).\label{eq: analysis overview 2}
\end{align}

\paragraph{Step 1.1: upper bounding \text{Term (i)}.}
By using a Bernstein-style concentration argument customized for TV robust expectations (Lemma~\ref{lem: bonus for optimistic and pessimistic value estimators restate}), we can bound Term (i) by the bonus function, i.e.,
\begin{align}
    \text{Term (i)}\leq 2\cdot \texttt{bonus}_h^k(s,a).\label{eq: analysis overview 3}
\end{align}

\paragraph*{Step 1.2: upper bounding \text{Term (ii)}.}
By our definition of the operator $\mathbb{E}_{\mathcal{P}_{\rho}(s,a;P_h^\star)}[V]$ in \eqref{eq: duality tv}, we have
\begin{align}
     \text{Term (ii)} &= \sup_{\eta\in[0,H]}\bigg\{- \mathbb{E}_{P_h^\star(\cdot|s,a)}\bigg[\Big(\eta-\overline{V}^{k}_{h+1}\Big)_+\bigg] + \left(1-\rho\right)\cdot\eta \bigg\} \\
     &\qquad  - \sup_{\eta\in[0,H]}\bigg\{- \mathbb{E}_{P_h^\star(\cdot|s,a)}\bigg[\Big(\eta-\underline{V}^{k}_{h+1}\Big)_+\bigg] + \left(1-\rho\right)\cdot\eta \bigg\} \\
     & \leq \sup_{\eta\in[0,H]}\bigg\{\mathbb{E}_{P_h^\star(\cdot|s,a)}\bigg[\Big(\eta-\underline{V}^{k}_{h+1}\Big)_+ - \Big(\eta-\overline{V}^{k}_{h+1}\Big)_+\bigg]\bigg\}.\label{eq: analysis overview 5}
\end{align}
By Lemma~\ref{lem: optimism and pessimism restate} which shows that $\overline{V}^{k}_{h+1} \geq \underline{V}^{k}_{h+1}$ and the fact that $(\eta - x)_+ - (\eta - y)_+ \leq y-x$ for any $y>x$,
we can further upper bound the right hand side of \eqref{eq: analysis overview 5} by
\begin{align}
    \text{Term (ii)} \leq\mathbb{E}_{P_h^\star(\cdot|s,a)}\Big[\overline{V}^{k}_{h+1} - \underline{V}^{k}_{h+1}\Big].\label{eq: analysis overview 6}
\end{align}

\paragraph*{Step 1.3: combining the upper bounds.}
Now combining \eqref{eq: analysis overview 3} and \eqref{eq: analysis overview 6} with \eqref{eq: analysis overview 2}, we have that
\begin{align}
    \overline{Q}_h^{k}(s,a) - \underline{Q}_h^{k}(s,a) &\leq \mathbb{E}_{P_h^\star(\cdot|s,a)}\Big[\overline{V}_{h+1}^{k} - \underline{V}_{h+1}^{k}\Big] + 4\cdot \texttt{bonus}_h^k(s,a).
\end{align}
By Lemma~\ref{lem: control of bonus restate}, we can  upper bound the bonus function, and after rearranging terms we further obtain that
\begin{align}
    \overline{Q}_h^{k}(s,a) - \underline{Q}_h^{k}(s,a)&\leq  \left(1+\frac{12}{H}\right)\cdot \mathbb{E}_{ P_h^{\star}(\cdot|s,a)}\Big[\overline{V}_{h+1}^{k} - \underline{V}_{h+1}^{k}\Big]  \\
        &\qquad  + 4\sqrt{\frac{\mathbb{V}_{P^{\star}_h(\cdot|s,a)}\Big[V^{\pi^k}_{h+1,P^{\star},\boldsymbol{\Phi}}\Big]\cdot c_1\iota}{N_h^k(s,a)\vee 1}} + \frac{4c_2H^2S\iota}{N_h^k(s,a)\vee 1} + 4c_3\sqrt{\frac{\iota}{N_h^k(s,a)\vee 1}} + \frac{4}{\sqrt{K}}, \label{eq: analysis overview 7}
\end{align}
where $c_1,c_2,c_3>0$ are absolute constants.
For the sake of brevity, we introduce the following notations of differences, for any $(h,k)\in[H]\times[K]$,
\begin{align}
    \Delta_h^k &:= \overline{V}_h^k(s_h^k) -  \underline{V}_h^k(s_h^k),\label{eq: m d 1}\\
    \zeta_h^k &:= \Delta_h^k - \Big(\overline{Q}_h^k(s_h^k, a_h^k) -  \underline{Q}_h^k(s_h^k, a_h^k)\Big), \label{eq: m d 2}\\
    \xi_h^k &:= \mathbb{E}_{P_h^{\star}(\cdot|s_h^k,a_h^k)}\Big[\overline{V}_{h+1}^k -  \underline{V}_{h+1}^k\Big] - \Delta_{h+1}^k.\label{eq: m d 3}
\end{align}
If we further define the filtration $\{\cF_{h,k}\}_{(h,k)\in[H]\times[K]}$ as
\begin{align}
    \cF_{h,k} = \sigma\left(\{(s_i^{\tau},a_i^{\tau})\}_{(i,\tau)\in[H]\times[k-1]}\bigcup \{(s_i^{k},a_i^{k})\}_{i\in[h-1]}\bigcup\{s_h^k\}\right),
\end{align}
then we can find that $\{\zeta_h^k\}_{(h,k)\in[H]\times[K]}$ is a martingale difference sequence with respect to $\{\cF_{h,k}\}_{(h,k)\in[H]\times[K]}$ and $\{\xi_h^k\}_{(h,k)\in[H]\times[K]}$ is a martingale difference sequence with respect to $\{\cF_{h,k}\cup\{a_h^k\}\}_{(h,k)\in[H]\times[K]}$.
Also, we further have that
\begin{align}
    \Delta_h^k &= \zeta_h^k + \Big(\overline{Q}_h^k(s_h^k, a_h^k) -  \underline{Q}_h^k(s_h^k, a_h^k)\Big) \label{eq: analysis overview 7+} \\
    &\leq \zeta_h^k + \left(1+\frac{12}{H}\right)\cdot \mathbb{E}_{ P_h^{\star}(\cdot|s_h^k,a_h^k)}\Big[\overline{V}_{h+1}^{k} - \underline{V}_{h+1}^{k}\Big] \notag\\
    &\qquad +4\sqrt{\frac{\mathbb{V}_{P^{\star}_h(\cdot|s_h^k,a_h^k)}\Big[V^{\pi^k}_{h+1,P^{\star},\boldsymbol{\Phi}}\Big]\cdot c_1\iota}{N_h^k(s_h^k, a_h^k)\vee 1}}
    + \frac{4c_2H^2S\iota}{N_h^k(s_h^k, a_h^k)\vee 1}
    + 4c_3\sqrt{\frac{\iota}{N_h^k(s_h^k, a_h^k)\vee 1}}
    + \frac{4}{\sqrt{K}} \\
    &=\zeta_h^k + \left(1+\frac{12}{H}\right)\cdot \xi_h^k  + \left(1+\frac{12}{H}\right)\cdot \Delta_{h+1}^k \notag\\
    &\qquad + 4\sqrt{\frac{\mathbb{V}_{P^{\star}_h(\cdot|s_h^k,a_h^k)}\Big[V^{\pi^k}_{h+1,P^{\star},\boldsymbol{\Phi}}\Big]\cdot c_1\iota}{N_h^k(s_h^k, a_h^k)\vee 1}}
    + \frac{4c_2H^2S\iota}{N_h^k(s_h^k, a_h^k)\vee 1}
    + 4c_3\sqrt{\frac{\iota}{N_h^k(s_h^k, a_h^k)\vee 1}}
    + \frac{4}{\sqrt{K}},
\end{align}
where the inequality applies \eqref{eq: analysis overview 7}.
Recursively applying \eqref{eq: analysis overview 7+} and using the fact that $(1+\frac{12}{H})^h\leq (1+\frac{12}{H})^H\leq c$ for some absolute constant $c>0$, we can upper bound the right hand side of \eqref{eq: analysis overview 1} as
\begin{align}
    \mathrm{Regret}_{\boldsymbol{\Phi}}(K)
    &\leq \sum_{k=1}^K\Delta_1^k \notag\\
    &\leq C_1\cdot\sum_{k=1}^K\sum_{h=1}^H \bigg(\zeta_h^k + \xi_h^k
    + \sqrt{\frac{\mathbb{V}_{P^{\star}_h(\cdot|s_h^k,a_h^k)}\Big[V^{\pi^k}_{h+1,P^{\star},\boldsymbol{\Phi}}\Big]\cdot \iota}{N_h^k(s_h^k, a_h^k)\vee 1}} \notag\\
    &\qquad\qquad\qquad\qquad
    + \frac{H^2S\iota}{N_h^k(s_h^k, a_h^k)\vee 1}
    + \sqrt{\frac{\iota}{N_h^k(s_h^k,a_h^k)\vee 1}}
    + \frac{1}{\sqrt{K}}\bigg).\label{eq: analysis overview 8}
\end{align}
where $C_1>0$ is an absolute constant.

\paragraph*{Step 2: controlling the summation of variance terms.}
In view of \eqref{eq: analysis overview 8}, it suffices to upper bound its right hand side.
The key difficulty is the analysis of the summation of the variance terms, which we focus on now.
By Cauchy-Schwartz inequality,
\allowbreak
\begin{align}
    &\sum_{k=1}^K\sum_{h=1}^H\sqrt{\frac{\mathbb{V}_{P^{\star}_h(\cdot|s_h^k,a_h^k)}\Big[V^{\pi^k}_{h+1,P^{\star},\boldsymbol{\Phi}}\Big]}{N_h^k(s_h^k,a_h^k)\vee 1}}  \leq \sqrt{\sum_{k=1}^K\sum_{h=1}^H\mathbb{V}_{P^{\star}_h(\cdot|s_h^k,a_h^k)}\Big[V^{\pi^k}_{h+1,P^{\star},\boldsymbol{\Phi}}\Big]\cdot \sum_{k=1}^K\sum_{h=1}^H\frac{1}{N_h^k(s_h^k,a_h^k)\vee 1}}.\label{eq: analysis overview 9}
\end{align}
On the right hand side of \eqref{eq: analysis overview 9}, the summation of the inverse of the count function is a well bounded term (Lemma~\ref{lem: lem7.5'}).
So the key is to upper bound the the summation of the variance of the robust value functions to obtain a sharp bound.
To this end, we invoke Lemma~\ref{lem: total variance restate} to obtain that with probability at least $1-\delta$,
\begin{align}
         \sum_{k=1}^K\sum_{h = 1}^H \mathbb{V}_{P_h^{\star}(\cdot|s_h^k,a_h^k)}\Big[V_{h+1,P^{\star},\boldsymbol{\Phi}}^{\pi^k}\Big]  \le C_2\cdot  \Big(\min\big\{H,\rho^{-1}\big\}\cdot HK +  \min\big\{H,\rho^{-1}\big\}^3\cdot H\iota\Big),\label{eq: total variance}
\end{align}
where $C_2>0$ is an absolute constant.
With inequality \eqref{eq: total variance} and Lemma \ref{lem: lem7.5'} that
\begin{align}
    \sum_{k=1}^K\sum_{h=1}^H\frac{1}{N_h^k(s_h^k, a_h^k)\vee 1} \leq C_2'\cdot HSA\iota,
\end{align}
with $C_2'>0$ being another constant, we can  upper bound the summation of the variance terms \eqref{eq: analysis overview 9} as
\begin{align}
    &\sum_{k=1}^K\sum_{h=1}^H\sqrt{\frac{\mathbb{V}_{P^{\star}_h(\cdot|s_h^k,a_h^k)}\Big[V^{\pi^k}_{h+1,P^{\star},\boldsymbol{\Phi}}\Big]}{N_h^k(s_h^k,a_h^k)\vee 1}}  \leq C_3  \sqrt{\min\big\{H,\rho^{-1}\big\}\cdot H^2SAK\iota +  \min\big\{H,\rho^{-1}\big\}^3\cdot H^2SA\iota^2}.\quad\quad\label{eq: analysis overview 10}
\end{align}
where $C_3>0$ is also an absolute constant.

\paragraph{Step 3: finishing the proof.}
With \eqref{eq: analysis overview 8} and \eqref{eq: analysis overview 10}, it suffices to control the remaining terms.
For the summation of the martingale difference terms, notice that by the definitions in \eqref{eq: m d 2} and \eqref{eq: m d 3}, both $\zeta_h^k$ and $\xi_h^k$ are bounded by $\min\{H,\rho^{-1}\}$ according to \eqref{eq: Q overline} and Lemma~\ref{lem: optimism and pessimism restate} (optimism and pessimism).
As a result, using Azuma-Hoeffding inequality, with probability at least $1-\delta$
\begin{align}
    \sum_{k=1}^K\sum_{h=1}^H(\zeta_h^k + \xi_h^k) \leq C_4\cdot \min\big\{H,\rho^{-1}\big\}\cdot \sqrt{HK\iota},
\end{align}
where $C_4>0$ is an absolute constant.
For the summation of the inverse of the count function in \eqref{eq: analysis overview 8}, it suffices to invoke again Lemma~\ref{lem: lem7.5'}.
{
For the additional square-root count term in \eqref{eq: analysis overview 8}, Cauchy--Schwarz and Lemma~\ref{lem: lem7.5'} give
\begin{align}
    \sum_{k=1}^K\sum_{h=1}^H
    \sqrt{\frac{\iota}{N_h^k(s_h^k,a_h^k)\vee 1}}
    \leq
    \sqrt{HK\iota\cdot\sum_{k=1}^K\sum_{h=1}^H
    \frac{1}{N_h^k(s_h^k,a_h^k)\vee 1}}
    \leq C_4'H\sqrt{SAK\iota^2},
\end{align}
for an absolute constant $C_4'>0$, which is absorbed by the leading term after adjusting logarithmic factors.
}
Combining all together, with probability at least $1-3\delta$, we have
\begin{align}
    \mathrm{Regret}_{\boldsymbol{\Phi}}(K) &\leq  C_5\cdot \bigg( \sqrt{\min\big\{H,\rho^{-1}\big\}\cdot H^2SAK\iota^2 +  \min\big\{H,\rho^{-1}\big\}^3\cdot H^2SA\iota^3}\\
    &\qquad  + \min\big\{H, \rho^{-1}\big\}\cdot \sqrt{HK\iota} + H^3S^2A\iota^2 + H\sqrt{SAK\iota^2} + H\sqrt{K}\bigg) \\
    & = \mathcal{O}\left(\sqrt{\min\big\{H,\rho^{-1}\big\}\cdot  H^2SAK\iota'} \right),
\end{align}
where $C_5>0$ is an absolute constant and $\iota' = \log^2(SAHK/\delta)$.
This completes the proof of Theorem~\ref{thm: regret tv}.
\end{proof}

\subsection{Key Lemmas}\label{subsec: key lemmas}

\begin{lemma}[Optimistic and pessimistic estimation of the robust values]\label{lem: optimism and pessimism restate}
    By setting the $\mathtt{bonus}_h^k$ as in \eqref{eq: bernstein bonus}, then under the typical event $\cE$, it holds that
    \begin{align}\label{eq: optimism and pessimism}
        \!\!\!\!\!\!\!\!\!\!\underline{Q}_h^k(s,a) \leq Q^{\pi^k}_{h,P^{\star},\boldsymbol{\Phi}}(s,a) \leq  Q^{\star}_{h,P^{\star},\boldsymbol{\Phi}}(s,a)\leq \overline{Q}^{k}_h(s,a), \quad \underline{V}_h^k(s) \leq V^{\pi^k}_{h,P^{\star},\boldsymbol{\Phi}}(s) \leq V^{\star}_{h,P^{\star},\boldsymbol{\Phi}}(s) \leq \overline{V}^{k}_h(s),
    \end{align}
    for any $(s,a,h,k)\in\mathcal{S}\times\mathcal{A}\times[H]\times[K]$.
\end{lemma}

\begin{proof}[Proof of Lemma~\ref{lem: optimism and pessimism restate}]
    See Appendix~\ref{subsec: proof optimism and pessimism restate} for a detailed proof.
\end{proof}

\begin{lemma}[Proper bonus for TV robust sets and optimistic and pessimistic value estimators]\label{lem: bonus for optimistic and pessimistic value estimators restate}
    By setting the $\mathtt{bonus}_h^k$ as in \eqref{eq: bernstein bonus}, then under the typical event $\cE$, it holds that
    \begin{align}
        \mathbb{E}_{\mathcal{P}_{\rho}(s,a;\widehat{P}_h^k)}\Big[\overline{V}_{h+1}^k\Big]- \mathbb{E}_{\mathcal{P}_{\rho}(s,a;P_h^{\star})}\Big[\overline{V}_{h+1}^k\Big] + \mathbb{E}_{\mathcal{P}_{\rho}(s,a;P_h^{\star})}\Big[\underline{V}_{h+1}^k\Big]- \mathbb{E}_{\mathcal{P}_{\rho}(s,a;\widehat{P}_h^k)}\Big[\underline{V}_{h+1}^k\Big] \leq 2\cdot\mathtt{bonus}_h^k(s,a),
    \end{align}
\end{lemma}

\begin{proof}[Proof of Lemma~\ref{lem: bonus for optimistic and pessimistic value estimators restate}]
    See Appendix~\ref{subsec: proof bonus for optimistic and pessimistic value estimators restate} for a detailed proof.
\end{proof}

\begin{lemma}[Control of the bonus term]\label{lem: control of bonus restate}
    Under the typical event $\cE$, the $\mathtt{bonus}_h^k$ in \eqref{eq: bernstein bonus} is bounded by
    \begin{align}
        \mathtt{bonus}_h^k(s,a)
        &\leq \sqrt{\frac{\mathbb{V}_{P^{\star}_h(\cdot|s,a)}\Big[V^{\pi^k}_{h+1,P^{\star},\boldsymbol{\Phi}}\Big]\cdot c_1\iota}{N_h^k(s,a)\vee 1}}  + \frac{4\cdot \mathbb{E}_{P^{\star}_h(\cdot|s,a)}\Big[\overline{V}_{h+1}^{k} \!-\! \underline{V}_{h+1}^{k}\Big]}{H} \nonumber\\
        &\qquad + \frac{c_2H^2S\iota}{N_h^k(s,a)\vee 1} + c_3\sqrt{\frac{\iota}{N_h^k(s,a)\vee 1}} + \frac{1}{\sqrt{K}},
    \end{align}
    where $\iota = \log(S^3AH^2K^{3/2}/\delta)$ and $c_1, c_2,c_3>0$ are absolute constants.
\end{lemma}

\begin{proof}[Proof of Lemma~\ref{lem: control of bonus restate}]
    See Appendix~\ref{subsec: proof control of bonus restate} for a detailed proof.
\end{proof}

\begin{lemma}[Total variance law for robust MDP with TV robust sets]\label{lem: total variance restate}
    With probability at least $1-\delta$, the following inequality holds
    \begin{align}
        \sum_{k=1}^K\sum_{h = 1}^H \mathbb{V}_{P_h^{\star}(\cdot|s_h^k,a_h^k)}\Big[V_{h+1,P^{\star},\boldsymbol{\Phi}}^{\pi^k}\Big] \leq  c_3\cdot  \Big(\min\{H,\rho^{-1}\}\cdot HK +  \min\{H,\rho^{-1}\}^3\cdot H\iota\Big).
    \end{align}
    where $\iota = \log(S^3AH^2K^{3/2}/\delta)$ and $c_3>0$ is an absolute constant.
\end{lemma}

\begin{proof}[Proof of Lemma~\ref{lem: total variance restate}]
    See Appendix~\ref{subsec: proof total variance restate} for a detailed proof.
\end{proof}

\subsection{Proof of Lemma~\ref{lem: optimism and pessimism restate}}\label{subsec: proof optimism and pessimism restate}

\begin{proof}[Proof of Lemma~\ref{lem: optimism and pessimism restate}]
    We prove Lemma \ref{lem: optimism and pessimism restate} by induction.
    Suppose the conclusion \eqref{eq: optimism and pessimism} holds at step $h+1$.
    For step $h$, let's first consider the robust $Q$ function part.
    Specifically, by using the robust Bellman optimal equation (Proposition~\ref{prop: robust bellman optimal equation}) and \eqref{eq: Q overline}, we have that
    \begin{align}
        &Q^{\star}_{h,P^{\star},\boldsymbol{\Phi}}(s,a) - \overline{Q}^{k}_h(s,a) \\
        &\quad \leq \max\bigg\{\mathbb{E}_{\mathcal{P}_{\rho}(s,a;P_h^{\star})}\Big[V_{h+1,P^{\star},\boldsymbol{\Phi}}^{\star}\Big] - \mathbb{E}_{\mathcal{P}_{\rho}(s,a;\widehat{P}_h^k)}\Big[\overline{V}_{h+1}^{k}\Big] - \texttt{bonus}_h^k(s,a), \,\,Q^{\star}_{h,P^{\star},\boldsymbol{\Phi}}(s,a)  - \min\big\{H, \rho^{-1}\big\} \bigg\}\\
        & \quad \leq \max\bigg\{\mathbb{E}_{\mathcal{P}_{\rho}(s,a;P_h^{\star})}\Big[V_{h+1,P^{\star},\boldsymbol{\Phi}}^{\star}\Big] - \mathbb{E}_{\mathcal{P}_{\rho}(s,a;\widehat{P}_h^k)}\Big[V_{h+1,P^{\star},\boldsymbol{\Phi}}^{\star}\Big] - \texttt{bonus}_h^k(s,a), \,\,0\bigg\},\label{eq: proof lem optimism and pessimism 1}
    \end{align}
    where the second inequality follows from the induction of $V_{h+1,P^{\star},\boldsymbol{\Phi}}^{\star}\leq \overline{V}_{h+1}^k$ at step $h+1$ and the fact that $Q^{\star}_{h,P^{\star},\boldsymbol{\Phi}} \leq \min\{H, \rho^{-1}\}$ (by Proposition~\ref{prop: gap} and Assumption~\ref{ass: zero min}).
    By Lemma~\ref{lem: bernstein bonus for optimal value}, we have that
    \begin{align}
        \mathbb{E}_{\mathcal{P}_{\rho}(s,a;P_h^{\star})}\Big[V^{\star}_{h+1,P^{\star},\boldsymbol{\Phi}}\Big] - \mathbb{E}_{\mathcal{P}_{\rho}(s,a;\widehat{P}_h^k)}\Big[V^{\star}_{h+1,P^{\star},\boldsymbol{\Phi}}\Big]\leq  \sqrt{\frac{\mathbb{V}_{\widehat{P}_h^k(\cdot|s,a)}\Big[V_{h+1,P^{\star},\boldsymbol{\Phi}}^{\star}\Big]\cdot c_1\iota}{N_h^k(s,a)\vee 1}} + \frac{c_2H\iota}{N_h^k(s,a)\vee 1} + \frac{1}{\sqrt{K}},
    \end{align}
    Now by further applying Lemma~\ref{lem: variance analysis 1} to the variance term in the above inequality, we can obtain that
    \allowdisplaybreaks
    \begin{align}
        &\mathbb{E}_{\mathcal{P}_{\rho}(s,a;P_h^{\star})}\Big[V^{\star}_{h+1,P^{\star},\boldsymbol{\Phi}}\Big] - \mathbb{E}_{\mathcal{P}_{\rho}(s,a;\widehat{P}_h^k)}\Big[V^{\star}_{h+1,P^{\star},\boldsymbol{\Phi}}\Big]\\
        &\quad \leq \sqrt{\frac{\left(\mathbb{V}_{\widehat{P}_h^k(\cdot|s,a)}\Big[\Big(\overline{V}_{h+1}^k+\underline{V}_{h+1}^k\Big) / 2\Big] + 4 H \cdot \mathbb{E}_{\widehat{P}_h^k(\cdot|s,a)}\Big[\overline{V}_{h+1}^k-\underline{V}_{h+1}^k\Big]\right)\cdot c_1\iota}{N_h^k(s,a)\vee 1}} + \frac{ c_2H\iota}{N_h^k(s,a)\vee 1} + \frac{1}{\sqrt{K}} \\
        &\quad \leq \sqrt{\frac{\mathbb{V}_{\widehat{P}_h^k(\cdot|s,a)}\Big[\Big(\overline{V}_{h+1}^k+\underline{V}_{h+1}^k\Big) / 2\Big]\cdot c_1\iota}{N_h^k(s,a)\vee 1}} + \sqrt{\frac{ \mathbb{E}_{\widehat{P}_h^k(\cdot|s,a)}\Big[\overline{V}_{h+1}^k-\underline{V}_{h+1}^k\Big]\cdot 4Hc_1\iota}{N_h^k(s,a)\vee 1}} + \frac{ c_2H\iota}{N_h^k(s,a)\vee 1} + \frac{1}{\sqrt{K}}\\
        &\quad \leq \sqrt{\frac{\mathbb{V}_{\widehat{P}_h^k(\cdot|s,a)}\Big[\Big(\overline{V}_{h+1}^k+\underline{V}_{h+1}^k\Big) / 2\Big]\cdot c_1\iota}{N_h^k(s,a)\vee 1}} + \frac{\mathbb{E}_{\widehat{P}_h^k(\cdot|s,a)}\Big[\overline{V}_{h+1}^k-\underline{V}_{h+1}^k\Big]}{H} + \frac{ c_2'H^2\iota}{N_h^k(s,a)\vee 1} + \frac{1}{\sqrt{K}},\label{eq: proof lem optimism and pessimism 2}
    \end{align}
    where the first inequality is due to Lemma~\ref{lem: variance analysis 1}, the second inequality is due to $\sqrt{a + b}\leq \sqrt{a} + \sqrt{b}$, and the last inequality is from $\sqrt{ab}\leq a+b$ where $c_2'>0$ is an absolute constant.
    Therefore, combining \eqref{eq: proof lem optimism and pessimism 1} and \eqref{eq: proof lem optimism and pessimism 2}, and the choice of $\texttt{bonus}_h^k(s,a)$ in \eqref{eq: bernstein bonus}, we can conclude that
    \begin{align}
        Q^{\star}_{h,P^{\star},\boldsymbol{\Phi}}(s,a) \leq  \overline{Q}^{k}_h(s,a).
    \end{align}
    Furthermore, it holds that $Q^{\pi^k}_{h,P^{\star},\boldsymbol{\Phi}}(s,a) \leq  Q^{\star}_{h,P^{\star},\boldsymbol{\Phi}}(s,a)$.
    Thus it reduces to prove $\underline{Q}_h^k(s,a)\leq Q^{\pi^k}_{h,P^{\star},\boldsymbol{\Phi}}(s,a)$.
    Again, by using the robust Bellman equation (Proposition~\ref{prop: robust bellman equation}) and \eqref{eq: Q underline}, we have that
    \begin{align}
        &\underline{Q}^{k}_h(s,a) - Q^{\pi^k}_{h,P^{\star},\boldsymbol{\Phi}}(s,a)\\
        &\qquad \leq \max\bigg\{\mathbb{E}_{\mathcal{P}_{\rho}(s,a;\widehat{P}_h^k)}\Big[\underline{V}_{h+1}^{k}\Big] - \mathbb{E}_{\mathcal{P}_{\rho}(s,a;P_h^{\star})}\Big[V_{h+1,P^{\star},\boldsymbol{\Phi}}^{\pi^k}\Big] - \texttt{bonus}_h^k(s,a),\,\, 0  \bigg\}\\
        & \qquad\leq \max\bigg\{\mathbb{E}_{\mathcal{P}_{\rho}(s,a;\widehat{P}_h^k)}\Big[V_{h+1,P^{\star},\boldsymbol{\Phi}}^{\pi^k}\Big] - \mathbb{E}_{\mathcal{P}_{\rho}(s,a;P_h^{\star})}\Big[V_{h+1,P^{\star},\boldsymbol{\Phi}}^{\pi^k}\Big]  - \texttt{bonus}_h^k(s,a),\,\,0\bigg\},\label{eq: proof lem optimism and pessimism 3}
    \end{align}
    where the second inequality follows from the induction of $\underline{V}_{h+1}^k\leq V_{h+1,P^{\star},\boldsymbol{\Phi}}^{\pi^k}$ at step $h+1$ and the fact that $Q^{\pi^k}_{h,P^{\star},\boldsymbol{\Phi}}\geq 0$.
    By Lemma~\ref{lem: bernstein bound for value of pi k}, we have that
    \begin{align}
        &\mathbb{E}_{\mathcal{P}_{\rho}(s,a;\widehat{P}_h^k)}\Big[V^{\pi^k}_{h+1, P^{\star},\boldsymbol{\Phi}}\Big] - \mathbb{E}_{\mathcal{P}_{\rho}(s,a;P_h^{\star})}\Big[V_{h+1, P^{\star}, \boldsymbol{\Phi}}^{\pi^k}\Big]  \\
        &\qquad \leq  \sqrt{\frac{\mathbb{V}_{\widehat{P}_h^k(\cdot|s,a)}\Big[V_{h+1,P^{\star},\boldsymbol{\Phi}}^{\star}\Big]\cdot c_1\iota}{N_h^k(s,a)\vee 1}} + \frac{\mathbb{E}_{\widehat{P}_h^k(\cdot|s,a)}\Big[\overline{V}_{h+1}^{k} - \underline{V}_{h+1}^{k}\Big]}{H} + \frac{c_2'H^2S\iota}{N_h^k(s,a)\vee 1} + \frac{1}{\sqrt{K}}.
    \end{align}
    Now by applying Lemma~\ref{lem: variance analysis 1} to the variance term, with an argument similar to \eqref{eq: proof lem optimism and pessimism 2}, we can obtain that
    \begin{align}
        &\mathbb{E}_{\mathcal{P}_{\rho}(s,a;\widehat{P}_h^k)}\Big[V^{\pi^k}_{h+1, P^{\star},\boldsymbol{\Phi}}\Big] - \mathbb{E}_{\mathcal{P}_{\rho}(s,a;P_h^{\star})}\Big[V_{h+1, P^{\star}, \boldsymbol{\Phi}}^{\pi^k}\Big]  \label{eq: proof lem optimism and pessimism 4} \\
        &\qquad \leq \sqrt{\frac{\mathbb{V}_{\widehat{P}_h^k(\cdot|s,a)}\Big[\Big(\overline{V}_{h+1}^k+\underline{V}_{h+1}^k\Big) / 2\Big]\cdot c_1\iota}{N_h^k(s,a)\vee 1}} + \frac{2\mathbb{E}_{\widehat{P}_h^k(\cdot|s,a)}\Big[\overline{V}_{h+1}^k-\underline{V}_{h+1}^k\Big]}{H} + \frac{ c_2''H^2\iota}{N_h^k(s,a)\vee 1} + \frac{1}{\sqrt{K}},
    \end{align}
    Thus by combining \eqref{eq: proof lem optimism and pessimism 3} and \eqref{eq: proof lem optimism and pessimism 4}, and the choice of $\texttt{bonus}_h^k(s,a)$ in \eqref{eq: bernstein bonus}, we can conclude that
    \begin{align}
        \underline{Q}_h^k(s,a)\leq Q^{\pi^k}_{h,P^{\star},\boldsymbol{\Phi}}(s,a).
    \end{align}
    Therefore, we have proved that at step $h$, it holds that
    \begin{align}
        \underline{Q}_h^k(s,a) \leq Q^{\pi^k}_{h,P^{\star},\boldsymbol{\Phi}}(s,a) \leq  Q^{\star}_{h,P^{\star},\boldsymbol{\Phi}}(s,a)\leq \overline{Q}^{k}_h(s,a).
    \end{align}
    Finally for the robust $V$ function part, consider that by robust Bellman equation (Proposition~\ref{prop: robust bellman equation}) and \eqref{eq: V},
    \begin{align}
        \underline{V}_h^k(s) = \mathbb{E}_{\pi_h^k(\cdot|s)} \Big[ \underline{Q}_h^k(s,\cdot)\Big] \leq \mathbb{E}_{\pi_h^k(\cdot|s)} \Big[ Q^{\pi^k}_{h,P^{\star},\boldsymbol{\Phi}}(s,\cdot)\Big] = V^{\pi^k}_{h,P^{\star},\boldsymbol{\Phi}}(s),
    \end{align}
    and that by robust Bellman optimal equation (Proposition~\ref{prop: robust bellman optimal equation}), the choice of $\pi^k$, and \eqref{eq: V},
    \begin{align}
        V^{\star}_{h,P^{\star},\boldsymbol{\Phi}}(s) = \max_{a\in\cA}Q^{\star}_{h,P^{\star},\boldsymbol{\Phi}}(s,a)\leq \max_{a\in\cA}\overline{Q}_{h}^k(s,a) = \overline{V}_{h}^k(s),
    \end{align}
    which proves that
    \begin{align}
        \underline{V}_h^k(s) \leq V^{\pi^k}_{h,P^{\star},\boldsymbol{\Phi}}(s) \leq V^{\star}_{h,P^{\star},\boldsymbol{\Phi}}(s) \leq \overline{V}^{k}_h(s).
    \end{align}
    Since the conclusion \eqref{eq: optimism and pessimism} holds for the $V$ function part at step $H+1$, an induction proves Lemma~\ref{lem: optimism and pessimism restate}.
\end{proof}

\subsection{Proof of Lemma~\ref{lem: bonus for optimistic and pessimistic value estimators restate}}\label{subsec: proof bonus for optimistic and pessimistic value estimators restate}

\begin{proof}[Proof of Lemma~\ref{lem: bonus for optimistic and pessimistic value estimators restate}]
    We upper bound the required signed sum by applying Lemma~\ref{lem: bernstein bound for opt and pess value} to the two absolute differences,
    \begin{align}
        &\mathbb{E}_{\mathcal{P}_{\rho}(s,a;\widehat{P}_h^k)}\Big[\overline{V}_{h+1}^k\Big]- \mathbb{E}_{\mathcal{P}_{\rho}(s,a;P_h^{\star})}\Big[\overline{V}_{h+1}^k\Big] + \mathbb{E}_{\mathcal{P}_{\rho}(s,a;P_h^{\star})}\Big[\underline{V}_{h+1}^k\Big]- \mathbb{E}_{\mathcal{P}_{\rho}(s,a;\widehat{P}_h^k)}\Big[\underline{V}_{h+1}^k\Big]\\
        &\qquad \leq 2\sqrt{\frac{\mathbb{V}_{\widehat{P}_h^k(\cdot|s,a)}\Big[V_{h+1,P^{\star},\boldsymbol{\Phi}}^{\star}\Big]\cdot c_1\iota}{N_h^k(s,a)\vee 1}} + \frac{2\cdot \mathbb{E}_{\widehat{P}_h^k(\cdot|s,a)}\Big[\overline{V}_{h+1}^{k} - \underline{V}_{h+1}^{k}\Big]}{H} + \frac{2c_2'H^2S\iota}{N_h^k(s,a)\vee 1} + \frac{2}{\sqrt{K}},\label{eq: proof lem bonus for optimistic and pessimistic value estimators restate 1}
    \end{align}
    where $c_1,c_2'>0$ are absolute constants.
    Then applying Lemma~\ref{lem: variance analysis 1} to the variance term in \eqref{eq: proof lem bonus for optimistic and pessimistic value estimators restate 1}, with an argument the same as \eqref{eq: proof lem optimism and pessimism 2} in the proof of Lemma~\ref{lem: optimism and pessimism restate}, we can obtain that
    \begin{align}
        &\mathbb{E}_{\mathcal{P}_{\rho}(s,a;\widehat{P}_h^k)}\Big[\overline{V}_{h+1}^k\Big]- \mathbb{E}_{\mathcal{P}_{\rho}(s,a;P_h^{\star})}\Big[\overline{V}_{h+1}^k\Big] + \mathbb{E}_{\mathcal{P}_{\rho}(s,a;P_h^{\star})}\Big[\underline{V}_{h+1}^k\Big]- \mathbb{E}_{\mathcal{P}_{\rho}(s,a;\widehat{P}_h^k)}\Big[\underline{V}_{h+1}^k\Big]\\
        &\qquad\leq 2\sqrt{\frac{\mathbb{V}_{\widehat{P}_h^k(\cdot|s,a)}\Big[\Big(\overline{V}_{h+1}^k+\underline{V}_{h+1}^k\Big) / 2\Big]\cdot c_1\iota}{N_h^k(s,a)\vee 1}} + \frac{4\cdot \mathbb{E}_{\widehat{P}_h^k(\cdot|s,a)}\Big[\overline{V}_{h+1}^k-\underline{V}_{h+1}^k\Big]}{H} + \frac{2c_2''H^2\iota}{N_h^k(s,a)\vee 1} + \frac{2}{\sqrt{K}}.
    \end{align}
    Therefore, by looking into the choice of $\texttt{bonus}_h^k(s,a)$ in \eqref{eq: bernstein bonus}, we can conclude that
    \begin{align}
        \mathbb{E}_{\mathcal{P}_{\rho}(s,a;\widehat{P}_h^k)}\Big[\overline{V}_{h+1}^k\Big]- \mathbb{E}_{\mathcal{P}_{\rho}(s,a;P_h^{\star})}\Big[\overline{V}_{h+1}^k\Big] + \mathbb{E}_{\mathcal{P}_{\rho}(s,a;P_h^{\star})}\Big[\underline{V}_{h+1}^k\Big]- \mathbb{E}_{\mathcal{P}_{\rho}(s,a;\widehat{P}_h^k)}\Big[\underline{V}_{h+1}^k\Big] \leq 2\cdot\mathtt{bonus}_h^k(s,a),
    \end{align}
    This finishes the proof of Lemma~\ref{lem: bonus for optimistic and pessimistic value estimators restate}.
\end{proof}

\subsection{Proof of Lemma~\ref{lem: control of bonus restate}}\label{subsec: proof control of bonus restate}

\begin{proof}[Proof of Lemma~\ref{lem: control of bonus restate}]
    Recall that the $\mathtt{bonus}_h^k(s,a)$ is defined as
    \begin{align}
        \texttt{bonus}_h^k(s,a) &= \sqrt{\frac{\mathbb{V}_{\widehat{P}_h^k(\cdot|s,a)}\Big[\Big(\overline{V}_{h+1}^k+\underline{V}_{h+1}^k\Big) / 2\Big]\cdot c_1\iota}{N_h^k(s,a)\vee 1}} + \frac{2\mathbb{E}_{\widehat{P}_h^k(\cdot|s,a)}\Big[\overline{V}_{h+1}^{k} - \underline{V}_{h+1}^{k}\Big]}{H}  + \frac{c_2H^2S\iota}{N_h^k(s,a)\vee 1} + \frac{1}{\sqrt{K}}.
    \end{align}
    The main thing we need to consider is to control the first term and the second term.
    We first deal with the second term of $\mathtt{bonus}_h^k(s,a)$ by invoking Lemma~\ref{lem: non-robust bound 1}, which gives
    \begin{align}
        \frac{2\mathbb{E}_{\widehat{P}_h^k(\cdot|s,a)}\Big[\overline{V}_{h+1}^{k} - \underline{V}_{h+1}^{k}\Big]}{H}& \leq \left(\frac{2}{H}+\frac{2}{H^2}\right)\cdot\mathbb{E}_{P_h^{\star}(\cdot|s,a)}\Big[\overline{V}_{h+1}^{k} - \underline{V}_{h+1}^{k}\Big] + \frac{c_2'HS\iota}{N_h^k(s,a)\vee 1}  \\
        & \leq \frac{3\mathbb{E}_{P_h^{\star}(\cdot|s,a)}\Big[\overline{V}_{h+1}^{k} - \underline{V}_{h+1}^{k}\Big]}{H} + \frac{c_2'HS\iota}{N_h^k(s,a)\vee 1},  \label{eq: proof lem control of bonus restate 1}
    \end{align}
    where the second inequality is from $H\geq 2$.
    Then we deal with the first term (variance term) of $\mathtt{bonus}_h^k(s,a)$ by invoking Lemma~\ref{lem: variance analysis 2}, which gives
    \begin{align}
        &\sqrt{\frac{\mathbb{V}_{\widehat{P}_h^k(\cdot|s,a)}\Big[\Big(\overline{V}_{h+1}^k+\underline{V}_{h+1}^k\Big) / 2\Big]\cdot c_1\iota}{N_h^k(s,a)\vee 1}} \label{eq: proof lem control of bonus restate 2}\\
        & \leq \sqrt{\frac{\left(\mathbb{V}_{P_h^{\star}(\cdot|s,a)}\Big[V_{h+1,P^{\star},\boldsymbol{\Phi}}^{\pi^k}\Big] + 4 H \cdot \mathbb{E}_{P_h^{\star}(\cdot|s,a)}\Big[\overline{V}_{h+1}^k-\underline{V}_{h+1}^k\Big]+\frac{c_2''H^4 S \iota}{N_h^k(s, a)\vee 1}+1\right)\cdot c_1\iota}{N_h^k(s,a)\vee 1}} \\
        & \leq \sqrt{\frac{\mathbb{V}_{P_h^{\star}(\cdot|s,a)}\Big[V_{h+1,P^{\star},\boldsymbol{\Phi}}^{\pi^k}\Big]\cdot c_1\iota}{N_h^k(s,a)\vee 1}} + \sqrt{\frac{4 H \cdot \mathbb{E}_{P_h^{\star}(\cdot|s,a)}\Big[\overline{V}_{h+1}^k-\underline{V}_{h+1}^k\Big]\cdot c_1\iota}{N_h^k(s,a)\vee 1}} + \frac{\sqrt{c_1c_2''S}H^2\iota}{N_h^k(s,a)\vee 1} + \sqrt{\frac{c_1\iota}{N_h^k(s,a)\vee 1}} \\
        & \leq \sqrt{\frac{\mathbb{V}_{P_h^{\star}(\cdot|s,a)}\Big[V_{h+1,P^{\star},\boldsymbol{\Phi}}^{\pi^k}\Big]\cdot c_1'\iota}{N_h^k(s,a)\vee 1}} + \frac{\mathbb{E}_{P_h^{\star}(\cdot|s,a)}\Big[\overline{V}_{h+1}^k-\underline{V}_{h+1}^k\Big]}{H} + \frac{\big(4c_1 + \sqrt{c_1c_2''S}\big)H^2\iota}{N_h^k(s,a)\vee 1} + c_3\sqrt{\frac{\iota}{N_h^k(s,a)\vee 1}}
    \end{align}
    Thus by combining \eqref{eq: proof lem control of bonus restate 1} and \eqref{eq: proof lem control of bonus restate 2} with the choice of $\mathtt{bonus}_h^k$, we can conclude the proof of Lemma~\ref{lem: control of bonus restate}.
\end{proof}

\subsection{Proof of Lemma~\ref{lem: total variance restate}}\label{subsec: proof total variance restate}

\begin{proof}[Proof of Lemma~\ref{lem: total variance restate}]
    The key idea is to relate the visitation distribution (w.r.t. $P^{\star}$) and the variance (w.r.t. $P^{\star}$) to the value function of $\pi^k$, after which we can derive an upper bound for the total variance.
    Throughout this proof, we use the shorthand
    \begin{align}
        \overline{H} = \min \big\{H, \rho^{-1}\big\}.
    \end{align}
    Under the convention stated after Definition~\ref{def: tv}, when $\rho=0$ we have $\overline H=H$.
    According to Proposition~\ref{prop: gap} and Assumption~\ref{ass: zero min}, for any policy $\pi$ and any step $h$, the robust value function of $\pi$ holds that
        \begin{align}\label{eq: proof total variance max - min}
        \max_{s \in \cS} V_{h, P^\star, \boldsymbol{\Phi}}^{\pi}(s)  \le \overline{H},
    \end{align}
    which we usually apply in the sequel.

    Now consider the following decomposition of our target,
    \allowdisplaybreaks
    \begin{align}
        &\sum_{k=1}^K\sum_{h = 1}^H \mathbb{V}_{P_h^{\star}(\cdot|s_h^k,a_h^k)}\Big[V_{h+1,P^{\star},\boldsymbol{\Phi}}^{\pi^k}\Big]  \\
        &\qquad = \underbrace{\sum_{k=1}^K\bigg\{\sum_{h = 1}^H \mathbb{V}_{P_h^{\star}(\cdot|s_h^k,a_h^k)}\Big[V_{h+1,P^{\star},\boldsymbol{\Phi}}^{\pi^k}\Big] - \mathbb{E}_{(s_h^k,a_h^k)\sim (P^{\star},\pi^k)}\left[\sum_{h = 1}^H \mathbb{V}_{P_h^{\star}(\cdot|s_h^k,a_h^k)}\Big[V_{h+1,P^{\star},\boldsymbol{\Phi}}^{\pi^k}\Big]\middle| \cG_{k-1} \right]\bigg\}}_{\displaystyle{\text{\textcolor{blue}{Term (i): martingale difference term}}}} \\
        &\qquad\qquad  +\underbrace{\sum_{k=1}^K\mathbb{E}_{(s_h^k,a_h^k)\sim (P^{\star},\pi^k)}\left[\sum_{h = 1}^H \mathbb{V}_{P_h^\star(\cdot|s_h^k,a_h^k)}\Big[V_{h+1,P^{\star},\boldsymbol{\Phi}}^{\pi^k}\Big]  \middle| \cG_{k-1} \right]}_{\displaystyle{\text{\textcolor{blue}{Term (ii): total variance law under $P^\star$}}}}.
    \end{align}
    where we denote the filtration $\cG_k = \sigma(\{(s_h^\tau,a_h^\tau,s_{h+1}^\tau)\}_{(h,\tau)\in[H]\times[k]})$ with $\cG_0$ understood as the trivial sigma-field.
    In the sequel, we upper bound each of the two terms respectively.

    \paragraph{Term (i): martingale difference term.}
    This is a summation of martingale difference term (with respect to filtration $\cG_k = \sigma(\{(s_h^\tau,a_h^\tau,s_{h+1}^\tau)\}_{(h,\tau)\in[H]\times[k]})$).
    By Azuma-Hoeffding's inequality, with probability at least $1-\delta$,
    \begin{align}
        \text{Term (i)}\leq c\cdot H\cdot \overline{H}^2\cdot\sqrt{K\iota},\label{eq: total variance term 1}
    \end{align}
    where $c>0$ is an absolute constant.
    We have utilized the fact of \eqref{eq: proof total variance max - min} to obtain the upper bound $H\overline{H}^2$ on each martingale difference term in the summation.

    \paragraph{Term (ii): total variance law under $P^\star$.}
    The upper bound of this term is the core part of the analysis, for which we summarize it in the following lemma.

\begin{lemma}[Total variance law under $P^\star$]\label{lem:total:variance:tv}
Under the same settings as Theorem~\ref{thm: regret tv}, given any deterministic policy $\pi$,
define
\begin{align}
\widetilde T_h(\cdot| s,a)\in\argmin_{P(\cdot)\in\mathcal P_\rho(s,a;P_h^\star)}\mathbb E_{P(\cdot)}\left[V_{h+1,P^\star,\boldsymbol{\Phi}}^{\pi}\right],    \quad\forall (s,a,h)\in\mathcal S\times\mathcal A\times[H],
    \label{eq:robust-bellman-tv-minimizer}
\end{align}
and set $\widetilde T=\{\widetilde T_h\}_{h=1}^H$. Then we have that
\begin{align}
    \mathbb E_{(s_h,a_h)\sim(P^\star,\pi)}\left[\sum_{h=1}^H\mathbb V_{P_h^\star(\cdot\mid s_h,a_h)}\left[V_{h+1,P^\star,\boldsymbol{\Phi}}^{\pi}\right]\right]\le 2H\cdot \overline H,
\label{eq:robust-bellman-tv-law}
\end{align}
Consequently, it holds that
\begin{align}
    \mathbb{E}_{(s_h^k,a_h^k)\sim (P^{\star},\pi^k)}\left[\sum_{h = 1}^H \mathbb{V}_{P_h^\star(\cdot|s_h^k,a_h^k)}\Big[V_{h+1,P^{\star},\boldsymbol{\Phi}}^{\pi^k}\Big]  \middle| \cG_{k-1} \right] \le    2H\cdot \overline H.    \label{eq:cond-exp-Zk-robust-bellman-tv-law}
\end{align}
\end{lemma}

    We defer the proof of Lemma~\ref{lem:total:variance:tv} to Appendix~\ref{subsec: proof lem total variance tv}.
    With Lemma~\ref{lem:total:variance:tv}, conditional on $\cG_{k-1}$, the policy $\pi^k$ is fixed and deterministic by \eqref{eq: V}; taking $\pi = \pi^k$ for $k\in[K]$ therein, we obtain that the Term (ii) is upper bounded by
    \begin{align}
        \text{Term (ii)}\leq 2H\cdot\overline{H}\cdot K.\label{eq: total variance term 2}
    \end{align}

    \paragraph{Finishing the proof.} Finally, combining the upper bounds for Terms (i) and (ii) i.e., \eqref{eq: total variance term 1} and \eqref{eq: total variance term 2}, we conclude that with probability at least $1-\delta$, it holds that
    \begin{align}
        \sum_{k=1}^K\sum_{h = 1}^H \mathbb{V}_{P_h^{\star}(\cdot|s_h^k,a_h^k)}\Big[V_{h+1,P^{\star},\boldsymbol{\Phi}}^{\pi^k}\Big]  &\leq c\cdot H\cdot \overline{H}^2\cdot\sqrt{K\iota} + 2 H\cdot\overline{H}\cdot K \\
        &\leq c'\cdot  H\cdot\overline{H}\cdot K + c'' \cdot H\cdot \overline{H}^3\cdot \iota,
    \end{align}
    where in the last inequality we use $\sqrt{ab}\leq a+b$ for any $a,b>0$.
    Plug in the notation that $\overline{H} = \min\{H,\rho^{-1}\}$
    and finish the proof of Lemma~\ref{lem: total variance restate}.
\end{proof}

\subsection{Proof of Lemma~\ref{lem:total:variance:tv}}\label{subsec: proof lem total variance tv}

\begin{proof}[Proof of Lemma~\ref{lem:total:variance:tv}]
For notational simplicity, given policy $\pi$, we denote
\begin{align}
    V_h(\cdot)
    :=
    V_{h,P^\star,\boldsymbol{\Phi}}^{\pi}(\cdot),
    \quad h\in[H+1].
\end{align}
According to Proposition~\ref{prop: gap} and Assumption~\ref{ass: zero min}, for any
step $h$, it holds that
\begin{align}
    0\le V_h(s)\le \overline H,
    \quad \forall (s,h)\in\mathcal S\times[H+1].
    \label{eq:value-bound-robust-bellman-tv-law}
\end{align}
Using the property of variance, we have that for any $s_h\in\mathcal S$ and
$a_h=\pi_h(s_h)$,
\begin{align}
    \mathbb V_{P_h^\star(\cdot\mid s_h,a_h)}
    \left[
        V_{h+1}
    \right]
    &=
    \mathbb E_{P_h^\star(\cdot\mid s_h,a_h)}
    \left[
        V_{h+1}^2
    \right]
    -
    \left(
        \mathbb E_{P_h^\star(\cdot\mid s_h,a_h)}
        \left[
            V_{h+1}
        \right]
    \right)^2 .
    \label{eq:variance-decomp-pstar}
\end{align}
By the robust Bellman equation and the definition of $\widetilde T_h$ in
\eqref{eq:robust-bellman-tv-minimizer}, we have that
\begin{align}
    V_h(s_h)
    =
    R_h(s_h,a_h)
    +
    \mathbb E_{\widetilde T_h(\cdot\mid s_h,a_h)}
    \left[
        V_{h+1}
    \right].
    \label{eq:robust-bellman-with-Ttilde}
\end{align}
We further define
\begin{align}
    \Delta_h(s,a)
    :=
    \mathbb E_{P_h^\star(\cdot\mid s,a)}
    \left[
        V_{h+1}
    \right]
    -
    \mathbb E_{\widetilde T_h(\cdot\mid s,a)}
    \left[
        V_{h+1}
    \right].
    \label{eq:def-delta-robust-bellman-tv-law}
\end{align}
Since $P_h^\star(\cdot\mid s,a)$ belongs to robust set
$\mathcal P_\rho(s,a;P_h^\star)$ and $\widetilde T_h$ minimizes the
expectation of $V_{h+1}$, we have
\begin{align}
    \Delta_h(s,a)\ge0,
    \quad
    \forall (s,a,h)\in\mathcal S\times\mathcal A\times[H].
    \label{eq:delta-nonnegative}
\end{align}
Combining \eqref{eq:robust-bellman-with-Ttilde} and
\eqref{eq:def-delta-robust-bellman-tv-law}, we obtain that
\begin{align}
    \mathbb E_{P_h^\star(\cdot\mid s_h,a_h)}
    \left[
        V_{h+1}
    \right]
    =
    V_h(s_h)-R_h(s_h,a_h)+\Delta_h(s_h,a_h).
    \label{eq:pstar-expectation-via-bellman}
\end{align}
Thus, by \eqref{eq:variance-decomp-pstar} and
\eqref{eq:pstar-expectation-via-bellman}, we have
\begin{align}
    \mathbb V_{P_h^\star(\cdot\mid s_h,a_h)}
    \left[
        V_{h+1}
    \right]
    &=
    \mathbb E_{P_h^\star(\cdot\mid s_h,a_h)}
    \left[
        V_{h+1}^2
    \right]
    -
    \left(
        V_h(s_h)-R_h(s_h,a_h)+\Delta_h(s_h,a_h)
    \right)^2
    \nonumber\\
    &=
    \mathbb E_{P_h^\star(\cdot\mid s_h,a_h)}
    \left[
        V_{h+1}^2
    \right]
    -
    \left(V_h(s_h)\right)^2
    +
    2V_h(s_h)
    \left(
        R_h(s_h,a_h)-\Delta_h(s_h,a_h)
    \right)
    \nonumber\\
    &\qquad
    -
    \left(
        R_h(s_h,a_h)-\Delta_h(s_h,a_h)
    \right)^2
    \nonumber\\
    &\le
    \mathbb E_{P_h^\star(\cdot\mid s_h,a_h)}
    \left[
        V_{h+1}^2
    \right]
    -
    \left(V_h(s_h)\right)^2
    +
    2\overline H .
    \label{eq:one-step-var-bound-pstar}
\end{align}
Here the last inequality follows from
\eqref{eq:value-bound-robust-bellman-tv-law},
$0\le R_h(s_h,a_h)\le1$, and $\Delta_h(s_h,a_h)\ge0$. Specifically, let
$x=R_h(s_h,a_h)-\Delta_h(s_h,a_h)$. If $x\le0$, then
$2V_h(s_h)x-x^2\le0$; otherwise $0<x\le1$, and
$2V_h(s_h)x-x^2\le2\overline H$.
Taking expectation with respect to the trajectory generated by
$(P^\star,\pi)$, we have for each $h\in[H]$ that
\begin{align}
    &\mathbb E_{(s_h,a_h)\sim(P^\star,\pi)}
    \left[
        \mathbb V_{P_h^\star(\cdot\mid s_h,a_h)}
        \left[
            V_{h+1}
        \right]
    \right]
    \le
    \mathbb E_{s_{h+1}\sim(P^\star,\pi)}
    \left[
        \left(V_{h+1}(s_{h+1})\right)^2
    \right]
    -
    \mathbb E_{s_h\sim(P^\star,\pi)}
    \left[
        \left(V_h(s_h)\right)^2
    \right]
    +
    2\overline H.
    \label{eq:one-step-var-bound-pstar-expectation}
\end{align}
Taking summation over $h\in[H]$ gives that
\begin{align}
    &\mathbb E_{(s_h,a_h)\sim(P^\star,\pi)}
    \left[
        \sum_{h=1}^H
        \mathbb V_{P_h^\star(\cdot\mid s_h,a_h)}
        \left[
            V_{h+1}
        \right]
    \right]
    \nonumber\\
    &\qquad\le
    \sum_{h=1}^H
    \left\{
        \mathbb E_{s_{h+1}\sim(P^\star,\pi)}
        \left[
            \left(V_{h+1}(s_{h+1})\right)^2
        \right]
        -
        \mathbb E_{s_h\sim(P^\star,\pi)}
        \left[
            \left(V_h(s_h)\right)^2
        \right]
    \right\}
    +
    2H\cdot \overline H
    \nonumber\\
    &\qquad=
    \mathbb E_{s_{H+1}\sim(P^\star,\pi)}
    \left[
        \left(V_{H+1}(s_{H+1})\right)^2
    \right]
    -
    \mathbb E_{s_1}
    \left[
        \left(V_1(s_1)\right)^2
    \right]
    +
    2H\cdot \overline H
    \nonumber\\
    &\qquad\le
    2H\cdot \overline H,
    \label{eq:telescoping-pstar-total-variance}
\end{align}
where the equality follows from telescoping and the last inequality uses
$V_{H+1}\equiv0$ and the nonnegativity of $(V_1(s_1))^2$.
This concludes the proof of
Lemma~\ref{lem:total:variance:tv}.
\end{proof}

\subsection{Other Technical Lemmas}\label{subsec: other technical lemmas}

Before presenting all lemmas, we recall that the typical event $\mathcal{E}$ is defined as
\begin{align}
    \mathcal{E} &= \biggg\{\left|\left(\mathbb{E}_{P_h^{\star}(\cdot|s,a)} - \mathbb{E}_{\widehat{P}_h^k(\cdot|s,a)}\right)\Big[\big(\eta - V_{h+1,P^{\star},\boldsymbol{\Phi}}^{\star}\big)_+\Big]\right|\leq \sqrt{\frac{\mathbb{V}_{\widehat{P}_h^k(\cdot|s,a)}\Big[\big(\eta - V_{h+1,P^{\star},\boldsymbol{\Phi}}^{\star}\big)_+\Big]\cdot c_1\iota}{N_h^k(s,a)\vee 1}} + \frac{c_2H\iota}{N_h^k(s,a)\vee 1},\\
    &\qquad\left|P_h^{\star}(s'|s,a) - \widehat{P}_h^{k}(s'|s,a)\right| \leq \sqrt{\frac{\min\left\{P_h^{\star}(s'|s,a),\widehat{P}_h^k(s'|s,a)\right\}\cdot c_1\iota}{N_h^k(s,a)\vee 1}} + \frac{c_2\iota}{N_h^k(s,a)\vee 1}, \\
    &\qquad \forall (s,a,s',h,k)\in\cS\times\cA\times\cS\times[H]\times[K],\,\,\forall \eta \in\mathcal{N}_{1/(S\sqrt{K})}\big([0,H]\big)\biggg\}, \quad \iota = \log\big(S^3AH^2K^{3/2}/\delta\big).  \label{eq: typical event}
\end{align}
where $c_1,c_2>0$ are two absolute constants, $\cN_{1/S\sqrt{K}}([0,H])$ denotes an $1/S\sqrt{K}$-cover of the interval $[0,H]$.

\subsubsection{Concentration Inequalities}\label{subsubsec: concentration tv}

\begin{lemma}[Bernstein bound for TV robust sets and the optimal robust value function]\label{lem: bernstein bonus for optimal value}
    Under event $\cE$ in \eqref{eq: typical event}, it holds that
    \begin{align*}
        \bigg|\mathbb{E}_{\mathcal{P}_{\rho}(s,a;\widehat{P}_h^k)}\Big[V^{\star}_{h+1,P^{\star},\boldsymbol{\Phi}}\Big] - \mathbb{E}_{\mathcal{P}_{\rho}(s,a;P_h^{\star})}\Big[V^{\star}_{h+1,P^{\star},\boldsymbol{\Phi}}\Big] \bigg|\leq  \sqrt{\frac{\mathbb{V}_{\widehat{P}_h^k(\cdot|s,a)}\Big[V_{h+1,P^{\star},\boldsymbol{\Phi}}^{\star}\Big]\cdot c_1\iota}{N_h^k(s,a)\vee 1}} + \frac{c_2H\iota}{N_h^k(s,a)\vee 1} + \frac{1}{\sqrt{K}},
    \end{align*}
    where $\iota  = \log(S^3AH^2K^{3/2}/\delta)$.
\end{lemma}

\begin{proof}[Proof of Lemma \ref{lem: bernstein bonus for optimal value}]
    By our definition of the operator $\mathbb{E}_{\mathcal{P}_{\rho}(s,a;\widehat{P}_h^k)}[V^{\star}_{h+1,P^{\star},\boldsymbol{\Phi}}]$ in \eqref{eq: duality tv}, we can arrive that
    \allowdisplaybreaks
    \begin{align}
        &\bigg|\mathbb{E}_{\mathcal{P}_{\rho}(s,a;\widehat{P}_h^k)}\Big[V^{\star}_{h+1,P^{\star},\boldsymbol{\Phi}}\Big] - \mathbb{E}_{\mathcal{P}_{\rho}(s,a;P_h^{\star})}\Big[V^{\star}_{h+1,P^{\star},\boldsymbol{\Phi}}\Big]\bigg|\\
        &\qquad = \left|\sup_{\eta\in[0,H]} \bigg\{ - \mathbb{E}_{\widehat{P}_h^k(\cdot|s,a)}\Big[\big(\eta-V^{\star}_{h+1,P^{\star},\boldsymbol{\Phi}}\big)_+\Big] + \left(1-\rho\right)\cdot\eta \bigg\} \right.\\
        &\qquad\qquad \left.-\sup_{\eta\in[0,H]} \bigg\{ - \mathbb{E}_{P^{\star}_h(\cdot|s,a)}\Big[\big(\eta-V^{\star}_{h+1,P^{\star},\boldsymbol{\Phi}}\big)_+\Big] + \left(1-\rho\right)\cdot\eta \bigg\}\right|\\
        &\qquad \leq \sup_{\eta\in[0,H]} \Bigg\{ \bigg|\left(\mathbb{E}_{\widehat{P}_h^k(\cdot|s,a)} - \mathbb{E}_{P_h^{\star}(\cdot|s,a)}\right)\Big[\big(\eta-V^{\star}_{h+1, P^{\star},\boldsymbol{\Phi}}\big)_+\Big]\bigg| \Bigg\},\label{eq: proof bernstein bonus for optimal value 1}
    \end{align}
    Now according to the first inequality of event $\cE$, we have that
    \begin{align}
        \bigg|\left(\mathbb{E}_{P_h^{\star}(\cdot|s,a)} - \mathbb{E}_{\widehat{P}_h^k(\cdot|s,a)}\right)\Big[\big(\eta - V_{h+1,P^{\star},\boldsymbol{\Phi}}^{\star}\big)_+\Big]\bigg| & \leq \sqrt{\frac{\mathbb{V}_{\widehat{P}_h^k(\cdot|s,a)}\Big[\big(\eta - V_{h+1,P^{\star},\boldsymbol{\Phi}}^{\star}\big)_+\Big]\cdot c_1\iota}{N_h^k(s,a)\vee 1}} + \frac{c_2H\iota}{N_h^k(s,a)\vee 1} \\
        &\leq \sqrt{\frac{\mathbb{V}_{\widehat{P}_h^k(\cdot|s,a)}\Big[V_{h+1,P^{\star},\boldsymbol{\Phi}}^{\star}\Big]\cdot c_1\iota}{N_h^k(s,a)\vee 1}} + \frac{c_2H\iota}{N_h^k(s,a)\vee 1},
    \end{align}
    for any $\eta\in\mathcal{N}_{1/(S\sqrt{K})}([0,H])$.
    Here the second inequality is because $\mathrm{Var}[(a-X)_+]\leq \mathrm{Var}[X]$.
    Therefore, by a covering argument, for any $\eta\in[0,H]$, it holds that
    \begin{align}
        \bigg|\left(\mathbb{E}_{P_h^{\star}(\cdot|s,a)} - \mathbb{E}_{\widehat{P}_h^k(\cdot|s,a)}\right)\Big[\big(\eta - V_{h+1,P^{\star},\boldsymbol{\Phi}}^{\star}\big)_+\Big]\bigg| & \leq  \sqrt{\frac{\mathbb{V}_{\widehat{P}_h^k(\cdot|s,a)}\Big[V_{h+1,P^{\star},\boldsymbol{\Phi}}^{\star}\Big]\cdot c_1\iota}{N_h^k(s,a)\vee 1}} + \frac{c_2H\iota}{N_h^k(s,a)\vee 1} + \frac{1}{\sqrt{K}}.
    \end{align}
    This finishes the proof of Lemma \ref{lem: bernstein bonus for optimal value}.
\end{proof}

\begin{lemma}[Bernstein bound for TV robust sets and the robust value function of $\pi^k$]\label{lem: bernstein bound for value of pi k}
    Under event $\cE$ in \eqref{eq: typical event}, suppose that the optimism and pessimism \eqref{eq: optimism and pessimism} holds at $(h+1, k)$, then it holds that
    \begin{align}
        &\bigg|\mathbb{E}_{\mathcal{P}_{\rho}(s,a;\widehat{P}_h^k)}\Big[V^{\pi^k}_{h+1, P^{\star},\boldsymbol{\Phi}}\Big] - \mathbb{E}_{\mathcal{P}_{\rho}(s,a;P_h^{\star})}\Big[V_{h+1, P^{\star}, \boldsymbol{\Phi}}^{\pi^k}\Big] \bigg| \\
        &\qquad \leq  \sqrt{\frac{\mathbb{V}_{\widehat{P}_h^k(\cdot|s,a)}\Big[V_{h+1,P^{\star},\boldsymbol{\Phi}}^{\star}\Big]\cdot c_1\iota}{N_h^k(s,a)\vee 1}} + \frac{\mathbb{E}_{\widehat{P}_h^k(\cdot|s,a)}\Big[\overline{V}_{h+1}^{k} - \underline{V}_{h+1}^{k}\Big]}{H} + \frac{c_2'H^2S\iota}{N_h^k(s,a)\vee 1} + \frac{1}{\sqrt{K}},
    \end{align}
    where $\iota  = \log(S^3AH^2K^{3/2}/\delta)$ and $c_1$, $c_2'$ are absolute constants.
\end{lemma}

\begin{proof}[Proof of Lemma~\ref{lem: bernstein bound for value of pi k}]
    By our definition of the operator $\mathbb{E}_{\mathcal{P}_{\rho}(s,a;P)}[V^{\pi^k}_{h+1,P^{\star},\boldsymbol{\Phi}}]$ in \eqref{eq: duality tv}, we can arrive that,
    \allowdisplaybreaks
    \begin{align}
        &\bigg|\mathbb{E}_{\mathcal{P}_{\rho}(s,a;\widehat{P}_h^k)}\Big[V^{\pi^k}_{h+1,P^{\star},\boldsymbol{\Phi}}\Big] - \mathbb{E}_{\mathcal{P}_{\rho}(s,a;P_h^{\star})}\Big[V^{\pi^k}_{h+1,P^{\star},\boldsymbol{\Phi}}\Big]\bigg|\\
        &\qquad = \left|\sup_{\eta\in[0,H]} \bigg\{ - \mathbb{E}_{\widehat{P}_h^k(\cdot|s,a)}\Big[\big(\eta-V^{\pi^k}_{h+1,P^{\star},\boldsymbol{\Phi}}\big)_+\Big]  + \left(1-\rho\right)\cdot\eta \bigg\} \right.\\
        &\qquad\qquad \left.-\sup_{\eta\in[0,H]} \bigg\{ - \mathbb{E}_{P^{\star}_h(\cdot|s,a)}\Big[\big(\eta-V^{\pi^k}_{h+1,P^{\star},\boldsymbol{\Phi}}\big)_+\Big] + \left(1-\rho\right)\cdot\eta \bigg\}\right|\\
        &\qquad \leq \sup_{\eta\in[0,H]} \Bigg\{ \bigg|\left(\mathbb{E}_{\widehat{P}_h^k(\cdot|s,a)} - \mathbb{E}_{P_h^{\star}(\cdot|s,a)}\right)\Big[\big(\eta-V^{\pi^k}_{h+1, P^{\star},\boldsymbol{\Phi}}\big)_+\Big]\bigg| \Bigg\} \\
        &\qquad \leq \underbrace{\sup_{\eta\in[0,H]} \Bigg\{ \bigg|\left(\mathbb{E}_{\widehat{P}_h^k(\cdot|s,a)} - \mathbb{E}_{P_h^{\star}(\cdot|s,a)}\right)\Big[\big(\eta-V^{\star}_{h+1, P^{\star},\boldsymbol{\Phi}}\big)_+\Big]\bigg| \Bigg\}}_{\displaystyle{\text{Term (i)}}} \\
        &\qquad\qquad + \underbrace{\sup_{\eta\in[0,H]} \Bigg\{ \bigg|\left(\mathbb{E}_{\widehat{P}_h^k(\cdot|s,a)} - \mathbb{E}_{P_h^{\star}(\cdot|s,a)}\right)\Big[\big(\eta-V^{\pi^k}_{h+1, P^{\star},\boldsymbol{\Phi}}\big)_+ - \big(\eta-V^{\star}_{h+1, P^{\star},\boldsymbol{\Phi}}\big)_+\Big]\bigg| \Bigg\}}_{\displaystyle{\text{Term (ii)}}},
    \end{align}
    We deal with \text{Term (i)} and \text{Term (ii)} respectively.
    For \text{Term (i)}, this is exactly the same as the right hand side of \eqref{eq: proof bernstein bonus for optimal value 1}.
    Therefore, applying the same argument as Lemma~\ref{lem: bernstein bonus for optimal value} gives the following upper bound,
    \begin{align}
        \text{Term (i)} \leq \sqrt{\frac{\mathbb{V}_{\widehat{P}_h^k(\cdot|s,a)}\Big[V_{h+1,P^{\star},\boldsymbol{\Phi}}^{\star}\Big]\cdot c_1\iota}{N_h^k(s,a)\vee 1}} + \frac{c_2H\iota}{N_h^k(s,a)\vee 1} + \frac{1}{\sqrt{K}}.\label{eq: proof bernstein bound for value of pi k 0}
    \end{align}
    For \text{Term (ii)}, we first apply the second inequality of event $\cE$ to obtain that,
    \begin{align}
        &\text{Term (ii)} \label{eq: proof bernstein bound for value of pi k 1}\\
        &\quad \leq \sup_{\eta\in[0,H]}\! \left\{ \sum_{s'\in\cS}\left(\sqrt{\frac{\widehat{P}_h^k(s'|s,a)\cdot c_1\iota}{N_h^k(s,a)\vee 1}} + \frac{c_2\iota}{N_h^k(s,a)\vee 1}\right)\cdot\left|\big(\eta-V^{\pi^k}_{h+1, P^{\star},\boldsymbol{\Phi}}(s')\big)_+ \!\!- \big(\eta-V^{\star}_{h+1, P^{\star},\boldsymbol{\Phi}}(s')\big)_+\right| \right\}.
    \end{align}
    By the assumption that \eqref{eq: optimism and pessimism} holds at $(h+1,k)$, we can upper bound the absolute value above by
    \begin{align}
        \left|\big(\eta-V^{\pi^k}_{h+1, P^{\star},\boldsymbol{\Phi}}(s')\big)_+ - \big(\eta-V^{\star}_{h+1, P^{\star},\boldsymbol{\Phi}}(s')\big)_+\right|&\leq \left|V^{\pi^k}_{h+1, P^{\star},\boldsymbol{\Phi}}(s') - V^{\star}_{h+1, P^{\star},\boldsymbol{\Phi}}(s')\right| \\
        &\leq \overline{V}_{h+1}^k(s') - \underline{V}_{h+1}^k(s').\label{eq: proof bernstein bound for value of pi k 2}
    \end{align}
    where the first inequality is due to the $1$-Lipschitz continuity of $\psi_{\eta}(x) = (\eta - x)_+$, and the second inequality is due to \eqref{eq: optimism and pessimism}.
    Thus combining \eqref{eq: proof bernstein bound for value of pi k 1} and \eqref{eq: proof bernstein bound for value of pi k 2}, we know that
    \begin{align}
        \text{Term (ii)} \leq \sum_{s'\in\cS}\left(\sqrt{\frac{\widehat{P}_h^k(s'|s,a)\cdot c_1\iota}{N_h^k(s,a)\vee 1}} + \frac{c_2\iota}{N_h^k(s,a)\vee 1}\right)\cdot\left(\overline{V}_{h+1}^k(s') - \underline{V}_{h+1}^k(s')\right).\label{eq: proof bernstein bound for value of pi k 3}
    \end{align}
    Now following the argument first identified by \cite{azar2017minimax}, we proceed to upper bound \eqref{eq: proof bernstein bound for value of pi k 3} as
    \begin{align}
        \text{Term (ii)} &\leq \sum_{s'\in\cS}\left(\frac{\widehat{P}_h^k(s'|s,a)}{H}+\frac{c_1 H\iota}{N_h^k(s,a)\vee 1} + \frac{c_2\iota}{N_h^k(s,a)\vee 1}\right)\cdot\left(\overline{V}_{h+1}^k(s') - \underline{V}_{h+1}^k(s')\right) \\
        &\leq \frac{\mathbb{E}_{\widehat{P}_h^k(\cdot|s,a)}\Big[\overline{V}_{h+1}^{k} - \underline{V}_{h+1}^{k}\Big]}{H} +  \frac{c_2'H^2S\iota}{N_h^k(s,a)\vee 1},\label{eq: proof bernstein bound for value of pi k 4}
    \end{align}
    where $c_2'>0$ is another absolute constant.
    The first inequality is by $\sqrt{ab}\leq a+b$ and the second inequality is due to $\overline{V}_{h+1}^k, \underline{V}_{h+1}^k\in[0,H]$.
    Finally, combining \eqref{eq: proof bernstein bound for value of pi k 0} and \eqref{eq: proof bernstein bound for value of pi k 4}, we prove Lemma~\ref{lem: bernstein bound for value of pi k}.
\end{proof}

\begin{lemma}[Bernstein bounds for TV robust sets and optimistic and pessimistic robust value estimators]\label{lem: bernstein bound for opt and pess value}
    Under event $\cE$ in \eqref{eq: typical event}, suppose that the optimism and pessimism \eqref{eq: optimism and pessimism} holds at $(h+1, k)$, it holds that
    \begin{align}
        &\max\bigg\{ \left|\mathbb{E}_{\mathcal{P}_{\rho}(s,a;\widehat{P}_h^k)}\Big[\overline{V}_{h+1}^k\Big] - \mathbb{E}_{\mathcal{P}_{\rho}(s,a;P_h^{\star})}\Big[\overline{V}_{h+1}^k\Big] \right|,  \left|\mathbb{E}_{\mathcal{P}_{\rho}(s,a;\widehat{P}_h^k)}\Big[\underline{V}_{h+1}^k\Big] - \mathbb{E}_{\mathcal{P}_{\rho}(s,a;P_h^{\star})}\Big[\underline{V}_{h+1}^k\Big] \right|\bigg\}\\
        &\qquad \leq  \sqrt{\frac{\mathbb{V}_{\widehat{P}_h^k(\cdot|s,a)}\Big[V_{h+1,P^{\star},\boldsymbol{\Phi}}^{\star}\Big]\cdot c_1\iota}{N_h^k(s,a)\vee 1}} + \frac{\mathbb{E}_{\widehat{P}_h^k(\cdot|s,a)}\Big[\overline{V}_{h+1}^{k} - \underline{V}_{h+1}^{k}\Big]}{H} + \frac{c_2'H^2S\iota}{N_h^k(s,a)\vee 1} + \frac{1}{\sqrt{K}},
    \end{align}
    where $\iota  = \log(S^3AH^2K^{3/2}/\delta)$ and $c_1, c_2'$ are absolute constants.
\end{lemma}

\begin{proof}[Proof of Lemma \ref{lem: bernstein bound for opt and pess value}]
    This follows from the same proof as Lemma~\ref{lem: bernstein bound for value of pi k} and is thus omitted.
\end{proof}

\begin{lemma}[Non-robust concentration]\label{lem: non-robust bound 1}
    Under event $\cE$ in \eqref{eq: typical event}, suppose that the optimism and pessimism \eqref{eq: optimism and pessimism} holds at $(h+1, k)$, then it holds that
    \begin{align}
        \bigg|\left(\mathbb{E}_{\widehat{P}_h^k(\cdot|s,a)} - \mathbb{E}_{P^{\star}_h(\cdot|s,a)}\right)\Big[\overline{V}_{h+1}^k - \underline{V}_{h+1}^k\Big]\bigg| \leq \frac{1}{H} \cdot \mathbb{E}_{P^{\star}_h(\cdot|s,a)}\Big[\overline{V}_{h+1}^k - \underline{V}_{h+1}^k\Big]+\frac{c_2'H^2S\iota}{N_h^k(s,a)\vee 1}.
    \end{align}
    where $\iota  = \log(S^2AH^2K^{3/2}/\delta)$ and $c_2'$ is an absolute constant.
\end{lemma}

\begin{proof}[Proof of Lemma \ref{lem: non-robust bound 1}]
    According to the second inequality of event $\cE$, we have that
    \begin{align}
        &\bigg|\left(\mathbb{E}_{\widehat{P}_h^k(\cdot|s,a)} - \mathbb{E}_{P^{\star}_h(\cdot|s,a)}\right)\Big[\overline{V}_{h+1}^k - \underline{V}_{h+1}^k\Big]\bigg|\\
        &\qquad \leq \sum_{s'\in\cS}\left(\sqrt{\frac{P_h^{\star}(s'|s,a)\cdot c_1\iota}{N_h^k(s,a)\vee 1}}+\frac{c_2\iota}{N_h^k(s,a)\vee 1}\right)\cdot\left(\overline{V}_{h+1}^{k}(s') - \underline{V}_{h+1}^k(s')\right),
    \end{align}
    where we also apply \eqref{eq: optimism and pessimism} that $\overline{V}_{h+1}^{k}(s') \geq  \underline{V}_{h+1}^k(s')$.
    Now using the same argument as \eqref{eq: proof bernstein bound for value of pi k 4} in the proof of Lemma~\ref{lem: bernstein bound for value of pi k}, we can arrive at
    \begin{align}
        \bigg|\left(\mathbb{E}_{\widehat{P}_h^k(\cdot|s,a)} - \mathbb{E}_{P^{\star}_h(\cdot|s,a)}\right)\Big[\overline{V}_{h+1}^k - \underline{V}_{h+1}^k\Big]\bigg| \leq \frac{\mathbb{E}_{P_h^{\star}(\cdot|s,a)}\Big[\overline{V}_{h+1}^{k}(s') - \underline{V}_{h+1}^k(s')\Big]}{H} + \frac{c_2'H^2S\iota}{N_h^k(s,a)\vee 1},
    \end{align}
    which finishes the proof of Lemma~\ref{lem: non-robust bound 1}.
\end{proof}

\subsubsection{Variance Analysis}

\begin{lemma}[Variance analysis 1]\label{lem: variance analysis 1}
    Suppose that the optimism and pessimism \eqref{eq: optimism and pessimism} holds at $(h+1, k)$,
    then the following inequality holds,
    \begin{align}
        \bigg|\mathbb{V}_{\widehat{P}_h^k(\cdot|s,a)}\Big[\Big(\overline{V}_{h+1}^k+\underline{V}_{h+1}^k\Big) / 2\Big]-\mathbb{V}_{\widehat{P}_h^k(\cdot|s,a)} \Big[V_{h+1,P^{\star},\boldsymbol{\Phi}}^{\star}\Big]\bigg|\leq 4 H \cdot \mathbb{E}_{\widehat{P}_h^k(\cdot|s,a)}\Big[\overline{V}_{h+1}^k-\underline{V}_{h+1}^k\Big].
    \end{align}
\end{lemma}

\begin{proof}[Proof of Lemma \ref{lem: variance analysis 1}]
    Directly consider that the left hand side can be upper bounded by the following,
    \begin{align}
        &\bigg|\mathbb{V}_{\widehat{P}_h^k(\cdot|s,a)}\Big[\Big(\overline{V}_{h+1}^k+\underline{V}_{h+1}^k\Big) / 2\Big]-\mathbb{V}_{\widehat{P}_h^k(\cdot|s,a)} \Big[V_{h+1,P^{\star},\boldsymbol{\Phi}}^{\star}\Big]\bigg| \\
        &\qquad \leq \Bigg|\mathbb{E}_{\widehat{P}_h^k(\cdot|s,a)}\bigg[\Big(\overline{V}_{h+1}^k+\underline{V}_{h+1}^k\Big)^2 / 4\bigg]-\mathbb{E}_{\widehat{P}_h^k(\cdot|s,a)} \bigg[\Big(V_{h+1,P^{\star},\boldsymbol{\Phi}}^{\star}\Big)^2\bigg]\Bigg| \\
        &\qquad\qquad + \bigg|\left(\mathbb{E}_{\widehat{P}_h^k(\cdot|s,a)}\Big[\Big(\overline{V}_{h+1}^k+\underline{V}_{h+1}^k\Big) / 2\Big]\right)^2-\left(\mathbb{E}_{\widehat{P}_h^k(\cdot|s,a)} \Big[V_{h+1,P^{\star},\boldsymbol{\Phi}}^{\star}\Big]\right)^2\bigg|.\label{eq: proof lem variance analysis 1 1}
    \end{align}
    Since all of $\overline{V}_{h+1}^k, \underline{V}_{h+1}^k, V_{h+1,P^{\star},\boldsymbol{\Phi}}^{\star}\in[0,H]$ (by the correctness of \eqref{eq: optimism and pessimism} and the definitions of $\overline{V}_{h+1}^k, \underline{V}_{h+1}^k$), we can further upper bound the right hand side of \eqref{eq: proof lem variance analysis 1 1} as
    \begin{align}
        \bigg|\mathbb{V}_{\widehat{P}_h^k(\cdot|s,a)}\Big[\Big(\overline{V}_{h+1}^k+\underline{V}_{h+1}^k\Big) / 2\Big]-\mathbb{V}_{\widehat{P}_h^k(\cdot|s,a)} \Big[V_{h+1,P^{\star},\boldsymbol{\Phi}}^{\star}\Big]\bigg| &\leq 4H\cdot\mathbb{E}_{\widehat{P}_h^k(\cdot|s,a)}\bigg[\Big|\Big(\overline{V}_{h+1}^k+\underline{V}_{h+1}^k\Big) / 2 - V_{h+1,P^{\star},\boldsymbol{\Phi}}^{\star}\Big|\bigg] \\
        &\leq 4H\cdot\mathbb{E}_{\widehat{P}_h^k(\cdot|s,a)}\Big[\overline{V}_{h+1}^k - \underline{V}_{h+1}^k\Big],
    \end{align}
    where the last inequality is due to the correctness of \eqref{eq: optimism and pessimism} at $(h+1,k)$.
    This proves Lemma~\ref{lem: variance analysis 1}.
\end{proof}

\begin{lemma}[Variance analysis 2]\label{lem: variance analysis 2}
    Under event $\cE$ in \eqref{eq: typical event}, suppose that optimism and pessimism \eqref{eq: optimism and pessimism} holds at $(h+1, k)$, then it holds that
    \begin{align*}
        & \bigg|\mathbb{V}_{\widehat{P}_h^k(\cdot|s,a)}\Big[\Big(\overline{V}_{h+1}^k+\underline{V}_{h+1}^k\Big) / 2\Big]-\mathbb{V}_{P_h^{\star}(\cdot|s,a)}\Big[ V_{h+1,P^{\star},\boldsymbol{\Phi}}^{\pi^k}\Big]\bigg| \leq  4 H \cdot \mathbb{E}_{P_h^{\star}(\cdot|s,a)}\Big[\overline{V}_{h+1}^k-\underline{V}_{h+1}^k\Big]+\frac{c_2'H^4 S \iota}{N_h^k(s, a)\vee 1}+1.
    \end{align*}
\end{lemma}

\begin{proof}[Proof of Lemma \ref{lem: variance analysis 2}]
    {
    We first compare the variance under the empirical kernel with the variance under the true kernel.
    Since $\big(\overline{V}_{h+1}^k+\underline{V}_{h+1}^k\big)/2\in[0,H]$,
    \begin{align}
        &\bigg|\mathbb{V}_{\widehat{P}_h^k(\cdot|s,a)}
        \Big[\big(\overline{V}_{h+1}^k+\underline{V}_{h+1}^k\big)/2\Big]
        -\mathbb{V}_{P_h^\star(\cdot|s,a)}
        \Big[\big(\overline{V}_{h+1}^k+\underline{V}_{h+1}^k\big)/2\Big]\bigg| \nonumber\\
        &\qquad\leq
        3H^2\sum_{s'\in\cS}
        \left|\widehat{P}_h^k(s'|s,a)-P_h^\star(s'|s,a)\right|. \label{eq: proof lem variance analysis 2 1}
    \end{align}
    Under event $\cE$, the last display is further bounded by
    \begin{align}
        3H^2
        \sum_{s'\in\cS}\left(\sqrt{\frac{P_h^\star(s'|s,a)c_1\iota}{N_h^k(s,a)\vee1}}
        +\frac{c_2\iota}{N_h^k(s,a)\vee1}\right) \leq 3H^2\left(
        \sqrt{\frac{c_1S\iota}{N_h^k(s,a)\vee1}}
        +\frac{c_2S\iota}{N_h^k(s,a)\vee1}\right).
    \end{align}
    Using $\sqrt{x}\leq x+1$ after adjusting constants yields
    \begin{align}
        &\bigg|\mathbb{V}_{\widehat{P}_h^k(\cdot|s,a)}
        \Big[\big(\overline{V}_{h+1}^k+\underline{V}_{h+1}^k\big)/2\Big]
        -\mathbb{V}_{P_h^\star(\cdot|s,a)}
        \Big[\big(\overline{V}_{h+1}^k+\underline{V}_{h+1}^k\big)/2\Big]\bigg|
        \leq 1+\frac{c_2'H^4S\iota}{N_h^k(s,a)\vee1}. \label{eq: proof lem variance analysis 2 2}
    \end{align}
    }
    Thus by \eqref{eq: proof lem variance analysis 2 2}, we can bound our target as
    \begin{align}
        &\bigg|\mathbb{V}_{\widehat{P}_h^k(\cdot|s,a)}\Big[\Big(\overline{V}_{h+1}^k+\underline{V}_{h+1}^k\Big) / 2\Big]-\mathbb{V}_{P_h^{\star}(\cdot|s,a)}\Big[ V_{h+1,P^{\star},\boldsymbol{\Phi}}^{\pi^k}\Big]\bigg|\\
        &\qquad \leq \bigg|\mathbb{V}_{P_h^{\star}(\cdot|s,a)}\Big[\Big(\overline{V}_{h+1}^k+\underline{V}_{h+1}^k\Big) / 2\Big]-\mathbb{V}_{P_h^{\star}(\cdot|s,a)}\Big[ V_{h+1,P^{\star},\boldsymbol{\Phi}}^{\pi^k}\Big]\bigg| + \frac{c_2'H^4S\iota}{N_h^k(s,a)\vee 1} + 1.\label{eq: proof lem variance analysis 2 3}
    \end{align}
    Now by the same proof of Lemma~\ref{lem: variance analysis 1}, using the correctness of \eqref{eq: optimism and pessimism} at $(h+1,k)$, we can show that
    \begin{align}
        \bigg|\mathbb{V}_{P_h^{\star}(\cdot|s,a)}\Big[\Big(\overline{V}_{h+1}^k+\underline{V}_{h+1}^k\Big) / 2\Big]-\mathbb{V}_{P_h^{\star}(\cdot|s,a)}\Big[ V_{h+1,P^{\star},\boldsymbol{\Phi}}^{\pi^k}\Big]\bigg| \leq 4H\cdot\mathbb{E}_{P_h^{\star}(\cdot|s,a)}\Big[\overline{V}_{h+1}^k-\underline{V}_{h+1}^k\Big].\label{eq: proof lem variance analysis 2 4}
    \end{align}
    Combining \eqref{eq: proof lem variance analysis 2 3} and \eqref{eq: proof lem variance analysis 2 4}, we can finish the proof of Lemma~\ref{lem: variance analysis 2}.
\end{proof}

\subsubsection{Other Auxiliary Lemmas}

\begin{lemma}[Lemma 7.5 in \cite{agarwal2019reinforcement}]\label{lem: lem7.5'}
    For the sequences of $\{s_h^k, a_h^k\}_{h, k=1}^{H, K}$, it holds that
    \begin{align}
        \sum_{k=1}^K \sum_{h=1}^H\frac{1}{N_h^k(s_h^k, a_h^k)\vee 1} \leq c\cdot HSA\log(K).
    \end{align}
    where $c>0$ is an absolute constant.
\end{lemma}
\begin{proof}[Proof of Lemma \ref{lem: lem7.5'}]
    See Lemma 7.5 in \cite{agarwal2019reinforcement} for a detailed proof.
\end{proof}

\section{Proofs for Extension \texorpdfstring{I}{I} (Section~\ref{subsec: extentions})}

In this section, we prove the theoretical results in Section~\ref{subsec: extentions}.

\subsection{Proof of Corollary~\ref{cor: regret discount}}\label{subsec: proof cor regret discount}

\begin{proof}[Proof of Corollary~\ref{cor: regret discount}]
    We consider applying Algorithm~\ref{alg: tv} on the auxiliary $\widetilde{\cS}\times\cA$-rectangular RMDP with a TV robust set $\widetilde{\cM}$ (see Section~\ref{subsec: extentions}) which satisfies the vanishing minimal value assumption (Assumption~\ref{ass: zero min}).
    Suppose the algorithm outputs $\widetilde{\pi}^1,\cdots,\widetilde{\pi}^K$ for the $K$ episodes.
    Then Theorem~\ref{thm: regret tv} shows that by a proper choice of the hyperparameters, with probability at least $1-\delta$
    \begin{align}
        \mathrm{Regret}_{\widetilde{\boldsymbol{\Phi}}}(K) = \sum_{k=1}^K\max_{\widetilde{\pi}}V_{1, \widetilde{P}^{\star}, \widetilde{\boldsymbol{\Phi}}}^{\widetilde{\pi}}(s_1) - V_{1, \widetilde{P}^{\star}, \widetilde{\boldsymbol{\Phi}}}^{\widetilde{\pi}^k}(s_1)  \leq \mathcal{O}\bigg(\sqrt{\min\big\{H,\rho^{-1}\big\} H^2(S+1)AK\iota'}\,\bigg),\label{eq: proof discounted 0}.
    \end{align}
    where $\iota' = \log^2(SAHK/\delta)$ and $\rho = 1-\rho'\in[0,1)$.
    In the sequel, we prove that for any policy $\widetilde{\pi}$ of $\widetilde{\cM}$ and its induced policy $\widetilde{\pi}_{\cS}$ of $\cM_{\gamma}$, their robust value functions coincide at the initial state $s_1\in\cS$, that is,
    \begin{align}
        V_{1, \widetilde{P}^{\star}, \widetilde{\boldsymbol{\Phi}}}^{\widetilde{\pi}}(s_1) = V_{1, P^{\star}, \boldsymbol{\Phi}'}^{\widetilde{\pi}_{\cS}}(s_1),
    \end{align}
    where $V_{1, \widetilde{P}^{\star}, \widetilde{\boldsymbol{\Phi}}}^{\widetilde{\pi}}$ is the robust value function of $\widetilde{\pi}$ in $\widetilde{\cM} = (\widetilde{\cS},\cA,H, \widetilde{P}^{\star}, \widetilde{R}, \widetilde{\boldsymbol{\Phi}})$, and $V_{1, P^{\star}, \boldsymbol{\Phi}'}^{\widetilde{\pi}_{\cS}}$ is the robust value function of $\widetilde{\pi}_{\cS}$ in $\cM_{\gamma} = (\cS,\cA,H,P^{\star}, R_{\gamma},\boldsymbol{\Phi}')$.
    To this end, we actually prove a stronger result that for any step $h\in[H]$, it holds that
    \begin{align}
        (\rho')^{h-1}\cdot V_{h, \widetilde{P}^{\star}, \widetilde{\boldsymbol{\Phi}}}^{\widetilde{\pi}}(s) = V_{h, P^{\star}, \boldsymbol{\Phi}'}^{\widetilde{\pi}_{\cS}}(s),\quad \forall s\in\cS.\label{eq: proof discounted 1}
    \end{align}
    We prove \eqref{eq: proof discounted 1} by induction.
    For step $H$, by robust Bellman equation, we have that, for any $(s,a)\in\cS\times\cA$,
    \begin{align}
        (\rho')^{H-1}\cdot Q_{H, \widetilde{P}^{\star}, \widetilde{\boldsymbol{\Phi}}}^{\widetilde{\pi}}(s, a) = (\rho')^{H-1}\cdot \left(\frac{\gamma}{\rho'}\right)^{H-1}\cdot R_H(s,a) = R_{\gamma, H}(s,a) = Q_{H, P^{\star}, \boldsymbol{\Phi}'}^{\widetilde{\pi}_{\cS}}(s, a),
    \end{align}
    and thus for any $s\in\cS$,
    \begin{align}
        (\rho')^{H-1}\cdot V_{H, \widetilde{P}^{\star}, \widetilde{\boldsymbol{\Phi}}}^{\widetilde{\pi}}(s) = \mathbb{E}_{\widetilde{\pi}(\cdot|s)}\Big[(\rho')^{H-1}\cdot Q_{H, \widetilde{P}^{\star}, \widetilde{\boldsymbol{\Phi}}}^{\widetilde{\pi}}(s,\cdot)\Big] =\mathbb{E}_{\widetilde{\pi}_{\cS}(\cdot|s)}\Big[Q_{H, P^{\star}, \boldsymbol{\Phi}'}^{\widetilde{\pi}_{\cS}}(s,\cdot)\Big] = V_{H, P^{\star}, \boldsymbol{\Phi}'}^{\widetilde{\pi}_{\cS}}(s).
    \end{align}
    This proves \eqref{eq: proof discounted 1} for step $H$.
    Suppose that \eqref{eq: proof discounted 1} holds at some step $h+1$, that is,
    \begin{align}
        (\rho')^{h}\cdot V_{h+1, \widetilde{P}^{\star}, \widetilde{\boldsymbol{\Phi}}}^{\widetilde{\pi}}(s) = V_{h+1, P^{\star}, \boldsymbol{\Phi}'}^{\widetilde{\pi}_{\cS}}(s),\quad \forall s\in\cS.\label{eq: proof discounted 2}
    \end{align}
    Then for step $h$, by robust Bellman equation and Proposition~\ref{prop: equivalent robust set}, we have that
    \allowdisplaybreaks
    \begin{align}
        (\rho')^{h-1}\cdot Q_{h, \widetilde{P}^{\star}, \widetilde{\boldsymbol{\Phi}}}^{\widetilde{\pi}}(s, a) &= (\rho')^{h-1}\cdot \left(\frac{\gamma}{\rho'}\right)^{h-1}\cdot R_h(s,a)  + (\rho')^{h-1}\cdot \mathbb{E}_{\widetilde{\mathcal{P}}_{\rho}(s,a;\widetilde{P}_h^{\star})}\Big[V_{h+1, \widetilde{P}^{\star}, \widetilde{\boldsymbol{\Phi}}}^{\widetilde{\pi}}\Big] \\
        &=R_{\gamma,h}(s,a) + (\rho')^{h-1}\cdot \rho'\cdot \mathbb{E}_{\widetilde{\mathcal{B}}_{\rho'}(s,a;\widetilde{P}_h^{\star})}\Big[V_{h+1, \widetilde{P}^{\star}, \widetilde{\boldsymbol{\Phi}}}^{\widetilde{\pi}}\Big],\label{eq: proof discounted 3}
    \end{align}
    where the last equality utilizes Proposition~\ref{prop: equivalent robust set} since $\min_{s\in\widetilde{\cS}}V_{h+1, \widetilde{P}^{\star}, \widetilde{\boldsymbol{\Phi}}}^{\widetilde{\pi}}(s)=0$, and we adopt the notation
    \begin{align}
        \widetilde{\mathcal{B}}_{\rho'}(s,a;\widetilde{P}_h^{\star}) = \left\{\widetilde{P}(\cdot) \in \Delta(\widetilde{\cS}):  \sup_{s' \in \widetilde{\cS}}\frac{\widetilde{P}(s')}{\widetilde{P}_h^\star(s' | s, a)} \le \frac{1}{\rho'} \right\}.
    \end{align}
    Notice that by the definition~\eqref{eq: transition new}, we know for $(s,a)\in\cS\times\cA$ it holds that $\widetilde{P}_h^\star(\cdot| s, a) = P_h^\star(\cdot| s, a)$ which is supported on $\cS$.
    Therefore, we can equivalently write
    \begin{align}
        \widetilde{\mathcal{B}}_{\rho'}(s,a;\widetilde{P}_h^{\star}) &= \left\{\widetilde{P}(\cdot) \in \Delta(\widetilde{\cS}):  \sup_{s' \in \cS}\frac{\widetilde{P}(s')}{\widetilde{P}_h^\star(s' | s, a)} \le \frac{1}{\rho'} \right\} \\
        &= \left\{\widetilde{P}(\cdot) \in \Delta(\cS):  \sup_{s' \in \cS}\frac{\widetilde{P}(s')}{P_h^\star(s' | s, a)} \le \frac{1}{\rho'} \right\}\\
        &= \mathcal{B}_{\rho'}(s,a;P_h^{\star}).\label{eq: proof discounted 4}
    \end{align}
    Thus by \eqref{eq: proof discounted 3} and \eqref{eq: proof discounted 4} and the induction hypothesis \eqref{eq: proof discounted 2}, we obtain that for any $(s,a)\in\cS\times\cA$,
    \begin{align}
        (\rho')^{h-1}\cdot Q_{h, \widetilde{P}^{\star}, \widetilde{\boldsymbol{\Phi}}}^{\widetilde{\pi}}(s, a) &= R_{\gamma,h}(s,a) + (\rho')^{h}\cdot\mathbb{E}_{\mathcal{B}_{\rho'}(s,a;P_h^{\star})}\Big[V_{h+1, \widetilde{P}^{\star}, \widetilde{\boldsymbol{\Phi}}}^{\widetilde{\pi}}\Big] \\
        &=  R_{\gamma,h}(s,a) + \mathbb{E}_{\mathcal{B}_{\rho'}(s,a;P_h^{\star})}\Big[V_{h+1, P^{\star}, \boldsymbol{\Phi}'}^{\widetilde{\pi}_{\cS}}\Big] = Q_{h, P^{\star}, \boldsymbol{\Phi}'}^{\widetilde{\pi}_{\cS}}(s, a),
    \end{align}
    where the second equality applies \eqref{eq: proof discounted 2} and the last equality is from robust Bellman equation.
    Consequently, for any $s\in\cS$, we have that
    \begin{align}
        (\rho')^{h-1}\cdot V_{h, \widetilde{P}^{\star}, \widetilde{\boldsymbol{\Phi}}}^{\widetilde{\pi}}(s) = \mathbb{E}_{\widetilde{\pi}(\cdot|s)}\Big[(\rho')^{h-1}\cdot Q_{h, \widetilde{P}^{\star}, \widetilde{\boldsymbol{\Phi}}}^{\widetilde{\pi}}(s,\cdot)\Big] =\mathbb{E}_{\widetilde{\pi}_{\cS}(\cdot|s)}\Big[Q_{h, P^{\star}, \boldsymbol{\Phi}'}^{\widetilde{\pi}_{\cS}}(s,\cdot)\Big] = V_{h, P^{\star}, \boldsymbol{\Phi}'}^{\widetilde{\pi}_{\cS}}(s),
    \end{align}
    which finishes the induction argument, proving our claim \eqref{eq: proof discounted 1}.
    By taking $h=1$, we can derive that for any initial state $s_1\in\cS$, it holds that for any policy $\widetilde{\pi}$ of $\widetilde{\cM}$ and its induced policy $\widetilde{\pi}_{\cS}$ of $\cM_{\gamma}$,
    \begin{align}
         V_{1, \widetilde{P}^{\star}, \widetilde{\boldsymbol{\Phi}}}^{\widetilde{\pi}}(s_1) = V_{1, P^{\star}, \boldsymbol{\Phi}'}^{\widetilde{\pi}_{\cS}}(s_1).
    \end{align}
    This indicates two facts: the first is that
    \begin{align}
        \max_{\widetilde{\pi}}V_{1, \widetilde{P}^{\star}, \widetilde{\boldsymbol{\Phi}}}^{\widetilde{\pi}}(s_1) = \max_{\pi} V_{1, P^{\star}, \boldsymbol{\Phi}'}^{\pi}(s_1),\label{eq: proof discounted 5}
    \end{align}
    where on the right hand side the maximization is with respect to all the policies for $\cM_{\gamma}$; the second is that
    \begin{align}
        V_{1, \widetilde{P}^{\star}, \widetilde{\boldsymbol{\Phi}}}^{\widetilde{\pi}^k}(s_1) = V_{1, P^{\star}, \boldsymbol{\Phi}'}^{\widetilde{\pi}^k_{\cS}}(s_1),\label{eq: proof discounted 6}
    \end{align}
    for each $k\in[K]$, where we recall that $\widetilde{\pi}^k$ is the policy output by Algorithm~\ref{alg: tv} for episode $k$.
    As a result, the $K$ policies $\{\widetilde{\pi}^k_{\cS}\}_{k=1}^K$ of $\cM_{\gamma}$ during interactive data collection satisfy with probability at least $1-\delta$,
    \begin{align}
        \mathrm{Regret}_{\boldsymbol{\Phi}'}(K) &= \sum_{k=1}^K\max_{\pi}V_{1, P^{\star}, \boldsymbol{\Phi}'}^{\pi}(s_1) - V_{1, P^{\star}, \boldsymbol{\Phi}'}^{\widetilde{\pi}^k_{\cS}}(s_1)  \\
        &= \sum_{k=1}^K\max_{\widetilde{\pi}}V_{1, \widetilde{P}^{\star}, \widetilde{\boldsymbol{\Phi}}}^{\widetilde{\pi}}(s_1) - V_{1, \widetilde{P}^{\star}, \widetilde{\boldsymbol{\Phi}}}^{\widetilde{\pi}^k}(s_1)  \\
        & \leq \mathcal{O}\bigg(\sqrt{\min\big\{H,(1-\rho')^{-1}\big\} H^2SAK\iota'}\,\bigg),
    \end{align}
    where in the second equality we apply the facts \eqref{eq: proof discounted 5} and \eqref{eq: proof discounted 6}, and the last inequality follows from \eqref{eq: proof discounted 0} and that $\rho = 1-\rho'$.
    This completes the proof of Corollary~\ref{cor: regret discount}.
\end{proof}


\section{Proofs for Extension II (Section~\ref{sec:rmg_extension})}

\subsection{Proof of Proposition~\ref{prop:robust_rmg_minimax}}\label{subsec:proof_rmg_minimax}

\begin{proof}[Proof of Proposition \ref{prop:robust_rmg_minimax}]
    This proposition is a finite-horizon robust analogue of Shapley's recursion for zero-sum stochastic games,
    where $\cS\times\cA\times\cB$-rectangular uncertainty ensures time consistency and enables backward induction \citep{blanchet2023double}.
    Our proof is based on the backward induction method.

\paragraph{Step 1: stage-game minimax at each $(h,s)$.}
Fix $(h,s)$ and treat $V^\star_{h+1,P^\star,\mathbf{\Phi}}$ as given.
By finiteness of $\mathcal S$ and compactness of the TV-ball, the infimum defining the robust expectation in \eqref{eq:robust_shapley_Q} is attained for every $(s,a,b)$, hence $Q^\star_{h,P^\star,\mathbf{\Phi}}(s,a,b)$ is well-defined.
Thus, for fixed $(h,s)$, $Q^\star_{h,P^\star,\mathbf{\Phi}}(s,\cdot,\cdot)$ defines a finite two-player zero-sum matrix game.
By von Neumann's minimax theorem, there exist mixed actions
$\pi_h^\star(\cdot|s)\in\Delta(\mathcal A)$ and $\nu_h^\star(\cdot|s)\in\Delta(\mathcal B)$ such that
\begin{equation}
\max_{\pi\in\Delta(\mathcal A)}\min_{\nu\in\Delta(\mathcal B)}\mathbb E_{\pi,\nu}\!\left[Q^\star_{h,P^\star,\mathbf{\Phi}}(s,a,b)\right]
=
\mathbb E_{\pi_h^\star,\nu_h^\star}\!\left[Q^\star_{h,P^\star,\mathbf{\Phi}}(s,a,b)\right]
=
\min_{\nu\in\Delta(\mathcal B)}\max_{\pi\in\Delta(\mathcal A)}\mathbb E_{\pi,\nu}\!\left[Q^\star_{h,P^\star,\mathbf{\Phi}}(s,a,b)\right].
\label{eq:stage_minimax_rmg}
\end{equation}
Collecting $\{\pi_h^\star(\cdot|s)\}_{h,s}$ and $\{\nu_h^\star(\cdot|s)\}_{h,s}$ yields Markov policies $(\pi^\star,\nu^\star)$.

\paragraph{Step 2: backward induction.} We now show \eqref{eq:robust_saddle_ineq} holds by backward induction on $h$. The claim is trivial at $h=H+1$ since $V^{\pi,\nu}_{H+1,P^\star,\mathbf{\Phi}}\equiv 0$.
Assume \eqref{eq:robust_saddle_ineq} holds at stage $h+1$ for all states.

Fix any opponent policy $\nu$.
By \eqref{eq:game_robust_bellman},
\[
V^{\pi^\star,\nu}_{h,P^\star,\mathbf{\Phi}}(s)
=
\mathbb E_{a\sim\pi_h^\star(\cdot|s),\,b\sim\nu_h(\cdot|s)}
\Big[
R_h(s,a,b)+\mathbb E_{\mathcal P_\rho(s,a,b;P_h^\star)}\!\left[V^{\pi^\star,\nu}_{h+1,P^\star,\mathbf{\Phi}}\right]
\Big].
\]
Since $V^{\pi^\star,\nu}_{h+1,P^\star,\mathbf{\Phi}}(s')\ge V^\star_{h+1,P^\star,\mathbf{\Phi}}(s')$ for all $s'$ by the induction hypothesis, monotonicity of the robust expectation yields
\[
\mathbb E_{\mathcal P_\rho(s,a,b;P_h^\star)}\!\left[V^{\pi^\star,\nu}_{h+1,P^\star,\mathbf{\Phi}}\right]
\ge
\mathbb E_{\mathcal P_\rho(s,a,b;P_h^\star)}\!\left[V^\star_{h+1,P^\star,\mathbf{\Phi}}\right].
\]
Therefore
\[
V^{\pi^\star,\nu}_{h,P^\star,\mathbf{\Phi}}(s)
\ge
\mathbb E_{a\sim\pi_h^\star,\,b\sim\nu_h}\!\left[Q^\star_{h,P^\star,\mathbf{\Phi}}(s,a,b)\right]
\ge
\min_{\lambda\in\Delta(\mathcal B)}\mathbb E_{a\sim\pi_h^\star,\,b\sim\lambda}\!\left[Q^\star_{h,P^\star,\mathbf{\Phi}}(s,a,b)\right]
=
V^\star_{h,P^\star,\mathbf{\Phi}}(s),
\]
where the last equality follows from the maximin optimality of $\pi_h^\star(\cdot|s)$ in \eqref{eq:stage_minimax_rmg}.

Fix any Player~1 policy $\pi$.
Similarly, by \eqref{eq:game_robust_bellman} and induction,
$V^{\pi,\nu^\star}_{h+1,P^\star,\mathbf{\Phi}}(s')\le V^\star_{h+1,P^\star,\mathbf{\Phi}}(s')$ for all $s'$,
hence (by monotonicity)
\[
\mathbb E_{\mathcal P_\rho(s,a,b;P_h^\star)}\!\left[V^{\pi,\nu^\star}_{h+1,P^\star,\mathbf{\Phi}}\right]
\le
\mathbb E_{\mathcal P_\rho(s,a,b;P_h^\star)}\!\left[V^\star_{h+1,P^\star,\mathbf{\Phi}}\right],
\]
which yields
\[
V^{\pi,\nu^\star}_{h,P^\star,\mathbf{\Phi}}(s)
\le
\mathbb E_{a\sim\pi_h,\,b\sim\nu_h^\star}\!\left[Q^\star_{h,P^\star,\mathbf{\Phi}}(s,a,b)\right]
\le
\max_{\mu\in\Delta(\mathcal A)}\mathbb E_{a\sim\mu,\,b\sim\nu_h^\star}\!\left[Q^\star_{h,P^\star,\mathbf{\Phi}}(s,a,b)\right]
=
V^\star_{h,P^\star,\mathbf{\Phi}}(s),
\]
where the last equality is the minimax optimality of $\nu_h^\star(\cdot|s)$ in \eqref{eq:stage_minimax_rmg}.
Combining the two bounds gives \eqref{eq:robust_saddle_ineq} and, by taking $\nu=\nu^\star$ (or $\pi=\pi^\star$),
also $V^{\pi^\star,\nu^\star}_{h,P^\star,\mathbf{\Phi}}(s)=V^\star_{h,P^\star,\mathbf{\Phi}}(s)$.

\paragraph{Step 3: strong duality.} The saddle inequalities \eqref{eq:robust_saddle_ineq} directly imply \eqref{eq:robust_global_minimax}.
\end{proof}

\subsection{Proof of Theorem~\ref{thm:onesided_regret}}
\label{sec:game_proof}

\begin{proof}[Proof of Theorem \ref{thm:onesided_regret}]
\begingroup
We recall the bonus choice in \eqref{eq: mg_bonus}:
\begin{equation}
    \mathrm{bonus}_h^k(s,a,b)
    =
    c_b\min\{H,\rho^{-1}\}
    \sqrt{
    \frac{S\iota}{N_h^k(s,a,b)\vee 1}
    },
    \qquad
    \iota=\log(SABHK/\delta),
    \label{eq:mg_hoeffding_bonus_new}
\end{equation}
where $c_b>0$ is a sufficiently large absolute constant.
The proof uses a positive-part regret decomposition and controls the robust Bellman error uniformly through an $L_1$ transition-estimation bound.

\paragraph{Minimax operator.}
For any $Q:\mathcal S\times\mathcal A\times\mathcal B\to\mathbb R$, define the state-wise minimax operator
\[
(\mathsf{MM}[Q])(s):=\max_{\pi\in\Delta(\mathcal A)}\min_{\nu\in\Delta(\mathcal B)}
\mathbb E_{a\sim\pi,\,b\sim\nu}[Q(s,a,b)].
\]
It is standard that $\mathsf{MM}$ is monotone: if $Q\le Q'$ entrywise, then $\mathsf{MM}[Q]\le \mathsf{MM}[Q']$ entrywise.

\paragraph{Step 1: a uniform transition-estimation event.}
We use a standard multinomial concentration event. There exists an event $\mathcal E$ with
$\mathbb P(\mathcal E)\ge 1-\delta/2$ such that, for all
$(k,h,s,a,b)\in[K]\times[H]\times\mathcal S\times\mathcal A\times\mathcal B$,
\begin{equation}
    \left\|
    \widehat P_h^k(\cdot|s,a,b)-P_h^\star(\cdot|s,a,b)
    \right\|_1
    \le
    c_e
    \sqrt{
    \frac{S\iota}{N_h^k(s,a,b)\vee 1}
    },
    \label{eq:onesided_l1_mg_new}
\end{equation}
where $c_e>0$ is an absolute constant and $\iota=\log(SABHK/\delta)$.
The deterministic inequalities below are derived on $\mathcal E$; the martingale concentration step is applied separately and then combined with $\mathcal E$ by a union bound.

We also use the game analogue of Proposition~\ref{prop: gap}: under the same TV ambiguity, the robust value and Q-functions have span at most $\min\{H,\rho^{-1}\}$.
The proof is identical to the single-agent backward-induction argument, with $(s,a)$ replaced by $(s,a,b)$ and the state-wise maximization replaced by the monotone max--min operator.
Together with Assumption~\ref{ass:onesided_vmv} and nonnegative rewards, this span bound implies that all true robust values and Q-values appearing below lie in $[0,\min\{H,\rho^{-1}\}]$; the optimistic iterates lie in the same interval by the clipping in \eqref{eq: game Q overline}.

We next record the consequence of \eqref{eq:onesided_l1_mg_new} for the TV-dual operator used in the robust Bellman updates.
For any function $f:\mathcal S\to[0,\min\{H,\rho^{-1}\}]$, this operator is
\[
\mathbb E_{\mathcal P_\rho(s,a,b;P)}[f]
=
\sup_{\eta\in[0,\min\{H,\rho^{-1}\}]}
\left\{
-
\mathbb E_{P(\cdot|s,a,b)}[(\eta-f)_+]
+(1-\rho)\eta
\right\}.
\]
The restriction of $\eta$ to $[0,\min\{H,\rho^{-1}\}]$ is without loss because $0\le f\le\min\{H,\rho^{-1}\}$.
Following the convention in \eqref{eq: duality tv}, we use this display as the dual operator in the algorithm.
When $f$ is $V^\star_{h,P^\star,\mathbf{\Phi}}$ or a realized value function $V^{\pi,\nu}_{h,P^\star,\mathbf{\Phi}}$, Assumption~\ref{ass:onesided_vmv} ensures the required vanishing condition, so the same display is the true TV-robust Bellman term.
Therefore, on $\mathcal E$,
\begin{align}
    &
    \left|
    \mathbb E_{\mathcal P_\rho(s,a,b;\widehat P_h^k)}[f]
    -
    \mathbb E_{\mathcal P_\rho(s,a,b;P_h^\star)}[f]
    \right|
    \nonumber\\
    &\qquad\le
    \sup_{\eta\in[0,\min\{H,\rho^{-1}\}]}
    \left|
    \mathbb E_{\widehat P_h^k(\cdot|s,a,b)}[(\eta-f)_+]
    -
    \mathbb E_{P_h^\star(\cdot|s,a,b)}[(\eta-f)_+]
    \right|
    \nonumber\\
    &\qquad\le
    \min\{H,\rho^{-1}\}
    \left\|
    \widehat P_h^k(\cdot|s,a,b)-P_h^\star(\cdot|s,a,b)
    \right\|_1
    \nonumber\\
    &\qquad\le
    \mathrm{bonus}_h^k(s,a,b),
    \label{eq:onesided_robust_hoeffding_mg_new}
\end{align}
where the last inequality follows by choosing $c_b$ sufficiently large.

\paragraph{Step 2: proper bonus for TV-robust expectations.}
By construction, the optimistic value iterates satisfy
$0\le \overline V_h^k\le\min\{H,\rho^{-1}\}$ for all $(h,k)$.
Thus \eqref{eq:onesided_robust_hoeffding_mg_new} implies that, on $\mathcal E$,
for all $(k,h,s,a,b)$,
\begin{equation}
    \left|
    \mathbb E_{\mathcal P_\rho(s,a,b;\widehat P_h^k)}
    \big[\overline V_{h+1}^k\big]
    -
    \mathbb E_{\mathcal P_\rho(s,a,b;P_h^\star)}
    \big[\overline V_{h+1}^k\big]
    \right|
    \le
    \mathrm{bonus}_h^k(s,a,b).
    \label{eq:proper_bonus_over_new}
\end{equation}

\paragraph{Step 3: optimism by backward induction.}
We next establish the optimism property needed for regret analysis.
On event $\mathcal E$, for all $k\in[K]$, $h\in[H]$, and $s\in\mathcal S$,
\begin{equation}
V^\star_{h,P^\star,\mathbf{\Phi}}(s)
\le
\overline V_h^k(s),
\label{eq:onesided_optimism_hoeffding_new}
\end{equation}
and entrywise, for all $(s,a,b)\in\mathcal S\times\mathcal A\times\mathcal B$,
\begin{equation}
Q^\star_{h,P^\star,\mathbf{\Phi}}(s,a,b)\le \overline Q_h^k(s,a,b).
\label{eq:onesided_Q_optimism_hoeffding_new}
\end{equation}
The proof is by backward induction on $h$.
For $h=H+1$, the claim is immediate.
Suppose the claim holds at step $h+1$.
For every $(s,a,b)$, the robust Bellman recursion, the induction hypothesis, monotonicity of the dual TV operator, and \eqref{eq:proper_bonus_over_new} imply
\begin{align*}
Q^\star_{h,P^\star,\mathbf{\Phi}}(s,a,b)
&=
R_h(s,a,b)
+
\mathbb E_{\mathcal P_\rho(s,a,b;P_h^\star)}
\big[V^\star_{h+1,P^\star,\mathbf{\Phi}}\big]\\
&\le
R_h(s,a,b)
+
\mathbb E_{\mathcal P_\rho(s,a,b;P_h^\star)}
\big[\overline V_{h+1}^k\big]\\
&\le
R_h(s,a,b)
+
\mathbb E_{\mathcal P_\rho(s,a,b;\widehat P_h^k)}
\big[\overline V_{h+1}^k\big]
+
\mathrm{bonus}_h^k(s,a,b).
\end{align*}
Since
\[
Q^\star_{h,P^\star,\mathbf{\Phi}}(s,a,b)\le\min\{H,\rho^{-1}\},
\]
the clipping in the definition of $\overline Q_h^k$ is harmless, and hence
\[
Q^\star_{h,P^\star,\mathbf{\Phi}}(s,a,b)\le \overline Q_h^k(s,a,b).
\]
The monotonicity of $\mathsf{MM}$ then implies
\[
V^\star_{h,P^\star,\mathbf{\Phi}}(s)
=
(\mathsf{MM}[Q^\star_{h,P^\star,\mathbf{\Phi}}])(s)
\le
(\mathsf{MM}[\overline Q_h^k])(s)
=
\overline V_h^k(s).
\]
This proves \eqref{eq:onesided_optimism_hoeffding_new}--\eqref{eq:onesided_Q_optimism_hoeffding_new}.

\paragraph{Step 4: episode-wise regret decomposition.}
Fix an adaptive Markov opponent sequence $\nu=\{\nu^k\}_{k=1}^K$ that is non-anticipating.
By the fixed-initial-state convention, $s_1^k=s_1$ for every episode $k$.
Define the realized-opponent optimistic gap
\[
D_h^k(s):=
\Big(
\overline V_h^k(s)-V^{\pi^k,\nu^k}_{h,P^\star,\mathbf{\Phi}}(s)
\Big)_+,
\qquad
\forall (h,k,s)\in[H]\times[K]\times\mathcal S.
\]
By \eqref{eq:onesided_optimism_hoeffding_new}, for each episode $k$,
\[
V^\star_{1,P^\star,\mathbf{\Phi}}(s_1^k)-V^{\pi^k,\nu^k}_{1,P^\star,\mathbf{\Phi}}(s_1^k)
\le
\overline V_1^k(s_1^k)-V^{\pi^k,\nu^k}_{1,P^\star,\mathbf{\Phi}}(s_1^k)
\le
D_1^k(s_1^k).
\]
Thus it suffices to upper bound $\sum_{k=1}^K D_1^k(s_1^k)$.

Now fix $(h,k,s,a,b)$.
By the definition of $\overline Q_h^k$, the robust Bellman recursion for
$Q^{\pi^k,\nu^k}_{h,P^\star,\mathbf{\Phi}}$, and \eqref{eq:proper_bonus_over_new},
\begin{align}
    &\overline Q_h^k(s,a,b)
    -
    Q^{\pi^k,\nu^k}_{h,P^\star,\mathbf{\Phi}}(s,a,b)
    \nonumber\\
    &\qquad\le
    \mathbb E_{\mathcal P_\rho(s,a,b;P_h^\star)}
    \big[\overline V_{h+1}^k\big]
    -
    \mathbb E_{\mathcal P_\rho(s,a,b;P_h^\star)}
    \big[V^{\pi^k,\nu^k}_{h+1,P^\star,\mathbf{\Phi}}\big]
    +
    2\mathrm{bonus}_h^k(s,a,b).
    \label{eq:game_hoeffding_analysis_1}
\end{align}
For the robust-expectation difference, the TV dual form gives
\begin{align}
    &\mathbb E_{\mathcal P_\rho(s,a,b;P_h^\star)}
    \big[\overline V_{h+1}^k\big]
    -
    \mathbb E_{\mathcal P_\rho(s,a,b;P_h^\star)}
    \big[V^{\pi^k,\nu^k}_{h+1,P^\star,\mathbf{\Phi}}\big]
    \nonumber\\
    &\qquad\le
    \sup_{\eta\in[0,\min\{H,\rho^{-1}\}]}
    \mathbb E_{P_h^\star(\cdot|s,a,b)}
    \left[
    (\eta-V^{\pi^k,\nu^k}_{h+1,P^\star,\mathbf{\Phi}})_+
    -
    (\eta-\overline V_{h+1}^k)_+
    \right]
    \nonumber\\
    &\qquad\le
    \mathbb E_{P_h^\star(\cdot|s,a,b)}
    \left[
    \Big(
    \overline V_{h+1}^k
    -
    V^{\pi^k,\nu^k}_{h+1,P^\star,\mathbf{\Phi}}
    \Big)_+
    \right]
    =
    \mathbb E_{P_h^\star(\cdot|s,a,b)}[D_{h+1}^k],
    \label{eq:game_hoeffding_analysis_2}
\end{align}
where the second inequality uses
$(\eta-x)_+-(\eta-y)_+\le (y-x)_+$.
Combining \eqref{eq:game_hoeffding_analysis_1} and \eqref{eq:game_hoeffding_analysis_2},
\begin{equation}
    \Big(
    \overline Q_h^k(s,a,b)
    -
    Q^{\pi^k,\nu^k}_{h,P^\star,\mathbf{\Phi}}(s,a,b)
    \Big)_+
    \le
    \mathbb E_{P_h^\star(\cdot|s,a,b)}[D_{h+1}^k]
    +
    2\mathrm{bonus}_h^k(s,a,b).
    \label{eq:game_hoeffding_analysis_3}
\end{equation}

For each $(h,k)$, define
\begin{align}
    \Delta_h^k
    &:=
    \mathbb E_{a\sim\pi_h^k(\cdot|s_h^k),\,b\sim\nu_h^k(\cdot|s_h^k)}
    \left[
    \Big(
    \overline Q_h^k(s_h^k,a,b)
    -
    Q^{\pi^k,\nu^k}_{h,P^\star,\mathbf{\Phi}}(s_h^k,a,b)
    \Big)_+
    \right],
    \label{eq:game_hoeffding_delta}\\
    \zeta_h^k
    &:=
    \Delta_h^k-
    \Big(
    \overline Q_h^k(s_h^k,a_h^k,b_h^k)
    -
    Q^{\pi^k,\nu^k}_{h,P^\star,\mathbf{\Phi}}(s_h^k,a_h^k,b_h^k)
    \Big)_+,
    \label{eq:game_hoeffding_zeta}\\
    \xi_h^k
    &:=
    \mathbb E_{P_h^\star(\cdot|s_h^k,a_h^k,b_h^k)}[D_{h+1}^k]
    -
    D_{h+1}^k(s_{h+1}^k).
    \label{eq:game_hoeffding_xi}
\end{align}
By the minimax update of $\overline V_h^k$,
\[
\overline V_h^k(s_h^k)
\le
\mathbb E_{a\sim\pi_h^k(\cdot|s_h^k),\,b\sim\nu_h^k(\cdot|s_h^k)}
[\overline Q_h^k(s_h^k,a,b)],
\]
and by the robust Bellman recursion under $(\pi^k,\nu^k)$,
\[
V^{\pi^k,\nu^k}_{h,P^\star,\mathbf{\Phi}}(s_h^k)
=
\mathbb E_{a\sim\pi_h^k(\cdot|s_h^k),\,b\sim\nu_h^k(\cdot|s_h^k)}
[Q^{\pi^k,\nu^k}_{h,P^\star,\mathbf{\Phi}}(s_h^k,a,b)].
\]
Therefore,
\[
D_h^k(s_h^k)
\le
\left(
\mathbb E_{a\sim\pi_h^k(\cdot|s_h^k),\,b\sim\nu_h^k(\cdot|s_h^k)}
[
\overline Q_h^k(s_h^k,a,b)
-
Q^{\pi^k,\nu^k}_{h,P^\star,\mathbf{\Phi}}(s_h^k,a,b)
]
\right)_+
\le
\Delta_h^k.
\]
The last inequality uses Jensen's inequality and the convexity of $x\mapsto x_+$.
Consequently,
\begin{equation}
    \mathrm{Regret}_{\mathbf{\Phi},\{\nu^k\}}(K)
    \le
    \sum_{k=1}^K D_1^k(s_1^k)
    \le
    \sum_{k=1}^K \Delta_1^k.
    \label{eq:game_hoeffding_analysis_4}
\end{equation}
Let
\[
\mathcal F_{h,k}
:=
\sigma\left(
\{(s_i^\tau,a_i^\tau,b_i^\tau,s_{i+1}^\tau)\}_{i\in[H],\,\tau<k}
\cup
\{(s_i^k,a_i^k,b_i^k,s_{i+1}^k)\}_{i<h}
\cup
\{s_h^k\}
\right).
\]
The policies $\pi_h^k(\cdot|s_h^k)$ and $\nu_h^k(\cdot|s_h^k)$ are $\mathcal F_{h,k}$-measurable, and the opponent is non-anticipating with respect to Player~1's current randomization.
Thus
\[
\mathbb E[\zeta_h^k\mid\mathcal F_{h,k}]=0,
\qquad
\mathbb E[\xi_h^k\mid\sigma(\mathcal F_{h,k},a_h^k,b_h^k)]=0.
\]
The martingale concentration below is applied to this natural sequential filtration, which first reveals the history, then the sampled action pair, and then the next state.
Using \eqref{eq:game_hoeffding_analysis_3},
\begin{align}
    \Delta_h^k
    &=
    \zeta_h^k+
    \Big(
    \overline Q_h^k(s_h^k,a_h^k,b_h^k)
    -
    Q^{\pi^k,\nu^k}_{h,P^\star,\mathbf{\Phi}}(s_h^k,a_h^k,b_h^k)
    \Big)_+
    \nonumber\\
    &\le
    \zeta_h^k
    +
    \mathbb E_{P_h^\star(\cdot|s_h^k,a_h^k,b_h^k)}[D_{h+1}^k]
    +
    2\mathrm{bonus}_h^k(s_h^k,a_h^k,b_h^k)
    \nonumber\\
    &=
    \zeta_h^k+\xi_h^k+D_{h+1}^k(s_{h+1}^k)
    +
    2\mathrm{bonus}_h^k(s_h^k,a_h^k,b_h^k)
    \nonumber\\
    &\le
    \zeta_h^k+\xi_h^k+\Delta_{h+1}^k
    +
    2\mathrm{bonus}_h^k(s_h^k,a_h^k,b_h^k),
    \label{eq:game_hoeffding_analysis_5}
\end{align}
where the last inequality uses $D_{h+1}^k(s_{h+1}^k)\le\Delta_{h+1}^k$.
Recursively applying \eqref{eq:game_hoeffding_analysis_5} from $h=1$ to $H$ and using
the convention $D_{H+1}^k\equiv0$ and $\Delta_{H+1}^k=0$ yields
\begin{equation}
    \sum_{k=1}^K\Delta_1^k
    \le
    \sum_{k=1}^K\sum_{h=1}^H(\zeta_h^k+\xi_h^k)
    +
    2\sum_{k=1}^K\sum_{h=1}^H
    \mathrm{bonus}_h^k(s_h^k,a_h^k,b_h^k).
    \label{eq:game_hoeffding_analysis_6}
\end{equation}

\paragraph{Step 5: summing up the bonuses.}
By the value bound $0\le D_h^k\le\min\{H,\rho^{-1}\}$ and the clipped $Q$-updates, both
$\zeta_h^k$ and $\xi_h^k$ are uniformly bounded by $\min\{H,\rho^{-1}\}$ up to an absolute constant.
Thus a separate application of the Azuma--Hoeffding inequality implies that, with probability at least $1-\delta/2$,
\begin{equation}
    \sum_{k=1}^K\sum_{h=1}^H(\zeta_h^k+\xi_h^k)
    \le
    C_1\min\{H,\rho^{-1}\}\sqrt{HK\iota},
    \label{eq:game_hoeffding_mds_sum}
\end{equation}
where $C_1>0$ is an absolute constant.

It remains to control the bonus sum. By \eqref{eq:mg_hoeffding_bonus_new},
\begin{align}
    &\sum_{k=1}^K\sum_{h=1}^H
    \mathrm{bonus}_h^k(s_h^k,a_h^k,b_h^k)
    \nonumber\\
    &\qquad\le
    c_b\min\{H,\rho^{-1}\}\sqrt{S\iota}
    \sum_{k=1}^K\sum_{h=1}^H
    \frac{1}{\sqrt{N_h^k(s_h^k,a_h^k,b_h^k)\vee 1}}.
    \label{eq:game_hoeffding_bonus_sum_1}
\end{align}
For each fixed $h$, summing over the $SAB$ state-action-opponent-action triples and using
$\sum_{j=0}^{n-1}(j\vee1)^{-1/2}\le 2\sqrt n+1$ gives
\[
\sum_{k=1}^K
\frac{1}{\sqrt{N_h^k(s_h^k,a_h^k,b_h^k)\vee 1}}
\le
C_2\sqrt{SABK}
\]
for an absolute constant $C_2>0$.
Summing over $h\in[H]$ and substituting into \eqref{eq:game_hoeffding_bonus_sum_1}, we get
\begin{equation}
    \sum_{k=1}^K\sum_{h=1}^H
    \mathrm{bonus}_h^k(s_h^k,a_h^k,b_h^k)
    \le
    C_3\min\{H,\rho^{-1}\}\,H S\sqrt{ABK\iota},
    \label{eq:game_hoeffding_bonus_sum_2}
\end{equation}
where $C_3>0$ is an absolute constant.

Combining
\eqref{eq:game_hoeffding_analysis_4},
\eqref{eq:game_hoeffding_analysis_6},
\eqref{eq:game_hoeffding_mds_sum}, and
\eqref{eq:game_hoeffding_bonus_sum_2}, and absorbing logarithmic factors into
$\widetilde{\mathcal O}(\cdot)$, we obtain by a union bound that, with probability at least $1-\delta$,
\[
\mathrm{Regret}_{\mathbf{\Phi},\{\nu^k\}}(K)
\le
\widetilde{\mathcal O}\!\left(
\min\{H,\rho^{-1}\}\,H S\sqrt{ABK}
\right).
\]
This completes the proof of Theorem~\ref{thm:onesided_regret}.
\endgroup
\end{proof}

\end{document}